%% file: main.tex
\documentclass[acmsmall, nonacm]{acmart}

\renewcommand\footnotetextcopyrightpermission[1]{}
\makeatletter
\let\@authorsaddresses\@empty
\makeatother
\usepackage{graphicx}
\usepackage{autonum}
\usepackage{cancel}
\usepackage[ruled,commentsnumbered]{algorithm2e}
\usepackage{wrapfig}
\usepackage{amsmath,amsfonts,amsthm,bbm,bm,comment,multirow}
\usepackage{mdframed}
\usepackage{caption}
\usepackage{subcaption}

\newtheorem{theorem}{Theorem}
\newtheorem{lemma}{Lemma}
\usepackage{nicefrac}
\usepackage{url}

\newtheorem{definition}{Definition}

\usepackage[utf8]{inputenc}

\newcommand{\sumT}{\sum_{t=1}^T} 
\newcommand{\N}{\mathcal{N}}

\newcommand{\X}{\mathcal{X}}
\newcommand{\convX}{\text{conv}(\mathcal{X})}
\newcommand{\ti}{\tilde} 
\newcommand{\dtp}[2]{\langle {#1}, {#2} \rangle}
\newcommand{\grd}{\nabla}
\newcommand{\err}{\|\theta_t - \ti\theta_t\|}

\newcommand{\expec}[2]{\mathop{\mathbb{E}}_{#1}\left[{#2}\right]}

\newcommand{\ser}{\left\{ x\right\}_T}

\newcommand{\convXs}{\text{conv}(\mathcal{X}_s)}
\DeclareMathOperator*{\argmax}{argmax}
\newcommand{\convXb}{\text{conv}(\mathcal{X}_b)}

\usepackage{tabularx}

\newcolumntype{Y}{>{\centering\arraybackslash}X}

\usepackage{makecell}

\begin{document}
\title{Optimistic No-regret Algorithms for Discrete Caching}

\author{Naram Mhaisen}
\affiliation{%
  \institution{Delft University of Technology}
  \city{Delft}
  \country{The Netherlands}
}

\author{Abhishek Sinha}
\affiliation{%
  \institution{Tata Institute of Fundamental Research}
  \city{Mumbai}
  \country{India}
}

\author{Georgios Paschos}
\affiliation{%
  \institution{Amazon}
  \city{Luxembourg}
  \country{Luxembourg}
}

\author{George Iosifidis}
\affiliation{%
  \institution{Delft University of Technology}
  \city{Delft}
  \country{The Netherlands}
}
\renewcommand{\shortauthors}{Naram Mhaisen et al.}

\begin{abstract}
We take a systematic look at the problem of storing whole files in a cache with limited capacity in the context of optimistic learning, where the caching policy has access to a prediction oracle (provided by, e.g., a Neural Network). The successive file requests are assumed to be generated by an adversary, and no assumption is made on the accuracy of the oracle. In this setting, we provide a universal lower bound for prediction-assisted online caching and proceed to design a suite of policies with a range of performance-complexity trade-offs. All proposed policies offer sublinear regret bounds commensurate with the accuracy of the oracle. Our results substantially improve upon all recently-proposed online caching policies, which, being unable to exploit the oracle predictions, offer only $O(\sqrt{T})$ regret. In this pursuit, we design, to the best of our knowledge, the first comprehensive optimistic Follow-the-Perturbed leader policy, which generalizes beyond the caching problem. We also study the problem of caching files with different sizes and the bipartite network caching problem. Finally, we evaluate the efficacy of the proposed policies through extensive numerical experiments using real-world traces.    
\end{abstract}
\settopmatter{printacmref=false}



\maketitle
\input{introduction}
\vspace{-1mm}
\input{related_work}

\input{lower_bound}
\input{OFTRL}
\input{OFTPL}
\input{OFTPL_US}
\input{OFTRL_US}
\input{OFTRL_BP}
\input{experts}

\input{experiments}

\input{Conclusions}

\begin{acks}
This publication has emanated from
research conducted with the financial support of the European Commission
through Grant No. 101017109 (DAEMON). Abhishek Sinha is supported in part by a US-India NSF-DST collaborative research grant coordinated by IDEAS-Technology Innovation Hub (TIH) at the Indian Statistical Institute, Kolkata.
\end{acks}

\bibliography{references, george}
\bibliographystyle{ACM-Reference-Format}
\input{appendix}

\end{document}

%% file: introduction.tex
\section{Introduction}\label{sec:intro}

This paper addresses the discrete caching (prefetching) problem: choose files to replicate in a local cache in order to maximize the probability that a new file request is served locally. \emph{Hitting} the cache speeds up CPU, optimizes user experience in CDN's \cite{bektas}, and enhances the performance of wireless networks \cite{femtocaching}. With the perpetual growth of Internet traffic fueled by new services such as AR/VR \cite{chatzopoulos-ARVR}, caching policies that learn fast to maximize cache hits can mitigate the increasing costs of information transportation \cite{paschos-jsac}, and similar benefits can be expected for embedded and other computing systems \cite{embedded-caching}. This work aspires to advance our theoretical understanding of this fundamental problem and proposes new provably-optimal and computationally-efficient caching algorithms using a new modeling and solution approach based on \emph{optimistic learning}.

\vspace{-1mm}
\subsection{Motivation}

\begin{figure*}
    \centering
    \includegraphics[scale=0.7]{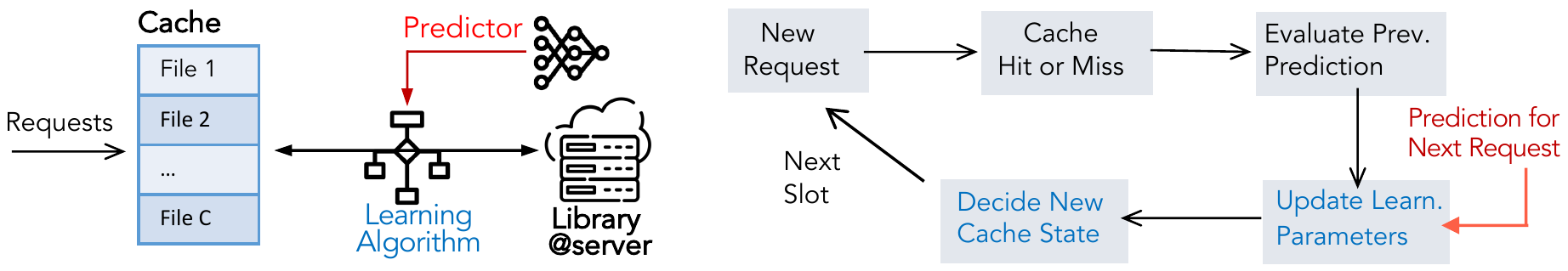}
    \vspace{-2mm}
    \caption{Optimistic online caching with predictions: system schematic (\emph{left}) \& algorithm template (\emph{right}).}
    \label{fig:caching-fig}
    \vspace{-5mm}
\end{figure*}

Common caching policies store the newly requested files and employ the Least-Recently-Used (LRU) \cite{lru-sigmetrics08}, Least-Frequently-Used (LFU) \cite{lfu-sigmetrics99} and other similar rules to evict files when the cache capacity is exhausted. Under certain statistical assumptions on the request trace, such policies maintain the cache at an optimal state, see \cite[Sec. 3.1-3.2]{paschos-book}. However, with frequent addition of new content to libraries of online services and the high volatility of file popularity \cite{volatility}, these policies can perform arbitrarily bad. This has spurred intensive research efforts for policies that operate under more general conditions by learning on-the-fly the request distribution or adapting dynamically, e.g., with Reinforcement Learning, to observed requests; see Sec. \ref{sec:related}. Nonetheless, these studies do not offer performance guarantees nor scale well for large libraries. The goal of this work is to design robust caching policies that are able to learn effective caching decisions with the aid of a prediction oracle of \emph{unknown quality} (Fig. \ref{fig:caching-fig} left) even when the file requests are made in an adversarial fashion.

To that end, we formulate the caching problem as an online convex optimization (OCO) problem \cite{hazan-book}. At each slot $t =1, 2, \dots, T$, a learner (the caching policy) selects a caching vector $x_t\in \X$ from the set of admissible cache states $\X \subseteq \{0,1\}^N$ for a cache of size $C,$ where $N$ is the library size. Then, a $1$-hot vector $\theta_t\in\{0,1\}^N$ with value $1$ for the requested file is revealed, and the learner receives a reward of $f_t(x_t)=\dtp{\theta_t}{x_t}$ for cache hits. The reward is revealed only after committing $x_t$, which naturally matches the dynamic caching operation where the cached files are decided before the next request arrives. Here, the learner makes no statistical assumptions and $\theta_t$ can follow any distribution, even one that is handpicked by an adversary. In the optimistic framework, the learner does not only consider its hit or miss performance so far when deciding $x_t$, but also the predictor's performance and output (Fig. \ref{fig:caching-fig} right). As customary in the online learning literature, we characterize the policy's performance by using the static \emph{regret} metric: 
\begin{align}
	R_T(\ser)\triangleq \sup_{ \{f_t\}_{t=1}^T}\left\{ \sumT f_t({x}^\star) - \sumT f_t({x}_t)\right\},
\end{align}
where ${x}^\star \!=\! \argmax_{x\in\mathcal{X}} \sumT f_t({x})$ is the (typically unknown) \emph{best-in-hindsight} cache decision that can be selected only with access to future requests.\footnote{It is interesting to note that ${x}^\star$ caches the most frequent requests, which coincides with the limit behavior of LFU.} The regret measures the accumulated reward gap between the online decisions $\{x_t\}_t$ and benchmark $x^{\star}$. An algorithm is said to achieve sublinear regret when its average performance gap $R_T/T$ vanishes as $T \rightarrow \infty$. In this context, recent works have proposed caching policies that offer $O(\sqrt{T})$ regret bound \cite{abhishek-sigm20,tareq-jrnl, tareq-conf, paschos-jrnl, naram-jrnl,leadcache,ocol}, which, in fact, is the optimal (as small as possible) achievable regret rate, see \cite[Thm. 5.1]{orabona2021modern}, \cite[Thm. 1]{abhishek-sigm20}.

Most of these regret-optimal algorithms have been designed for \emph{continuous} caching, where it is assumed that each file is encoded and divided into a large number of small chunks such that storing them can be approximated by continuous variables \cite{maddah-ali}. In this case, the set of eligible caching states $\X$ is convex and hence one can readily apply standard OCO algorithms such as the Online Gradient Ascent (OGA). Albeit a handy assumption, there are settings where continuous caching cannot be used for practical reasons. Namely, keeping chunk meta-data consumes non-negligible storage; the coding operation is often computationally demanding; and the number of chunks might not be big enough to render continuous caching a good approximation. Therefore, we consider here the more realistic, and more challenging to solve, \emph{discrete} caching problem. Indeed, in discrete caching the set $\X$ is naturally non-convex (containing binary file-caching decisions) and thus standard OCO policies cannot be employed. While first steps in the study of discrete caching, with equal-sized files, were recently made by \cite{tareq-jrnl, abhishek-sigm20, leadcache}. In this paper, we extend their scope and design algorithms with substantially improved performance guarantees.

Namely, while regret minimization yields robust policies that learn under adversarial conditions, this framework receives the fair criticism that the policies have often suboptimal performance when the requests (cost functions, in general) are predictable, e.g., stationary. In such situations, we would like the policy to gauge the predictability of requests, and optimize aggressively the cache. For instance, requests in services like Facebook are often amenable to accurate forecasts; while in YouTube and Netflix the viewers receive recommendations which can effectively serve as predictions for their forthcoming requests \cite{netflix, recommend2}. Unfortunately, regret-based caching policies, such as \cite{paschos-jrnl, tareq-conf, tareq-jrnl, abhishek-sigm20, leadcache, 10.1145/3491047}, are \emph{pessimistically} designed for the worst-case request sequence and cannot benefit from predictable requests. We tackle this shortcoming by designing a new suite of \emph{optimistic} caching algorithms. An optimistic algorithm \cite{rakhlin-nips2013, mohri-aistats2016} has access to a prediction of an unknown quality for the next-slot utility function. The ultimate goal is to achieve constant (independent of $T$) regret when the predictions are accurate, while maintaining the worst-case regret bounds when predictions fail. This best-of-two-worlds approach, we show theoretically and demonstrate numerically, brings significant performance gains to dynamic caching.

\vspace{-1mm}
\subsection{Methodology and Contributions}

We study key variants of the discrete caching problem, namely the single cache with equal or unequal-sized files and the bipartite caching, and propose a suite of optimistic learning algorithms with different pros and cons.\footnote{Optimistic learning was originally proposed for problems with slowly-varying (hence, predictable) cost functions \cite{rakhlin-nips2013}; in caching, we note the additional motivation coming from the abundance of forecasting models, \emph{e.g.,} by a Neural Network.}
Our first result demonstrates the best achievable regret in the setup we consider, which turns out to be $R_T= \Omega([\sum_t\err]^{\nicefrac{1}{2}})$, indicating a significant potential of obtaining a regret that scales with the predictor's error rather than the time horizon $T$ (Sec. \ref{sec:lb}). We then proceed to propose variants of the seminal Follow-The-Regularized-Leader (FTRL) and Follow-the-Perturbed-Leader (FTPL) algorithms, which can be both viewed as \emph{smoothing} techniques for stabilizing learning decisions (Sec. \ref{sec:smoothing}), whose regret match this lower bound up to constants. In detail, we expand the optimistic FTRL algorithm \cite{mohri-aistats2016, ocol, naram-jrnl} that was designed for convex problems, to handle, through sampling, discrete (non-convex) decisions (Sec. \ref{sec:oftrl-m}). We prove this approach attains expected regret $O(\sqrt T)$ for worst-case predictions and \emph{zero-regret} for perfect predictions with an improved prefactor that does not depend on library size $N$. However, the OFTRL implementation can be hindered by an involved projection step that might be computationally expensive\footnote{In some cases the projection can be optimized, but in general it is $O(N^2)$ even for the non-weighted capped simplex \cite{projection-simplex}.}. Thus, we develop a new optimistic FTPL algorithm that applies prediction-adaptive perturbations to achieve a similar regret bound with linear ($O(N)$) computation overhead (Sec. \ref{sec:oftpl}). The flip side is that its regret bound contains $O(\textit{poly-log}(N))$ term.

We first derive results for equal-sized files, in line with all prior learning-based works for discrete caching \cite{abhishek-sigm20, leadcache, Lykouris-ML, Rohatgi} or continuous caching \cite{paschos-jrnl, naram-jrnl, tareq-jrnl}. Subsequently, we drop this assumption and study the single cache problem with different file sizes (Sec. \ref{sec:arb-sizes}). These first-of-their-kind regret-based algorithms require a point-wise approximation scheme for solving efficiently the NP-Hard Knapsack instance at each slot, while keeping the accumulated regret bound sublinear. To that end, we use the help of a rounding subroutine, \texttt{DepRound} \cite{depround}, to a known almost-discrete optimal solution \texttt{Dantz} \cite{dantzing}. We show that the proposed policies achieve $(1/2)$-approximate regret of $O(\sqrt T)$ and zero-regret for adversarial and perfect predictions, respectively. We also extend the OFTRL analysis to the widely used bipartite network caching model \cite{femtocaching, paschos-jsac} (Sec. \ref{sec:OFTRL-BP}), where we optimize jointly the discrete caching and routing decisions to obtain prediction-modulated performance.

In (Sec. \ref{sec:experts}), we change tack and incorporate the optimism through the celebrated Experts model. The caching system in this case is a \emph{meta-learner} which receives caching advice from an \emph{optimistic expert} that suggests to cache solely w.r.t. predicted requests, and from a \emph{pessimistic expert} that ignores predictions. We propose a tailored OGD-based scheme that allows the meta-learner to adapt to predictions' accuracy ( performance of the optimistic expert) in a way that achieves negative regret when that expert is reliable, and, again, maintains an $O(\sqrt T)$ regret for unreliable predictions.

In summary, we provide a comprehensive toolbox of algorithms having different computation overheads and performance, hence enabling practitioners to select the best approach to their problem. Moreover, we include technical results that are of independent interest, such as the non-convex OFTPL algorithm with improved regret bounds; the approximate non-convex OFTRL algorithm for the Knapsack problem; and an analysis of OFTRL/OFTPL with a probabilistic prediction model.

\textbf{Notation.} We denote sets with calligraphic capital letters, e.g., $\N= \{1,2,\ldots,N\}$; vectors with $x=(x_i, i\in\mathcal{N})$ where $x_i$ is the $i$th component; and denote $x_{i}^t$ the $i$th component of the time-indexed vector $x_t$. The shorthand notation $x_{1:t}$ is used for $\sum_{i=1}^t x_i$. Also, $\{x_t\}_{t=1}^k$ denotes the sequence of vectors $\{x_1, x_2, \ldots, x_k\}$, and we use the succinct version $\{x_t\}_T$ for $\{x_1, x_2, \ldots, x_T\}$. When clear from context, we often drop the notation of actions and denote the regret $R_T(\ser)$ simply with $R_T$.

%% file: related_work.tex
\section{Background and Related Work}\label{sec:related}

\subsection{Caching and Learning}
Research on caching optimization spans several decades and we refer the reader to survey \cite{paschos-book} for an introduction to the recent developments in this area. A large body of works focuses on offline policies which use the anticipated request pattern to proactively populate the caches with files that maximize the expected hits \cite{bektas}. At the other extreme, dynamic caching solutions studied variants of the LFU/LRU policies \cite{lru-sigmetrics08,lfu-sigmetrics99,leonardi18, giovanidis-georgraphic}; tracked the request distribution \cite{snm,kauffman} and optimized accordingly the caching \cite{mathieu}; employed reinforcement learning to adapt the caching decisions to requests \cite{giannakis-q-learning, gunduz-reinforcement, 8790766}; and, more recently, applied online convex optimization towards enabling the policies to handle unknown (adversarial) request patterns \cite{paschos-jrnl, tareq-jrnl, naram-jrnl, abhishek-sigm20, leadcache, tareq-conf,10.1145/3491047}. These latter works assume that the files can be fetched dynamically at each slot to optimize the cache configuration, as opposed to works such as \cite{Lykouris-ML, Rohatgi} which study pure eviction policies.

The interplay between predictions and caching has attracted attention from both machine learning and networking communities. The studies in \cite{jordan-tmc, 9528062, 9681354} formulated the joint caching and recommendation problem, considering static models, and assuming full knowledge of requests and the users' propensity to follow recommendations (i.e., they place assumption on the prediction accuracy). On the other hand, \cite{Lykouris-ML, Rohatgi} presented a mechanism agnostic to requests that uses untrusted predictions to achieve \emph{competitive-ratio} guarantees. Their approach was generalized to metrical task systems by \cite{antoniadis_20} and improved with nearly lower-bound matching for the competitive ratios in \cite{rutten_22}. However, as proved in \cite{pmlr-v30-Andrew13}, algorithms that ensure constant competitive-ratios do not necessarily guarantee sub-linear regret, which is the performance criterion we employ here following the recent regret-based caching research \cite{paschos-jrnl, tareq-jrnl, naram-jrnl, abhishek-sigm20, leadcache, tareq-conf, 10.1145/3491047}. We note that all the above works consider files with equal size, while we extend the framework to the general scenario of \emph{unequally-sized} files. In addition, none of the above works studies discrete caching with predictions. Finally, it is worth stressing that employing predictions for improving the performance of communication/computing systems is not a new idea: predictions have been incorporated in stochastic optimization \cite{longbo-ton18, lui-ton21} which assume the requests and system perturbations are \emph{stationary}; and in online learning \cite{xu-mobihoc19, comden-sigmetrics19} which \emph{do not adapt} to predictions' accuracy (considered known). Here, we make \emph{no assumptions} on the predictions' quality, which can be even adversarial.

\subsection{Adaptive Smoothing} \label{sec:smoothing}

In contrast to the above studies, our optimistic learning approach is based on \emph{adaptive smoothing}. Abernethy \textit{et. al.} \cite{abernethy14} introduced a unified view of FTRL and FTPL as techniques to add smoothing, through \emph{regularization} or \emph{perturbation}, to a non-smooth potential function. This perspective is useful to our work since we leverage both ideas. Namely, let us define:
$\Phi(\theta) \triangleq {\max_{x \in \mathcal{X}}\dtp{x}{\theta}}$, and consider the potential function $\Phi(\Theta_t)$, where  $\Theta_t= \theta_{1:t}$ is the vector of aggregated gradients (file requests). An intuitive strategy is to choose the action that maximizes the rewards seen so far: 
\begin{eqnarray}
    x_t = \argmax_{x\in\mathcal{X}} \dtp{\Theta_{t-1}}{x} = \grd \Phi(\Theta_{t-1}),
    \label{eq:FTL}
\end{eqnarray}
which is known as Follow The Leader (FTL) and is optimal when the utility functions are samples from a stationary statistical distribution. In contrast, FTL has linear regret in the adversarial setting \cite{sarah-between, follow-hedge}, since successive gradients of non-smooth functions can be arbitrarily far from each other, thus leading to unstable actions.  \cite{abernethy14} proposed to stabilize the learner actions by \emph{smoothing} the potential function, and selecting actions based on the \emph{smoothed} potential $\grd\ti\Phi(\theta)$. In FTRL, the smoothing is achieved by adding a strongly convex function to the potential, i.e.,\footnote{While the maximization requires that $x_t$ to be in the convex hull of $\mathcal{X}$, feasibility can be recovered via appropriate rounding.}$x_t = \argmax_{x\in\convX} \dtp{x}{\Theta_t} - r_{1:t}(x)$

where $r_t(x)$ is a $\sigma_t$-strongly convex regularizer. This framework generalizes the Online Gradient Ascent (OGA) and the Exponentiated Weights (EG) algorithms, which were employed for the caching problem in \cite{paschos-jrnl} and \cite{tareq-jrnl} respectively\footnote{We note that these papers present their algorithms as instances of a similar framework to FTRL called Online Mirror Descent (OMD). Nonetheless, there exist equivalence results between these two frameworks (see \cite[Sec. 6.1 ]{mcmahan-survey17}) for specific choices of the \emph{mirror-map} (in OMD), or equivalently the regularizer (in FTRL).}. As for FTPL, the smoothing is done by adding perturbation to the accumulated cost parameter of the potential. And the actions are decided by\footnote{In this case the gradient of the smoothed potential is in fact the expectation $\grd\Phi(\Theta_t+\eta_t\gamma) = \expec{\gamma}{x_t}$} $x_t = \arg\max_{x\in\mathcal{X}} \dtp{x}{\Theta_t+\eta_t\gamma}$,

where $\gamma\!\! \sim\!\! \mathcal{N}(0,\mathbf{1})$ and $\eta_t$ is a scaling factor that controls the smoothing. FTPL was shown to provide optimal regret guarantees for the discrete caching problem in \cite{abhishek-sigm20}. Computationally efficiency is also a notable feature for FTPL updates as it requires an ordering operation instead of projection.

We propose to modulate the regularization $\sigma_t$ and perturbation $\gamma_t$ parameters with the predictions quality. Intuitively, accurate predictions should lead to less regularization/perturbation (less smoothing), enabling the learner to align its decisions more with the predictions. On the other hand, inaccurate predictions induce more smoothing, which alleviates their effects on the decisions. We show that careful tuning of these smoothing parameters leads to regret bounds that interpolate between $R_T \leq 0$, and $R_T \leq O(\sqrt{T})$. Nonetheless, these two algorithms have considerable differences in terms of computational complexity and constants in the bounds, which are discussed in detail. 

\subsection{Optimistic Learning}

For \emph{regret} minimization with predictions, \cite{dekel-nips17} used predictions $\theta_t$ for the gradient $\ti\theta_t = \grd f_t(x_t)$ with guaranteed correlation $\dtp{\ti\theta_t}{\theta_t}\geq \alpha \|\theta_t\|^2$ to improve the regret . In \cite{google-2020}, this assumption was relaxed to allow predictions to fail the correlation condition at some steps, obtaining bounds that interpolate between $O(\log(T))$ and $O(\sqrt{T})$; while this idea was extended to multiple predictors in \cite{google-nips-2020}. A different line of works \cite{rakhlin-nips2013}, \cite{mohri-aistats2016} use adaptive regularizers and define the $t$-slot prediction errors $\|\theta_t - \ti \theta_t\|_2$ to obtain $O([\sum_{t} \|\theta_t-\ti \theta_t \|^2]^{1/2})$ regret bounds. Specifically, OFTRL versions have been proposed in \cite{mohri-aistats2016} and recently used in \cite{daron-ton} for problems with budget constraints, while \cite{naram-jrnl, ocol} tailored these ideas to \emph{continuous} caching. The problem of discrete caching is fundamentally different. Through a careful analysis,  we manage to reuse these results after relaxing the cache integrality constraints, and then employing a randomized rounding technique that recovers the same prediction-modulated regret in expectation. The regret bounds have the desirable property of being \emph{dimension-free}. Nonetheless, we proceed to remark that OFTRL can have a computational bottleneck due to involving a projection step, which can be avoided in FTPL.

Optimistic versions of FTPL were recently investigated in \cite{suggala-1} and \cite{suggala-2}. In \cite{suggala-1}, the regret bound grows polynomially w.r.t. the decision set dimension. In the caching problem, this would imply a highly-problematic polynomial growth of the regret w.r.t. the typically huge library size $N$. The dependence of the regret on the dimension was improved in \cite{suggala-2}, but it still remains linear. On the contrary, our proposed OFTPL exploits the structure of the decision set and utilizes adaptive perturbation to obtain a regret bound that depends on dimension only by $O(\log(N)^{1/4})$, is order-optimal (based on the achievable lower bound), returns zero-regret for perfect predictions, does not require knowing the time-horizon $T$, nor the prediction errors. None of these desirable features is available in these prior works. We kindly refer the reader to the table in Appendix. \ref{appendix:summary-table} for an overview of the presented algorithms in the context of the most related literature.

%% file: lower_bound.tex
\section{Achievable regret for Caching with a Predictor}
\label{sec:lb}
We first introduce a lower bound for the regret of any online caching policy $\pi$, working with a cache of capacity $C$, and has access to an \emph{untrusted} and potentially \emph{adversarial} prediction oracle. In general, the predictions refer to the next function $\tilde f_t(\cdot)$. However, since most OCO algorithms learn based on the observed gradients, it suffices to have predictions $\tilde \theta_t=\nabla \tilde f_t(x_t)$. And for caching, this coincides with a prediction for the next request\footnote{In fact this model can be readily generalized to other linear utilities beyond cache-hits, so as to incorporate e.g., file-specific caching gains, time-varying network conditions, and so on; see similar models in \cite{paschos-jrnl, ocol}.}. Now, unlike all prior works in optimistic learning \cite{rakhlin-nips2013, mohri-aistats2016, google-2020}, we adopt here the more general \emph{probabilistic} prediction model where $\tilde\theta_t$ is not necessarily a one-hot vector (as the actual $\theta_t$), but a probability distribution over the library. Thus, each $\ti{\theta}_t$ is drawn from the N-dimensional probability simplex $\Delta_N$. This more general approach is rather intuitive as the forecasting models (e.g., a Neural Network) typically yield probabilistic inferences. It also enhances the performance of our optimistic algorithms and allows efficient training of the forecaster using a convex loss function (please see Appendix \ref{appendix:simulations} for examples and justification). It does require, however, a more elaborate technical analysis, especially for the case of OFTPL. In this setup, we have the following lower bound:
\begin{theorem}
For any online caching policy $\pi$, there exist a sequence of requests $\{\theta_t\}_T$ and predictions $\{\tilde\theta_t\}_T$ for which the regret $R_T$ satisfies
\begin{equation}
    \expec{}{R_T} \geq \sqrt{\frac{C}{2\pi}}\sqrt{\sum_{t=1}^T || \theta_t -\ti\theta_t||_2^2}-\Theta\left(\frac{1}{\sqrt{T}}\right).
\end{equation}
\end{theorem}
\begin{proof}
To prove the lower bound, we show the existence of a request and prediction sequence under which the regret is guaranteed to be larger than the stated bound regardless of the online policy $\pi$. For that, we use the standard probabilistic method \cite{alon2016probabilistic} with an appropriately constructed random file request and prediction sequence as detailed below.

Denote by $\xi_t$ and $\ti\xi_t$ the random variables representing the requested file ($\theta_t$) and its prediction ($\tilde{\theta}_t$) at time $t$, respectively. Denote by $\{X_t^{\pi}\}_{t\geq 1}$ the random variables representing the action of any policy $\pi$.
We use a setup where $N\geq2C$ and consider
an ensemble of caching problems (i.e., request and prediction sequences) where at each slot $t$, the requested file $\xi_t$ is chosen independently and uniformly at random from the library $\mathcal{N}$. The predictions $\ti\xi_t$ are also chosen independently and uniformly at random from the probability simplex $\Delta_N$. Specifically, we let
\begin{eqnarray*}
	\{\tilde{\xi}_t\}_t\stackrel{\textrm{i.i.d.}}{\sim} \textsc{Dirichlet}(\lambda_1, \ldots, \lambda_n, \ldots \lambda_N) \quad \text{with} \quad \lambda_n=1, \forall n\in \mathcal N. 
\end{eqnarray*}
 
Hence, the expected reward obtained by any caching policy $\pi$ on any slot $t$, conditional on the information available to the policy can be bounded as  
\begin{eqnarray*}
	\mathbb{E}\left[\langle \xi_t, X_t^{\pi} \rangle|\{\ti\xi_\tau\}_{\tau=1}^t, \{\xi_\tau\}_{\tau=1}^{t-1}\right] &&\stackrel{(a)}{=}\mathbb{E}\mathbb{E}\left[\langle \xi_t, X_t^{\pi} \rangle| \{\ti\xi_\tau\}_{\tau=1}^t, \{\xi_\tau\}_{\tau=1}^{t-1}, X_t^\pi\right]\\
	&&\stackrel{(b)}{=}\frac{1}{2C}\mathbb{E}\bigg[\langle \bm{1}, X_t^{\pi} \rangle| \{\ti\xi_\tau\}_{\tau=1}^t, \{\xi_\tau\}_{\tau=1}^{t-1}, X_t^\pi\bigg]
	\stackrel{(c)}{\leq} \frac{1}{2}, 
\end{eqnarray*}
	where $(a)$ follows from the tower property of expectations, $(b)$ from the fact $\xi_t \perp \big(\{\ti\xi_\tau\}_{\tau=1}^t, \{\xi_\tau\}_{\tau=1}^{t-1}, X_t^\pi\big)$ and hence $\mathbb{E}(\xi_t|\{\ti\xi_\tau\}_{\tau=1}^t, \{\xi_\tau\}_{\tau=1}^{t-1}, X_t^\pi)= \mathbb{E}(\xi_t)= \frac{1}{2C} \bm{1}_{N\times 1}$.
	Finally $(c)$ since $\langle \bm{1}, X_t^\pi\rangle \leq C,$ which holds because of the cache capacity constraint. Taking expectation of the above bound, we have $
	\mathbb{E}[\langle \xi_t, X_t^{\pi} \rangle]	\leq \frac{1}{2}$. Hence, using the linearity of expectations, the expected value of the cumulative hits up to slot $T$ under any policy $\pi$ is upper bounded as $\mathbb{E}\big[\sum_{t=1}^T\langle \xi_t, X_t^{\pi} \rangle\big]	\leq \frac{T}{2}$.

Now we compute a lower bound to the expected number of cumulative hits achieved by the best-in-hindsight fixed cache configuration $X_T^\star$. Similar to \cite{abhishek-sigm20}, we identify the problem with the classic setup of balls (requests) into bins (files). In this framework, it follows that the offline benchmark achieves cumulative hits which are equal to the total number of balls into the most-loaded $C$ bins when $T$ balls are thrown uniformly at random into $N=2C$ bins. Hence, from \cite[Lemma 1]{abhishek-sigm20}:
\begin{equation}
	\mathbb{E}\big[\sum_{t=1}^T\langle \xi_t, X_T^\star \rangle\big] \geq \frac{T}{2}+ \sqrt{\frac{CT}{2\pi}}-\Theta\left(\frac{1}{\sqrt{T}}\right).
\end{equation}
Hence, the expected regret achieved by any policy in the optimistic set up is lower bounded as
\begin{eqnarray} \label{lb1}
	\mathbb{E}\left[R_T\right] = \mathbb{E} \bigg(\sum_{t=1}^T\langle \xi_t, X_T^\star \rangle -  \sum_{t=1}^T\langle \xi_t, X_t^{\pi} \rangle\bigg) \geq  \sqrt{\frac{CT}{2\pi}}-\Theta\left(\frac{1}{\sqrt{T}}\right).
\end{eqnarray}
Finally, we evaluate the expected value of the quantity $M_T \triangleq \sum_{t=1}^T \|\tilde\xi_t -\xi_t \|_2^2$ as follows.
\begin{align}
	\mathbb{E}\left[M_T\right] &= \mathbb{E}\left[\sum_{t=1}^T \|\tilde\xi_t -\xi_t \|_2^2\right]
	\stackrel{(a)}{=} T\ \mathbb{E}\left[||\tilde{\xi}_1 - \xi_1||_2^2\right]
	\stackrel{(b)}{=} TN\ \mathbb{E}\left[(\tilde{\xi}^1_1 - \xi^1_1)^2\right]
	\\
	&=TN\ \mathbb{E}\left[(\tilde{\xi}^1_1)^2 - 2\tilde{\xi}^1_1 \xi^1_1 + (\xi^1_1)^2\right]
	\stackrel{(c)}{=} TN\bigg[\textrm{Var}(\tilde{\xi}_{1,1}) + (\mathbb{E}(\tilde{\xi}_{1,1}))^2 -2 \mathbb{E}(\ti{\xi}_{1,1})\mathbb{E}(\xi_{1,1}) + \mathbb{E}(\xi^2_{1,1})\bigg]
	\\
	&\stackrel{(d)}{=} TN\bigg[\frac{(N-1)}{N^2(N+1)}+ \frac{1}{N^2}-\frac{2}{N^2} + \frac{1}{N}\bigg]
	=T\big(1-\frac{2}{N(N+1)}\big) \leq T,
\end{align}
where $(a)$ follows from the i.i.d. assumption of the random vectors at each $t$, $(b)$ from the i.i.d assumption of each component of vectors $\tilde{\xi}_1$ and $\xi_1$, $(c)$ from $\xi_t \perp \tilde\xi_t$, and $(d)$ from standard results on Dirichlet distribution. Combining the above bound with \eqref{lb1}, we have by Jensen's inequality
\begin{eqnarray*}
\mathbb{E}\left[R_T\right] \geq \sqrt{\frac{C\mathbb{E}[M_T]}{2\pi}} - \Theta\left(\frac{1}{\sqrt{T}}\right) \geq \mathbb{E}\left[ \sqrt{\frac{C M_T}{2\pi}}-\Theta\left(\frac{1}{\sqrt{T}}\right)\right]\quad\quad\quad \textit{i.e., }
\end{eqnarray*}
\begin{eqnarray*}
     \mathbb{E}\left[R_T - \sqrt{\frac{C \sum_{t=1}^T\|\tilde\xi_t -\xi_t \|_2^2}{2\pi}}\right] \geq -\Theta\left(\frac{1}{\sqrt{T}}\right) 
 \end{eqnarray*}
From the above inequality, the result now follows from the standard probabilistic arguments. 
\end{proof}

We will see that the proposed optimistic algorithms in Sections \ref{sec:oftrl-m} and \ref{sec:oftpl} attain this bound within an absolute and a poly-logarithmic factor, respectively.

%% file: OFTRL.tex
\section{Caching through Optimistic Regularization (\texttt{OFTRL-Cache})}
\label{sec:oftrl-m}

The first algorithm we propose is based on OFTRL. Prediction adaptive regularization was explored before in \cite{rakhlin-nips2013} and later improved via proximal regularizers in \cite{mohri-aistats2016}, all for convex sets. The gist of our approach is that we use OFTRL to obtain $\hat x_t \in \convX, \forall t$, and then apply Madow's sampling scheme \cite{madow} to recover integral caching vectors $x_t \in \mathcal{X}, \forall t,$ which satisfy the hard capacity non-convex constraint. In other words, we define:
\begin{align}
	\mathcal{X} =  \left\{ {x} \in\left\{0, 1 \right\}^N \, \Bigg| \, \sum_{i=1}^N x_{i} \leq C \right\},
\end{align}
where $\N$ is the set of unit-sized files (library) and $C$ is the cache capacity (in file units); and $x_i\!=\!1$ decides to cache file $i\!\in\! \N$. Interestingly, despite having to operate on this non-convex set, this approach yields \emph{in expectation} the same regret bounds as OFTRL for continuous caching \cite{ocol}.

\begin{algorithm}[t]
	\begin{footnotesize}
		\caption{\small{Optimistic Follow The Regularized Leader (\texttt{OFTRL-Cache})}}
		\label{alg:oftrl}
		\nl \textbf{Input}: $\sigma=1/\sqrt{C}$, $\delta_1=\|\theta_1-\ti\theta_1\|_2^2$, $\sigma_1=\sigma \sqrt{\delta_1}$, $x_1=\arg\min_{x \in \mathcal X} \dtp{x}{\theta_1}$ \\%
		\nl \textbf{Output}: $\{x_t \in \mathcal{X}\}_T$ \hfill$//$ \emph{Feasible discrete caching vector at each slot}\\%
		\nl \For{ $t= 2, 3\ldots$ }{
			\nl $\ti\theta_{t} \leftarrow \texttt{ prediction}$ \hfill $//$\emph{Obtain request prediction for slot $t$}\\
			\nl $\hat x_t = \argmax_{x \in \convX} \left\{ - r_{1:t-1} (x) + \dtp{x}{\Theta_{t-1}+\ti\theta_{t}}\right\}$\label{algstep:update1} \hfill$//$ \emph{Update the continuous cache vector}\\
			\nl $x_t\leftarrow MadowSample(\hat x_t)$  \hfill$//$ \emph{Obtain the discrete cache vector using Algorithm \ref{alg:ms}}  \\
			\nl $\Theta_t = \Theta_{t-1} + \theta_t$  \hfill$//$ \emph{Receive $t$-slot request and update total gradient} \\
			\nl  $\sigma_{t} = \sigma \left( \sqrt{\delta_{1:t}} -  \sqrt{\delta_{1:t-1}} \right)$ \hfill$//$ \emph{Update the regularization parameter} \\
		}
	\end{footnotesize}	
\end{algorithm}

Let us define the prediction error at slot $t$ as $\delta_t \triangleq \err_2^2 $, and introduce the proximal $\sigma_t$-strongly convex regularizer w.r.t. the Euclidean $\ell_2$ norm:
\begin{align}
    r_t(x) = \frac{\sigma_t}{2} \| x- x_t \|_2^2. \label{oftrl-reg}
\end{align}
Following \cite{naram-jrnl}, we define parameters $\{\sigma_t\}_t$ using the accumulated prediction errors, namely:
\begin{align}
\sigma_1 = \sigma\sqrt{\delta_1}, \quad\quad\quad \sigma_t = \sigma \left( \sqrt{\delta_{1:t}} -  \sqrt{\delta_{1:t-1}} \right) \quad \forall t\geq2, \quad \text{with} \quad \sigma=1/\sqrt C. \label{eq:ftrl-reg}
\end{align}

The basic OFTRL update stems from using these regularizers in the FTRL update formula. Namely, at each slot $t$ we update the cache to maximize the aggregated utility. This maximization is \emph{regularized} through a term (the above-defined regularizers) that depends on the predictor's accuracy.

The detailed steps are summarized in Algorithm \ref{alg:oftrl}. In the first iteration we draw randomly a feasible caching vector $x_1$ and observe the prediction error $\delta_1=\|\theta_1 - \ti \theta_1\|_2^2$. In each iteration we need to solve a strongly convex program (line 5) which returns the continuous caching vector $\hat x_t$, that is transformed to a feasible discrete $x_t$ (line 6) using Madow's Sampling (see Appendix \ref{appendix:madow}). The algorithm notes the new gradient vector, by simply observing the next request\footnote{For modeling convenience, we define the time slots to be the (non-uniform) time intervals that receive only one request. Our analysis can be readily extended to a bounded number of requests per slot.}, and updates the accumulated gradient $\Theta_t$ (line 7). The regret guarantee of Algorithm \ref{alg:oftrl} is described next.

\begin{theorem}
Algorithm \ref{alg:oftrl} ensures, for any time horizon $T$ and $N \geq 2C$, the expected regret bound:
\begin{align}
\mathbb E[R_T]
\leq
2\sqrt{C}
\sqrt{\sumT \err_2^2}  \notag
\end{align}
\label{thm:oftrl-m}
\vspace{-2mm}
\end{theorem}

\begin{proof}
 We define first the regret w.r.t. the continuous actions $\{\hat x_t\}_T $ as $\hat R_T \triangleq \dtp{\Theta_T}{x^\star} - \sumT \dtp{\theta_t}{\hat x_t}$, where $x^\star$ is the optimal-in-hindsight caching vector\footnote{This benchmark remains unchanged if we switch from the continuous to the discrete space.}. We also define the scaled Euclidean $\ell_2$ norm $\|\cdot\|_{(t)} = \sqrt{\sigma_{1:t}} \ \|\cdot\|_2$ so that $r_{1:t}$ is 1-strongly convex w.r.t $\|\cdot\|_{(t)}$, and note that its dual norm is  $\|\cdot\|_{(t),\star} = \frac{1}{\sqrt{\sigma_{1:t}}}\ \| \cdot \|_2$. Our starting point is \cite[Lem. 1]{naram-jrnl}, which we restate below:
\begin{lemma}
Let $r_{1:t}$ be a 1-strongly convex w.r.t. a norm $\|\cdot\|_{(t)}$. Then, the OFTRL iterates produced by line $5$ in Algorithm \ref{alg:oftrl} guarantee the bound $    \hat R_T \leq r_{1:T}(x^\star) + \frac{1}{2} \sumT \err ^2_{(t),\star}$.

\end{lemma}
Now, we first get a deterministic regret bound on $\hat x_t$. Assuming that\footnote{$N$ is typically orders of magnitude higher than $C$. In the cases where this does not hold the current analysis is still valid but can be improved by using the tighter diameter $\sqrt{2(N-C)}$.} $C\in (0,N/2]$, we can bound the $\ell_2$ diameter of the caching set as $\|x^\star - x_t\|_2 \leq \sqrt{2C}, \forall x^\star, x_t \in \convX$. Thus, we can upper-bound the regularizers in \eqref{oftrl-reg}, replace in the above Lemma and telescope to get:
\vspace{-1mm}
\begin{align}
    \hat R_T \leq \sigma_{1:T}\ C  + \frac{1}{2} \sumT \frac{\delta_t}{\sigma_{1:t}}. \label{eq:bnd_1}
\end{align}
Observing that the sum $\sigma_{1:t}$ telescopes to $\sigma \sqrt{\delta_{1:t}}$,  we can substitute it in \eqref{eq:bnd_1} and use the standard identity \cite[Lem. 4.13]{orabona2021modern} to bound the second term via $\sumT \delta_t / \sqrt{\delta_{1:t}}\leq 2\sqrt{\delta_{1:T}}$. Therefore, we obtain: 
$\hat R_T \leq \sigma \sqrt{\delta_{1:T}}\ C  + \frac{1}{\sigma} \sqrt{\delta_{1:T}}$, and by setting the parameter $\sigma$ to its optimal value $\sigma = 1/\sqrt{C}$, we can eventually write: 
\vspace{-1mm}
\begin{align}
\label{eq:hat_regret}
\hat R_T \leq 2\sqrt{C} \sqrt{\sumT \err_2^2}.
\end{align}
\vspace{-1mm}

The last step requires  Madow's sampling (line 6). By construction, the routine selects $C$ files and hence returns a feasible integral caching vector $x_t$ (or, \emph{sampled vector}). In addition, each item is included in the sampled vector with a probability based on the continuous $\hat x_t$. Namely, it holds $\Pr(x_{i}^t = 1) = \pi_i - \pi_{i-1} = \hat x_{i}^t$, where the auxiliary parameter $\pi_i$ aggregates the (interpreted as) probabilities for caching the first $i$ files, i.e.,  $\pi_i = \sum_{k=0}^i \hat x_k$. Since each $x_{i}^t$ is binary, it holds $\mathbb E[x_t] \!=\! \Pr(x_{i}^t \!= 1) \!=\! \hat x_t$. The result follows by using \eqref{eq:hat_regret} and observing:
\begin{align}
\vspace{-1mm}
\expec{}{R_T\left(\{x\}_T\right)} =  \dtp{\Theta_T}{x^\star} - \expec{}{\sumT \dtp{\theta_t}{x_t}}= \hat R_T.  
\end{align}
\end{proof}
\vspace{-1mm}

\textbf{Discussion.} The bound in Theorem \ref{thm:oftrl-m} ensures the desirable prediction-based modulation of the algorithm's performance, as the achieved regret \emph{shrinks} with the prediction quality. If all predictions are accurate, we get $R_T \leq 0$; when all predictions fail, we get  $R_T \leq 2 \sqrt{2CT}$. That is, in the worst scenario (e.g., when the predictions are created by an adversary) the regret bound is worse by a constant factor of $\sqrt{2}$ compared to the FTRL algorithm that does not use predictions \cite[Sec. 3.4]{mcmahan-survey17}), and $\sim 5$ compared to the lower bound derived in Sec. \ref{sec:lb}. Moreover, due to selecting an $\ell_2$ regularizer, the bounds are dimension-free and do not depend on the library size $N$. This is particularly important since in caching problems oftentimes the library size is an even bigger concern than the time horizon. Finally, note that the algorithm does not need to know the horizon $T$ beforehand. The drawback of this optimistic caching approach is the computational complexity of the iteration (line 5)  which involves a projection operation. While $\ell_2$ projections have received attention \cite[Sec. 7]{hazan-book}, they can hamper the scalability of the algorithm under certain conditions\footnote{For instance, this can be a bottleneck if the library size is extremely large, while the slot duration is very short and the available computation power is limited.}.  In the following section we show how perturbation-based smoothing can avoid the projection step. 

%% file: OFTPL.tex
\section{Caching through Optimistic Perturbations (\texttt{OFTPL-Cache})} \label{sec:oftpl}
We propose next a new OFTPL algorithm that significantly improves previous OFTPL proposals \cite{suggala-1,suggala-2}, both in terms of their bounds and implementation, and as such is of independent interest with potential applications that extend beyond caching to other $k-$set structured problems such as those discussed in \cite{cohen15}. The improvement is possible by setting the perturbation parameters $\eta_t$ in a manner that is adaptive to prediction error witnessed until $t-1$.

\begin{algorithm}[t]
	\caption{\small{Optimistic Follow The Perturbed Leader (\texttt{OFTPL-Cache})}}
	\label{alg:oftpl}
	\begin{footnotesize}
		\nl \textbf{Input}: $\eta_1=0$, $y_1=\arg\min_{y \in \mathcal X} \dtp{y}{\theta_1}$\\%
		\nl \textbf{Output}: $\{y_t \in \mathcal{X}\}_T$ \hfill$//$ \emph{Feasible discrete caching vector at each slot} \\%
		\nl  $\gamma \stackrel{iid}\sim \mathcal{N}(0,\boldsymbol{1}_{N\times 1})$ \hfill$//$ \emph{Sample a perturbation vector}\\
		\nl \For{ $t=2,3,\ldots$  }{
		    \nl $\ti\theta_{t} \leftarrow \texttt{ prediction}$ \hfill $//$\emph{Obtain request prediction for slot $t$}\\
			\nl $\eta_t = \frac{1.3}{\sqrt{C}}\left(\frac{1}{\ln (Ne/C)}\right)^{\frac{1}{4}}\sqrt{\sum_{\tau=1}^{t-1}\|\theta_\tau - \ti\theta_\tau\|^2_1}$ \hfill$//$ \emph{Update the perturbation parameter}  \\
			\nl  $y_t = \argmax_{y \in \mathcal{X}}\dtp{y}{\Theta_{t-1} + \ti\theta_t + \eta_{t}\gamma}$ \hfill$//$ \emph{Update the discrete cache vector} \\
			\nl $\Theta_t = \Theta_{t-1} + \theta_t $ \hfill$//$ \emph{Receive the request for slot $t$ and update total gradient}
		}
	\end{footnotesize}
\end{algorithm}

Following the discussion in Sec. \ref{sec:related}, we remind the reader that the FTPL actions are derived by solving in each slot $t$ a linear program (LP) with a parameterized perturbed cumulative utility vector, $\Theta_{t-1}+\eta_t\gamma$, where $\eta_t\in\mathbb R_+$ is the perturbation parameter. In order to obtain the optimistic FTPL variant we introduce two twists: \emph{(i)} the prediction for the next-slot utility $\ti \theta_t$ is added to the cumulative utility; and \emph{(ii)} the perturbation parameter $\eta_t$ is scaled according to the accumulated prediction error. Interestingly, due to the structure of the decision set $\mathcal{X}$, the LP solution reduces to identifying the $C$ files with the highest coefficients. This step can be efficiently implemented in deterministic linear $(O(N))$ time using, e.g., the Median-of-Medians algorithm \cite{cormen2022introduction}. The steps of the proposed scheme are presented in Algorithm \ref{alg:oftpl}, where we denote the $t$-slot OFTPL decisions with $y_t \in \mathcal{X}$. The following theorem characterizes the performance of this new OFTPL algorithm.

	\begin{theorem}
	    Algorithm \ref{alg:oftpl} ensures, for any time horizon $T$ and $N\geq2C$ with $C\geq11$, the expected regret bound:

		\begin{align}
			\mathbb{E}_\gamma[R_T] \leq 3.68 \sqrt{C}\ \bigg(\ln \frac{Ne}{C}\bigg)^{1/4} \sqrt{\sum_{t=1}^T ||\theta_t - \tilde{\theta}_t||_1^2}. 
		\end{align}
		\label{thm:oftpl}
	\end{theorem}

\begin{proof}
We consider the following \emph{baseline} potential function $\Phi(\theta) \triangleq {\max_{y \in \mathcal{X}}\dtp{y}{\theta }}$, which is a sub-linear function\footnote{A function $f$ is sub-linear if it is sub-additive (i.e., $f(a)+f(b)\geq f(a+b)$, which implies $f(a)-f(b)\leq f(a-b)$), and positive homogeneous (i.e., $f(\lambda a) = \lambda f(a), \lambda>0)$.}. The associated \emph{Gaussian smoothed} potential function for each $t$, is defined as:
\begin{align}
	\Phi_t(\theta) \triangleq \expec{\gamma\sim 
	\mathcal{N}(0, I)}{\max_{y \in \mathcal{X}}\dtp{y}{\theta+\eta_t\gamma}} = \expec{\gamma}{\Phi(\theta+\eta_t\gamma)}     
\end{align}
Clearly, $\Phi_t(\theta)$ is convex in $\theta$. Recall that the cumulative file request vector is defined as $\Theta_t = \Theta_{t-1}+ \theta_t.$ A Taylor expansion of $\Phi_t(\cdot)$ around the point $\Theta_{t-1}+\ti\theta_t$, evaluated at $\Theta_{t}$, with a second order remainder is: 
\begin{multline}
\Phi_t(\Theta_t) = \Phi_t(\Theta_{t-1}+\ti\theta_t) + \dtp{\grd\Phi_t(\Theta_{t-1}+\ti\theta_t)}{\theta_t-\ti\theta_t}
+\frac{1}{2}\dtp{\theta_t-\ti\theta_t}{\grd^2\Phi_t(\hat\theta_t)\ (\theta_t-\ti\theta_t)}, \label{eq:te} 
\end{multline}
where $\hat\theta_t$ is a point on the line segment connecting $\Theta_t$ and $\Theta_{t-1} + \ti\theta_t$. From the convexity of $\Phi_t(\cdot)$: 
\begin{align}
\Phi_t(\Theta_{t-1}+\ti\theta_t) \leq  \Phi_t(\Theta_{t-1}) + \dtp{\grd\Phi_t(\Theta_{t-1}+\ti\theta_t)}{\ti\theta_t}. \label{eq:convex}
\end{align}
From \eqref{eq:te} and \eqref{eq:convex}, we can eventually write:
\begin{align}
\Phi_t(\Theta_t) \leq \Phi_t(\Theta_{t-1}) + \dtp{\grd\Phi_t(\Theta_{t-1}+\ti\theta_t)}{\theta_t}
+\frac{1}{2}\dtp{\theta_t-\ti\theta_t}{\grd^2\Phi_t(\hat\theta_t)(\theta_t-\ti\theta_t)}. \label{eq:te2}
\end{align}
	
Now, note that it holds $\grd\Phi_t(\Theta_{t-1}+\ti\theta_t) =$
\begin{align}
\grd \expec{\gamma}{\Phi_t(\Theta_{t-1}+\ti\theta_t)}
\stackrel{(a)}= \expec{\gamma}{\grd\Phi_t(\Theta_{t-1}+\ti\theta_t)}= \expec{\gamma}{\arg\max_{y\in\mathcal{X}}\dtp{y}{\Theta_{t-1}+\ti\theta_t+\eta_t\gamma}} = \expec{\gamma}{y_t}, \label{eq:grad_is_expec}
\end{align}
where $(a)$ stems from \cite[Prop. 2.2]{bertsekas_stochastic}. Thus, $\eqref{eq:te2}$ can be written as: 
\begin{equation}
\Phi_t(\Theta_t) \leq \Phi_t(\Theta_{t-1}) + \expec{\gamma}{\dtp{\theta_t}{y_t}}+\frac{1}{2}\dtp{\theta_t-\ti\theta_t}{\grd^2\Phi_t(\hat\theta_t)(\theta_t-\ti\theta_t)}.
\end{equation}
Subtracting $\Phi_{t-1}(\Theta_{t-1})$ from both sides and telescoping over $T$ and setting $\eta_0 = 0$, we get: 
\begin{align}
\Phi_{T}(\Theta_{T}) \leq \sumT \left( \Phi_t(\Theta_{t-1}) - \Phi_{t-1}(\Theta_{t-1}) + \expec{\gamma}{\dtp{\theta_t}{y_t}} +\frac{1}{2}\dtp{\theta_t-\ti\theta_t}{\grd^2\Phi_t(\hat\theta_t)(\theta_t-\ti\theta_t)} \right).
\end{align}
Then, by Jensen's inequality: $\max_{y\in \mathcal{X}} \expec{\gamma}{\dtp{y}{\Theta_T+\eta_T\gamma}} = \max_{y\in \mathcal{X}} {\dtp{y}{\Theta_T}} = \Phi(\Theta_{T}) \leq \Phi_{T}(\Theta_T)$, and writing the last term as the norm of the vector $(\theta_t - \ti\theta_t)$ induced by the symmetric
positive semidefinite matrix $\grd^2\Phi_t(\hat\theta_t) \triangleq H_t$, we get the following upper bound of the regret:
\begin{align}
R_T \leq \Phi(\Theta_T) - \sumT \expec{\gamma}{\dtp{\theta_t}{y_t}} \leq \sumT \left(\Phi_t(\Theta_{t-1}) - \Phi_{t-1}(\Theta_{t-1}) + \frac{1}{2} \| \theta_t - \ti\theta_t \|_{H_t}\right). \label{eq:regret_2parts}
\end{align}
We now bound the first term in the RHS of inequality \eqref{eq:regret_2parts}:
\begin{align}
&\sumT\! \Phi_{t}(\Theta_{t-1}) \!-\! \Phi_{t-1}(\Theta_{t-1}) = \sumT\! \expec{\gamma}{\Phi(\Theta_{t-1}\!+\eta_t\gamma) - \Phi\left(\Theta_{t-1}\!+\eta_{t-1}\gamma\right)}\stackrel{(a)}{\leq}\!\sumT \expec{\gamma}{ \Phi\left((\eta_t \!- \eta_{t-1})\gamma\right)}\notag \\ 
&\stackrel{(b)}{\leq}\sumT (\eta_t - \eta_{t-1})\expec{\gamma}{\Phi(\gamma)} {\leq}\eta_T\expec{\gamma}{\max_y\dtp{y}{ \gamma}} \stackrel{(c)}{\leq} \eta_T \sqrt{2C \ln\binom{N}{C}} \stackrel{(d)}{\leq} \eta_T C \sqrt{2\ln(Ne/C)}, \label{eq:sa3}
\end{align}
where inequalities $(a)$ and $(b)$ follow from the sub-linearity of the potential function; $(c)$ from Massart's lemma which gives an upper bound the expected sum of the top $C$ elements in a Gaussian random vector (e.g., \cite[Lem. 9]{cohen15}); and finally $(d)$ is due to $\binom{N}{C} \leq (\frac{Ne}{C})^C$.
	
We now upper bound the second term in the RHS of \eqref{eq:regret_2parts}. From \cite[Eqn. (4)]{cohen15}, the $(i,j)$\textsuperscript{th} entry of the Hessian matrix is given by $H^t_{i,j}= \frac{1}{\eta_t}\mathbb{E}\big[ \tilde{y}(\hat{\theta}_t+\eta_t \gamma)_i \gamma_j\big],$ where $\tilde y (\cdot) = \arg\max_{y \in \mathcal{X}} \dtp{y}{\cdot}$. Hence, we have the following bound on the absolute value of each entry:
 \begin{align}
|H^t_{i,j}| = \frac{1}{\eta_t} \left|\expec{\gamma}{\tilde y (\hat \theta + \eta_t \gamma)_i \gamma_j} \right| \leq \frac{1}{\eta_t}
\expec{\gamma}{|\tilde y (\hat \theta + \eta_t \gamma)_i|\ |\gamma_j|}  \leq
\frac{1}{\eta_t} \expec{\gamma}{|\gamma_i|} \leq \frac{1}{\eta_t}\sqrt\frac{2}{\pi}, \label{hess-bd}
\end{align}
where the first inequality follows from Jensen's inequality, the second holds since $\tilde y_i = \{0,1\}$; and the last one is a property of Gaussian r.v.s. Thus each of the quadratic forms on the RHS of Eqn. \eqref{eq:regret_2parts} can be bounded as follows: 

 \begin{align}
&\| \theta_t - \ti\theta_t \|_{H_t} = \dtp{\theta_t-\ti\theta_t}{H_t(\theta_t-\ti\theta_t)} =  \sum_{i,j} (\theta_{t,i}- \tilde{\theta}_{t,i})H^t_{ij}(\theta_{t,j}- \tilde{\theta}_{t,j}) \\
&\stackrel{(a)}{\leq} \sum_{i,j} |(\theta_{t,i}- \tilde{\theta}_{t,i})||H^t_{ij}||(\theta_{t,j}- \tilde{\theta}_{t,j})| \stackrel{(b)}{\leq} \frac{1}{\eta_t}\sqrt{\frac{2}{\pi}} (\sum_i |(\theta_{t,i}- \tilde{\theta}_{t,i})|)^2=  \frac{1}{\eta_t}\sqrt{\frac{2}{\pi}} ||\theta_t - \tilde{\theta}_t||_1^2.
\label{eq:hess_b_1}
\end{align}
where $(a)$ follows from the triangle inequality and $(b)$ from the bound \eqref{hess-bd}.

Another way to bound $\|\theta_t - \ti\theta_t \|_{H_t}$, which will be useful later\footnote{This second bound on the norm enables us to set $\eta_t$ parameters based solely on the prediction error witnessed so far $\sum_{\tau=1}^{t-1}\| \theta_{\tau} - \ti\theta_{\tau}\|$. Consequently, the regret will depend solely on a scaled prediction error (without additive constants).}, starts from \eqref{eq:te} to get:
 \begin{align}
 &\frac{1}{2} \| \theta_t - \ti\theta_t \|_{H_t}= \Phi_t(\Theta_t) - \Phi_t(\Theta_{t-1}+\ti\theta_t) - \dtp{\grd \Phi_t(\Theta_{t-1}+\ti\theta_t)}{\theta_t-\ti\theta_t} \notag \\
 &= \expec{\gamma}{\Phi(\Theta_{t}+\eta_t\gamma) - \Phi(\Theta_{t-1}+ \ti\theta_t +\eta_t\gamma)} +\dtp{\grd \Phi_t(\Theta_{t-1}+\ti\theta_t)}{\ti\theta_t-\theta_t}\stackrel{(a)}{\leq} \Phi(\theta_t - \ti\theta_t ) + \notag \\ 
 &+ \dtp{\grd \Phi_t(\Theta_{t-1}+\ti\theta_t)}{\ti\theta_t-\theta_t}= \max_{y \in \mathcal{X}}\dtp{y}{\theta_t - \ti\theta_t } +  \dtp{\grd \Phi_t(\Theta_{t-1}+\ti\theta_t)}{\ti\theta_t-\theta_t} \stackrel{(b)}\leq 2 \err_1, \label{eq:hess_b_2}
\end{align}
where $(a)$ follows from the sub-additivity of $\Phi(\cdot)$, and in $(b)$ we use that $y_i \in \{0,1\}, \forall i$ and bounded both terms using triangle inequality.
Hence, combining the bounds \eqref{eq:hess_b_1} and \eqref{eq:hess_b_2}, we get: 
\begin{align}
     \frac{1}{2} \| \theta_t - \ti\theta_t \|_{H_t} \leq \min\Big(\frac{1}{\sqrt{2\pi}} \frac{||\theta_t-\tilde{\theta_t}||_1^2}{\eta_t}, 2||\theta_t - \tilde{\theta_t} ||_1\Big).
\end{align}
Now we choose the learning rate $\eta_t = \beta \sqrt{\sum_{\tau=1}^{t-1} ||\theta_\tau - \tilde{\theta}_\tau||_1^2}, t\geq 1$ for some constant $0<\beta \leq \frac{1}{\sqrt{2\pi}}$ that will be specified later. Hence, we have:
\begin{align}
   \!\!\!\!\!  \frac{1}{2} \| \theta_t \!- \ti\theta_t \|_{H_t} \!\leq\!  \min\left( \frac{||\theta_t \!-\tilde{\theta_t}||_1^2}{\sqrt{2\pi}\beta\sqrt{\sum_{\tau=1}^{t-1} ||\theta_\tau \!- \tilde{\theta}_\tau||_1^2}}, 2\frac{||\theta_t - \tilde{\theta_t} ||^2_1}{\sqrt{2\pi}\beta\sqrt{||\theta_t \!- \tilde{\theta_t} ||^2_1}}\right) \!\stackrel{(a)}{\!\leq\!} \frac{3}{\sqrt{2\pi}\beta}\frac{||\theta_t \!- \tilde{\theta_t} ||^2_1}{\sqrt{\sum_{\tau=1}^{t} ||\theta_\tau \!- \tilde{\theta}_\tau||_1^2}} \!\!\label{eq:hess_b_3}
\end{align}
where in $(a)$, we used the fact that $\min (a_1/a_2, b_1/b_2) \leq \frac{a_1+a_2}{b_1+b_2}$ for any two positive fractions and $\sqrt{x+y}\leq \sqrt{x}+ \sqrt{y},$ for any non-negative $x$ and $y$'s. 

Now that we have a bound for the smoothing-overhead in \eqref{eq:sa3}, and the per-step regret bounds \eqref{eq:hess_b_3}, we can substitute them in \eqref{eq:regret_2parts} to get:

\begin{align}
	\expec{\gamma}{R_T} \leq {\eta_T C \sqrt{2\log(Ne/C)}} + \frac{3}{\sqrt{2\pi}\beta}{\sumT  \frac{||\theta_t - \tilde{\theta_t} ||^2_1}{\sqrt{\sum_{\tau=1}^{t} ||\theta_\tau - \tilde{\theta}_\tau||_1^2}}}. \label{reg-bd}
\end{align}
The second term above can be upper-bounded as:
\begin{align}
    \sum_{t=1}^T \frac{||\theta_t - \tilde{\theta_t} ||^2_1}{\sqrt{\sum_{\tau=1}^{t} ||\theta_\tau - \tilde{\theta}_\tau||_1^2}} 
    &\leq \sum_{t=1}^T \int_{\sum_{\tau=1}^{t-1} ||\theta_\tau - \tilde{\theta}_\tau||_1^2}^{\sum_{\tau=1}^{t} ||\theta_\tau - \tilde{\theta}_\tau||_1^2} \frac{dx}{\sqrt{x}} 
    =  \int_{0}^{\sum_{\tau=1}^{T} ||\theta_\tau - \tilde{\theta}_\tau||_1^2} \frac{dx}{\sqrt{x}}=2 \sqrt{\sum_{\tau=1}^T ||\theta_\tau - \tilde{\theta}_\tau||_1^2}. \label{hess-bd-final}
\end{align}

Substituting the above bound into \eqref{reg-bd} and using the definition of $\eta_T$ we get the regret upper bound:
\begin{align}
    \mathbb{E}_\gamma[R_T] \leq  \sqrt{\sum_{\tau=1}^T ||\theta_\tau - \tilde{\theta}_\tau||_1^2} \bigg(C \sqrt{2 \ln (Ne/C)}) \beta + \frac{6}{\beta \sqrt{2\pi}} \bigg) \label{eq:oftpl-regret-beta}
\end{align}
Optimizing over the constant $\beta$, we get that $\beta = \sqrt{\frac{3}{C}}(\frac{1}{\pi\ln(Ne/C)})^{\nicefrac{1}{4}}$, $C\geq11$ (Recall we that $0<\beta\leq 1/\sqrt{2\pi}$). Substituting this valud for $\beta$ back in \eqref{eq:oftpl-regret-beta} we arrive at the result.\end{proof}

\textbf{Discussion.} Similarly to Theorem \ref{thm:oftrl-m}, the regret bound here is modulated with the quality of predictions: it collapses to zero when predictions are perfect, and maintains $R_T \!\leq\! 3.68\ \ln \left(\frac{Ne}{C}\right)^{1/4} \sqrt{2CT} $ for arbitrary bad predictions. Interestingly, this worst-case scenario is only inferior by a factor of $\sim 2.5$ compared to the recent FTPL algorithm in \cite{abhishek-sigm20} that does not use (and cannot benefit from) predictions. Hence, incorporating predictions (even, of unknown quality) comes without cost in the proposed OFTPL algorithm. Furthermore, for the more common case of predictions within a certain range of error, the regret bounds diminish in proportion to their quality.

Comparing with Theorem \ref{thm:oftrl-m}, OFTPL achieves regret bounds worse by a factor of $\sim 1.9 \ln\left(\frac{Ne}{C}\right)^{1/4}$, which depends on $N$, albeit in a small order. The regret bounds are also different in nature, for OFTRL it is $\ell_2$ squared whereas for OFTPL, it is the much larger $\ell_1$ squared of the prediction error. On the other hand, Algorithm \ref{alg:oftpl} does not involve the expensive projection operation that appears in Algorithm \ref{alg:oftrl}, but rather a simple quantile-finding operation (top $C$ files) with a worst-case complexity of $O(N)$. This facilitates greatly the implementation of OFTPL in systems with low computing capacity or in applications that require decisions in near real-time. A notable point about Theorem  \ref{thm:oftpl} is that in the special case where the predictor $\tilde{\theta_t}$ suggests a single file (i.e., deterministic), the regret scales with the square root of the \emph{number of mistakes} as opposed to the conventional \emph{number of time slots} (i.e., horizon $T$) in previous no regret discrete caching works\footnote{The same property of depending on prediction mistakes rather than $T$ holds for Theorem \ref{thm:oftrl-m} but the regret scales as the square root of double the number of mistakes, due to the use of $\ell_2$ norm.}\cite{abhishek-sigm20, tareq-jrnl, leadcache}. It is also interesting to note that the bounds of Theorem \ref{thm:oftpl}, and \ref{thm:oftrl-m}, provide insights on the appropriate loss function to be optimized by the prediction oracle (squared norms). This is helpful while  training and tuning machine learning models on request traces detests. Lastly, we note that we can sample a fresh perturbation vector at each time step in order to handle adaptive adversaries (i.e., adversaries that do not fix the cost sequence in advance but react to the choices of the algorithm). The same analysis applies since perturbations are equal in distribution. We refer the reader to e.g., \cite[Sec. 8]{JMLR:v6:hutter05a} for techniques for reducing guarantees from oblivious to adaptive adversaries via fresh-sampling.

%% file: OFTPL_US.tex
\section{Caching files with Arbitrary Sizes} \label{sec:arb-sizes}

While the caching problem with equal-sized files has been studied using regret analysis and competitive analysis, to the best of the authors' knowledge, there are no results for the more challenging case of files with different sizes. This section fills this gap by extending the above tools accordingly. In particular, we consider the setting where each file $i\in \N$ has a size of $s_i$ units, $s_i \leq C$. Hence, the set of feasible caching vectors needs to be redefined as:
\begin{align}
	\mathcal X_s=  \left\{ {x} \in\left\{0, 1 \right\}^N \bigg| \sum_{i=1}^N s_ix_{i} \leq C \right\},
\end{align}
where the caching decisions are calibrated with the respective file sizes in the capacity constraint. And similarly, the benchmark (designed-in-hindsight) policy is redefined as $x^\star=y^\star\triangleq\argmax_{x\in \mathcal X_s}\langle x, \Theta_T \rangle$. We present two solution approaches for this problem, using both OFTRL and OFTPL. These results are of independent interest with applications beyond caching. 

\subsection{Approximate OFTPL} \label{sec:oftpl-arb}

Similarly to Algorithm \ref{alg:oftpl}, the OFTPL algorithm in this case determines the next cache configuration $y_{t}$ by solving the following integer programming problem at each round $t$:
\begin{align} 
\mathbb{P}_1:\qquad	\max_{y \in \mathcal{X}_s }\,\, \langle \Theta_{t-1}+\tilde{\theta_t} + \eta_t \gamma, y \rangle, \label{eq:oftpl-uneq-obj}
\end{align}
which is a Knapsack instance with profit vector $p = \Theta_{t-1}+\ti\theta_t+\eta_t\gamma$; size vector $s = \left(s_i, \forall i \in \N\right)$; and capacity $C$. Since the Knapsack problem is NP-Hard \cite{knapsack-book}, we cannot solve $\mathbb{P}_1$ efficiently (fast and accurately) at each slot, and hence it is not practical (or, even possible) to use the approach of Sec. \ref{sec:oftpl}. Instead, we resort here to an approximation scheme for solving $\mathbb P_1$ and, importantly, do so in a way that these approximately-solved instances do not accumulate an unbounded regret w.r.t. $y^\star$. This requires a tailored approximation analysis and to define a new regret metric.

In detail, we leverage Dantzig's approach for tackling packing problems \cite{dantzing}, to obtain an \emph{almost}-integral solution from the respective integrality-relaxed problem; and then recover, via a \emph{point-wise randomized rounding}, a fully-integral solution which, as we prove, keeps the long-term $\nicefrac{1}{2}$-approximate regret bounded. First, recall that the $\alpha$-approximate regret is defined as \citep{kalai2005efficient}:
\begin{eqnarray}
	R_T^{(\alpha)} \triangleq \alpha \dtp{\Theta_T}{x^\star} - \sumT \dtp{\theta_t}{x_t}, \label{eq:def-alpha-regret}
\end{eqnarray}
for a positive constant $\alpha$. This generalized regret metric allows to use a parameterized benchmark, in line with prior works, e.g., see \cite{leadcache} and references therein. Now, it is important to see that while the Knapsack problem admits an FPTAS \cite{vazirani2001approximation}, due to the online nature of our caching problem, not all $\alpha$-approximation schemes for the offline OFTPL problem provide an $\alpha$-approximate regret guarantee. In light of this, we employ the stronger notion of \emph{point-wise} $\alpha$-approximation scheme, which yields an $\alpha$-regret guarantee for the online learning problem \cite{kalai2005efficient}. We restate the definition:
\begin{definition}[$\alpha$-point-wise approximation]
\noindent A randomized $\alpha$-point-wise approximation algorithm $\mathcal{A}$ for a fractional solution $\hat y=(\hat y_i, i\in \N)$ of a maximizing LP with non-negative coefficients, is one that returns an integral solution $y=(y_i, i\in \N)$ such that $\expec{}{y_i} \geq \alpha \hat{y_i}, \forall i \in \N$ and some $\alpha>0$; where the expectation is taken over possible random choices made by algorithm $\mathcal{A}$.
\end{definition}
\noindent In our case, we set $\alpha=1/2$ and propose an $(\nicefrac{1}{2})$-point-wise approximation algorithm for $\mathbb{P}_1$. 

Our starting point is Dantzig's approach which operates on the integrality-relaxed version of $\mathbb P_1$. In particular, the integrality-relaxed LP for the Knapsack problem with profit vector $p$, weight vector $s$, and capacity $C$, through the following steps:

\vspace{1mm}
\underline{\texttt{Dantz}$(C, p, s)$:}
\begin{enumerate}
	\item Index files in decreasing profit-to-size ratios, \emph{i.e., } $({p_1}/{s_1}) \geq ({p_2}/{s_2}) \geq \ldots \geq ({p_N}/{s_N})$.

	\item	Set $k = \min \big\{j \,\big \vert \, \sum_{i=1}^j s_i >C\big\}$ and $\widetilde{C} = C - \sum_{i=1}^{k-1}s_i$.
	
	\item	Assign the continuous variables $\hat y_i\in[0,1], i\in\N$ as follows:
	\begin{eqnarray*}
		\hat{y}_i	= \begin{cases}
			1, ~~ \,\,\,\textrm{if}\,\,\,~ i\in [k-1]	\\
			\frac{\widetilde{C}}{s_k}, ~~ \textrm{if}\,\,\,~ i=k\\
			0, ~~ \,\,\,\textrm{otherwise}.
		\end{cases}
	\end{eqnarray*}
\end{enumerate}
To streamline presentation, we denote with \texttt{Dantz}$(C, p, s)$ the operation of steps (1)-(3) above on the Knapsack instance $(C, p, s)$, which return the solution $\hat y$ and the value of parameter $k$. An interesting property of returned solution $\hat{y}$, which we exploit in our randomized approximation scheme, is that at most one component of the vector $\hat{y}$ is non-integral.

\begin{algorithm}[t]
	\caption{\small{\texttt{OFTPL-UneqCache}}}
	\label{alg:oftpl-unequal}
	\begin{footnotesize}
	\nl \textbf{Input}: $\eta_1=0$, $y_1=\arg\min_{y \in \mathcal X_s} \dtp{y}{\theta_1}$, $s=(s_i, i\in \N)$\\ 
	\nl \textbf{Output}: $\{y_t\in\X_s\}_T$ \quad  \hfill$//$ \emph{Feasible discrete caching vector at each slot} \\%
	\nl $\gamma \stackrel{iid}\sim \mathcal{N}(0,\boldsymbol{1}_{N\times 1})$ \hfill$//$ \emph{Sample a perturbation vector}\\
	\nl \For{ $t=2,3,\ldots$  }{
        \nl $\ti\theta_{t} \leftarrow \texttt{ prediction}$ \hfill $//$\emph{Obtain request prediction for slot $t$}\\
        \nl $ \eta_t = \frac{1.3}{\sqrt{C}}\left(\frac{1}{\ln(Ne/C)}\right)^{1/4}\sqrt{\sum_{\tau=1}^{t-1}\|\theta_\tau - \ti\theta_\tau\|^2_1}$\hfill$//$ \emph{Update the perturbation parameter}  \\
        \nl $p\leftarrow \Theta_{t-1}+\ti\theta_t+\eta_t\gamma$ \hfill$//$ \emph{Update the profit vector}\\
		\nl $(\hat y_t, k)\leftarrow$ \texttt{Dantz}$(C,p,s)$ \hfill $//$ \emph{Compute the "almost integral" cache vector}\\
		\nl  $y_t \leftarrow$ \texttt{Rand}$(\hat y_t, k)$ \hfill $//$ \emph{Perform Randomized Rounding}\\
		\nl $\Theta_t = \Theta_{t-1} + \theta_t$ \hfill$//$ \emph{Receive $t$-slot request and update total grad}
	}
\end{footnotesize}
\end{algorithm}

The detailed steps of the proposed OFTPL scheme are presented in Algorithm \ref{alg:oftpl-unequal}. At each slot we obtain a new (probabilistic) prediction for the next requested file $\ti \theta_t$ (line 5) and update the perturbation parameter $\eta_t$ (line 6). Then we calculate the new profits $p_i=\Theta_{t-1}+\ti\theta_t+\eta_t\gamma$, $i\in \N$, and solve the relaxed Knaspack by invoking \texttt{Dantz}$(C, p, s)$ to obtain the almost-integral $\hat y_t$ and parameter $k$ (line 7). This vector has $k-1$ components equal to 1, one additional non-negative component, and $N-k$ components equal to 0. This solution is then rounded through the randomization scheme:

\vspace{1mm}
\underline{\texttt{Rand}$(\hat y_t, k)$:}
\begin{enumerate}
	\item Set $S={1,2,\ldots,k-1}\triangleq [k-1]$ with probability $\nicefrac{1}{2}$; Set $S = \{k\}$ with probability $\nicefrac{1}{2}$.
	\item Set $y^t_{i}=0, \forall i \in \N$; and update to $y^t_{i}=1$ for each $i\in S$.
\end{enumerate}
\vspace{1mm}
The $\texttt{Rand}$ operation is invoked and creates integral caching vector $y_t$ which satisfies the capacity constraint (line 9). Finally, we observe the new gradient, update the aggregate gradient vector and repeat the process (line 10). The following theorem, proved in Appendix. \ref{appendix:theorem-oftpl-unequal}, characterizes the guarantees of Algorithm \ref{alg:oftpl-unequal}.

	\begin{theorem}
		Algorithm \ref{alg:oftpl-unequal} ensures, for any time horizon $T$, the expected regret bound:
		\begin{align}
			\expec{}{R^{(\nicefrac{1}{2})}_T}\leq
			1.84\sqrt{C}\left(\ln{\frac{Ne}{C}}\right)^{1/4}
			\sqrt{\sumT \err_1^2}  \notag
		\end{align}

		\label{thm:oftpl-unequal}
	\end{theorem}

\textbf{Discussion.} The bounds of Theorem \ref{alg:oftpl-unequal} possess the desirable property of being modulated with the prediction errors, and in fact are improved by a factor of half compared to the equal-sizes bound. However, we remind the reader that the regret metric in this section is defined w.r.t. a weaker benchmark, i.e., a benchmark that achieves $\nicefrac{1}{2}$ the utility of the best-in-hindsight utility $\langle \Theta_T,y^\star \rangle$. Relying upon such approximations is an inevitable concession for NP-hard problems -- especially when having to solve them repeatedly. Interestingly, the computational complexity of Algorithm \ref{alg:oftpl-unequal} is comparable to that of Algorithm \ref{alg:oftpl}, which is quite surprising since we are able to handle integral caching decisions with arbitrary-sized files. To the best of the authors' knowledge, this work is the first to propose an OCO-based solution for this variant of the online caching problem. We proceed next to remove the effect of the library size on the regret bound.

%% file: OFTRL_US.tex
\subsection{Approximate OFTRL} \label{sec:oftrl-arb}

We introduce next an OFTRL algorithm that can handle arbitrary file sizes. To that end, we use the OFTRL update on the integrality-relaxed version of the problem, then modify the obtained vector so as to be almost integral, and finally employ the same randomized rounding technique as above to correct for feasibility. The new intermediate step (that yields the almost-integral vector) is necessary since, unlike in Sec. \ref{sec:oftpl-arb}, we cannot apply the $(\nicefrac{1}{2})$-sampling technique directly on $\{\hat x_t\}_t$ because these vectors do not have this required almost-integral form. Thus, we leverage the \texttt{DepRound} subroutine from \cite{depround}, which is known to achieve the useful property re-stated below. 
\begin{lemma}
\label{lemm:dep-round}
For $a \in [0, 1]^N$, the \texttt{DepRound} scheme in Algorithm \ref{alg:dr} returns a vector $b$ such that 
\begin{itemize}
    \item All elements of $b$ are integral except one: $b_i \in \{0, 1\}$ $\forall i \in \N / j$, $b_j \in [0,1]$.
    \item $b_i$ is an unbiased estimator of $a_i$: $\expec{}{b_i} = a_i, \forall i \in \N$.
    \item $b$ respects the linear constraints satisfied by $a$: $\Pr[\sum_{i \in\N} s_i\ a_i = \sum_{i \in\N} s_i\ b_i] = 1$.
\end{itemize}
\end{lemma}

With this almost-integral solution at hand, we re-use the sampling technique \texttt{Rand}, to achieve an $(\nicefrac{1}{2})$-Regret guarantee w.r.t. the same benchmark as in Sec. \ref{sec:oftpl-arb}. The steps are presented in Algorithm \ref{alg:oftrl_US}. The diligent reader will observe that essentially we merge steps from Algorithm \ref{alg:oftrl} and Algorithms \ref{alg:oftpl-unequal}. The detailed steps of the \texttt{DepRound} subroutine are presented in the Appendix. The next theorem, proved in the Appendix \ref{appendix:theorem-oftrl_US}, characterizes the algorithm's performance.

\begin{algorithm}[t]
\caption{\small{\texttt{OFTRL-UneqCache}}}
\label{alg:oftrl_US}
\begin{footnotesize}
\nl \textbf{Input}: $\sigma=1/\sqrt{C}$, $\delta_1=\|\theta_1-\ti\theta_1\|_2^2$, $\sigma_1=\sigma \sqrt{\delta_1}$, $x_1=\arg\min_{x \in \mathcal X_s} \dtp{x}{\theta_1}$ \\
\nl \textbf{Output}: $\{y_t\in\X_s\}_T$ \quad  \hfill$//$ \emph{Feasible discrete caching vector at each slot} \\%
\nl \For{ $t=2,3, \ldots$  }{
	\nl $\ti\theta_{t} \leftarrow \texttt{ prediction}$ \hfill $//$\emph{Obtain gradient prediction for slot $t$}\\	
    \nl $\hat x_t = \arg\max_{x \in \convXs} \left\{-r_{1:t-1} (x) + \dtp{x}{\Theta_{t-1}+\ti\theta_{t}}\right\}$\\
    \nl $\bar x_t \longleftarrow$ \texttt{DepRound}$(\hat x_t)$\hfill $//$ \emph{Compute the "almost integral" cache vector}\\
	\nl  $x_t \longleftarrow$ \texttt{Rand}$(\bar x_t,k)$ \hfill $//$ \emph{Perform Randomized Rounding}\\    
	\nl $\Theta_t = \Theta_{t-1} + \theta_t$ \hfill$//$ \emph{Receive $t$-slot request and update total grad}\\
    \nl  $\sigma_{t} = \sigma \left( \sqrt{\delta_{1:t}} -  \sqrt{\delta_{1:t-1}} \right)$ \hfill$//$ \emph{Update the regularization parameter} \\	
}
\end{footnotesize}
\end{algorithm}

\begin{theorem}\label{thm:oftrl_US}
Algorithm \ref{alg:oftrl_US} ensures, for any time horizon $T$, the expected regret bound:
\begin{align}
\expec{}{R_T^{(\nicefrac{1}{2})}}\leq
\sqrt{C}
\sqrt{\sumT \err_2^2}  \notag
\end{align}
\end{theorem}
\textbf{Discussion.} Compared to the approach we used to extend OFTPL to unequal-sized files, an additional rounding technique (\texttt{DepRound}) was necessary to extend OFTRL. The complexity of this sub-routine is linear in the library size. Thus, despite preserving the order-level complexity of Algorithm \ref{alg:oftrl_US}, handling such files increases the overhead to get the almost-integral caching vectors $\{\bar x_t\}_t$. The important point is that we recover the dimension-free bounds for the regret, which are better by a constant factor of $1.84 \left(\ln\frac{Ne}{C}\right)^{1/4}$ compared to \texttt{OFTPL-UneqCache}, at the expense of performing a (potentially challenging) projection step. Hence, one can select the most suitable method depending on the requirements of the application (size of library, slot length, etc.) and the computing resources when executing the algorithm. Lastly, in Algorithm \ref{alg:oftrl_US}, and all introduced algorithms for the single cache case, the request vector $\theta$ can be extended to include (time-varying and unknown) weights that depend on users or network properties. 

%% file: OFTRL_BP.tex
\section{Bipartite Caching Through Optimism (\texttt{OFTRL-BipCache})}\label{sec:OFTRL-BP}

Finally, we extend our study to caching networks where a set of edge caches ${\mathcal J}=\{1,2,\dots, J\}$ serves a set of users  ${\mathcal I}=\{1,2,\dots, I\}$ requesting files from the library $\mathcal N$. The connectivity between $\mathcal I$ and $\mathcal J$ is modeled with parameters $d=\big(d_{ij}\in \{0,1\}: i\!\in\!\mathcal{I}, j\in\mathcal{J} \big)$, where $d_{ij}=1$ if cache $j$ can be reached from user location $i$. Each user can be (potentially) served by any connected cache. This is a widely-studied non-capacitated bipartite model \cite{paschos-book, femtocaching}. 

We introduce the new caching variables $k$ and routing variables $u$. Namely, $k_{nj}^t\!\in\!\{0, 1\}$ decides whether file $n\!\in\!\mathcal N$ is stored at cache $j\!\in\! \mathcal J$ at the beginning of slot $t$, and the $t$-slot caching vector  $k_t\!=\!(k_{nj}^t: n\!\in\! \mathcal N, j\!\in\!\mathcal J)$ belongs to
$
\mathcal{K}=\left\{k\in \{0,1\}^{N\cdot J} ~\big|~ \sum_{n\in\mathcal{N}}k_{nj}\leq C_j, ~j\in \mathcal J\right\},
$
where $C_j$ is the capacity of cache $j\in \mathcal J$. We use the routing variable $u_{nij}^t\in\{0,1\}$ to decide the delivery of file $n$ to user $i$ from cache $j$, and define the $t$-slot routing vector $u_t=(u_{nij}^t\!\in\![0,1]: n\!\in\!\mathcal N, i\!\in\!\mathcal I, j\!\in\!\mathcal J)$ that is selected from the set:
$
\mathcal{U}=\left\{u\in \{0,1\}^{N\cdot J\cdot I} ~\big|~ \sum_{j\in\mathcal{J}}u_{nij}\leq 1, ~n\in \mathcal N, i\in\mathcal I\right\}.
$
Note also that unserved requests, i.e., when the summation is strictly smaller than $1$, are satisfied by the root cache. This option, however, yields zero benefit for the users (no cache-hit gains); see also \cite{femtocaching}. The request vector $\theta_t$ is redefined to reflect a request's source and destination: $\theta_t\! =\!(\theta_{ni}^t\!\in\!\{0,1\}: n\!\in\!\mathcal N, i\!\in\!\mathcal I)$, and is drawn from the set:
$
	\mathcal{Q}=\left\{\theta\in \{0,1\}^{N\cdot I} ~\big |~ \sum_{n\in\mathcal{N}}\sum_{i\in\mathcal{I}} \theta_{ni}=1\right\}. \notag
$

\begin{algorithm}[t]
	\caption{\small{\texttt{OFTRL-BipCache} }}
	\label{alg:oftrl-bp}
	\begin{footnotesize}
		\nl \textbf{Input}: $\sigma=1/\sqrt{1+JC}$, $\delta_1=\|\theta_1-\ti\theta_1\|_2^2$, $\sigma_1=\sigma \sqrt{\delta_1}$, $x_1=\arg\min_{x \in \mathcal X_b} \dtp{x}{\theta_1}$ \\%
		\nl \textbf{Output}: $\{x_t=(k_t, u_t) \in \X_b\}_T$\quad  \hfill$//$ \emph{Feasible discrete caching/routing vector at each slot}\\%
		\nl \For{ $t=2,3,\ldots$  }{
			\nl $\ti\theta_{t} \leftarrow \texttt{ prediction}$ \hfill $//$\emph{Obtain request prediction for slot $t$}\\	
			\nl $\hat x_t = \arg\max_{x \in \convXb} \left\{ - r_{1:t-1} (x) + \dtp{x}{\Theta_{t-1}+\ti\theta_{t}}\right\}$\label{algstep:updateb} \hfill$//$ \emph{Update the continuous policy vector}\\    
			\nl $(k_{j}^t) \longleftarrow MadowSample(\hat k_t)$, $\forall j\in\mathcal J$  \hfill$//$ \emph{Obtain the discrete cache vector for each cache independently}  \\    
			\nl $\Theta_t = \Theta_{t-1} + \theta_t$  \hfill$//$ \emph{Receive $t$-slot request and update total grad}\\
			\nl Set $u^t_{n i j'} \gets 1 $ for a randomly selected $j' \in \mathcal{J}^{ni}$ \hfill$//$\emph{Update the discrete routing vector}\\
            \nl  $\sigma_{t} = \sigma \left( \sqrt{\delta_{1:t}} -  \sqrt{\delta_{1:t-1}} \right)$ \hfill$//$ \emph{Update the regularization parameter} \\				
		}
	\end{footnotesize}
\end{algorithm}

We can now introduce the $t$-slot utility function:
\begin{equation}
    f_t(x_t) = \sum_{n\in\mathcal{N}}\sum_{i\in\mathcal{I}}\sum_{j\in\mathcal{J}} \theta^t_{ni}k_{nij}^t\ .
\end{equation}
where we abuse notation and redefine $x_t \triangleq (k_t, u_t)$. Therefore, the utility-maximizing caching-routing policy at each slot $t$ is found by solving the following problem:

\begin{align}
\mathbb P_2:\,\, \max_{x}  \,\,\,& \sum_{t=1}^T f_t(x)\,\,\,\,\,\,\,
\text{s.t.}\,\,\,\,\,\, u\in \mathcal U, \,\,\,\, k\in\mathcal K; \,\,\,\,u_{nij} \leq k_{nj}d_{ij},\, i\in\mathcal I, j\in\mathcal J, n\in\mathcal N,
\end{align} 
and we define the feasible caching/routing set as $\mathcal{X}_b\triangleq\left\{ \{\mathcal K \times \mathcal U\} \cap \{ u_{nij}\leq k_{nj}d_{ij}\} \right\}$. $\mathbb P_2$ is known to be NP-Hard via a reduction to the set cover problem \cite[Sec. 3]{femtocaching}, \cite[Sec. 4.1]{leadcache}. Hence, we will be using below also its integrality-relaxed version $\mathbb{P}_{2}^{\prime}$ with continuous variables $\hat x_t \triangleq (\hat k_t, \hat u_t)$.  

Our strategy for tackling $\mathbb P_2$ is to use OFTRL on the convex hull of $\mathcal{X}_b$ (essentially learning w.r.t. $\mathbb P_2^{\prime}$) to optimize $\hat x_t$, and then obtain discrete caching vectors with Madow's sampling applied to each cache separately. As last step we select a proper routing solution for the received request $\theta_t$. Namely, upon receiving a request for file $n$ from user $i$, the corresponding routing variable is set to $1$ if \emph{any} cache connected to $i$ stores file $n$. Thus, we define the auxiliary set $\mathcal{J}^{ni} = \left\{j \in \mathcal{J} \big|\ y_{nj}\ d_{ij} = 1\right\}$,
and assign $u^t_{nij'} = 1$ for the $(n, i)$ pair and some $j' \in \mathcal{J}^{ni}$. It is important to stress that in such uncapacitated models, the routing plan is directly determined once a caching vector is fixed\footnote{In other words, the routing variables are auxiliary in the Femtocaching model, and indeed in \cite{femtocaching} these variables where omitted, while they appear with different name in subsequent works, e.g., as virtual caching variables in \cite{leadcache}.}. The detailed steps of the proposed OFTRL schemed are presented in Algorithm \ref{alg:oftrl-bp}, where we reuse the regularization scheme from Sec. \ref{sec:oftrl-m} with the difference that it operates now on the newly defined variables and request vectors. The performance of the algorithm is characterized with the next theorem, the proof of which can be found in Appendix. \ref{appendix:oftrl-m-bp}
\begin{theorem}\label{thm:oftrl-m-bp}
Algorithm \ref{alg:oftrl-bp} ensures, for any horizon $T$, the expected $(1-\nicefrac{1}{e})-$Regret bound:
\begin{align}
\expec{}{R_T^{(1-\nicefrac{1}{e})}}\leq
1.3 \sqrt{1+JC} \sqrt{\sumT\err_2^2}  \notag
\end{align}
\end{theorem}

 \textbf{Discussion.} Similar to the single cache, Theorem \ref{thm:oftrl-m-bp} improves the regret by removing the effect of library size $N$, as opposed to the recently-proposed bipartite OFTPL algorithm for equal-sized files in \cite{leadcache}\footnote{We note that the bound in \cite{leadcache} contains additionally the number of users since they consider a different request model of one request per user per time slot. Our work can be readily extended in that direction, as explained above.}. This improvement comes at the expense of a projection operation in the OFTRL step. Additionally, Algorithm \ref{alg:oftrl-bp} manages to reduce further the constant terms when it has access to high quality predictions for the next-slot requests.
 

%% file: experts.tex
\section{Expert-based Optimistic Caching}\label{sec:experts}
Changing tack in this section, we explore a different approach for optimism that is based on the classical experts framework. Specifically, we consider a model with two experts: a pessimistic (or robust) learner and an optimistic learner. The pessimist expert \emph{proposes} caching decisions based on the OGA policy \cite{paschos-jrnl} that does not use predictions, and provides adversarial regret guarantees. The optimistic expert proposes a caching policy that is optimized solely w.r.t. the predicted request, i.e., as if the predictions are fully reliable. Finally, a meta-learner receives the proposals from the two experts and gradually discerns which of them should be trusted. The expert-based approach to optimistic learning has been previously proposed for \emph{continuous} caching in \cite{ocol, naram-jrnl}. We expand it here to handle discrete decisions and demonstrate it using the single cache scenario. 

Formally, the pessimistic expert $(p)$ proposes caching $\{\hat z^{(p)}_t\}_t$ according to \emph{adaptive} OGA:
\begin{eqnarray}
    \hat z^{(p)}_t = \mathcal{P}_{\text{conv}(\mathcal{X})} \left\{ z^{(p)}_{t-1} + \frac{1}{\sqrt{t}} \theta_t \right\},
    \label{eq:pes-action}
\end{eqnarray}
where $\mathcal P_{\convX}$ is the Euclidean projection onto the convex hull of $\X$. We denote the regret of this expert by $R_T^{(p)} = \dtp{\Theta_T}{\hat z^\star} - \sumT \dtp{\theta_t}{\hat z^{(p)}_t}$, where $z^\star = \argmax_{z\in\convX}\langle \Theta_T, z \rangle$. On the other hand, the optimistic expert $(o)$ solves the following LP $z^{(o)} = \arg\max_{z\in\mathcal{X}} \dtp{\ti\theta_t}{z}$,

and we denote its regret with $R_T^{(o)} = \dtp{\Theta_T}{ z^\star} - \sumT \dtp{\theta_t}{ z^{(o)}_t}$.

Unlike the previous sections where predictions were used to modify the perturbation and regularization parameters, here they are treated independently through the optimistic expert. The challenge is then to learn which of the two experts’ proposals to follow. To that end, a meta-learner combines the proposals through a set of learned weights. Namely, the meta-learner's decision variable is $w \triangleq (w^{(p)}, w^{(o)}) \in \Delta_2$, where $\Delta_2 = \left\{w\in[0, 1]^2 \big| \|w\|_1=1\right\}$, and is used to create a convex combination of the provided caching proposals, i.e., $ \hat z_t = w^{(p)}_{t} \hat z^{(p)}_t + w^{(o)}_{t} z^{(o)}_t.$

Clearly,  by its definition, it holds that $\hat z_t \in \text{conv}(\mathcal{X})$. The weights are updated with adaptive OGA:
\begin{eqnarray}
     \hat w_t = \mathcal{P}_{\text{conv}(\Delta_2)} \left\{ w_{t-1} + \frac{1}{\sqrt{t}} l_t \right\},
     \label{eq:exps-weights-update}
\end{eqnarray}
where $l_t \triangleq \left(\dtp{\theta_t}{\hat z^{(p)}_t}, \dtp{\theta_t}{z^{(o)}_t}\right)$ is the experts' utility vector at slot $t$.
We then have the following result for the regret of the actual mixed action \cite[Thm. 3]{ocol}:
\begin{eqnarray}
    \hat R_T\{\hat z \}_T = \dtp{\Theta_T}{\hat z^\star} - \sumT \dtp{\theta_t}{\hat z_t} \leq 
    R_T^{(w)} + \text{min}\left\{R_T^{(p)},  R_T^{(o)}\right\} \leq 2 \sqrt{2T} + \text{min}\left\{R_T^{(p)},  R_T^{(o)}\right\}
\end{eqnarray}
Finally, similar to what we have shown in the OFTRL section, it is possible to use Madow's sampling to recover integral cache states $z_t \in \mathcal{X}$ with the associated bound
\begin{align}
\expec{}{R_T\{z_t\}} \leq R_T^{(w)} + \text{min}\left\{R_T^{(p)}, R_T^{(o)}\right\}. \label{experts-regret}
\end{align}
The steps of this scheme are summarized in Algorithm \ref{alg:exps}.

\begin{algorithm}[t]
	\caption{\small{\texttt{Experts-Cache}}}
	\label{alg:exps}
	\begin{footnotesize}
		\nl \textbf{Input}: $z_1\in\X$\\%
		\nl \textbf{Output}: $\{z_t \in \mathcal{X}\}_T$ \hfill $//$ \emph{Feasible caching vector at each slot} \\%
		\nl \For{ $t=2,3,\ldots$  }{
			\nl $\hat z^{(p)}_t = \mathcal{P}_{\text{conv}(\mathcal{X})} \left\{ z^{(p)}_{t-1} + \frac{1}{\sqrt{t}} \theta_t \right\}$ \hfill $//$ \emph{Pessimistic expert makes proposal}\\
			\nl $z^{(o)} = \arg\max_{z\in\mathcal{X}} \dtp{\ti\theta_t}{z}$ \hfill $//$ \emph{Optimistic expert makes proposal based on the oracle's prediction}\\
			\nl $\hat z_t = w^{(p)}_{t} \hat z^{(p)}_t + w^{(o)}_{t} z^{(o)}_t$ \hfill $//$ \emph{Meta-learner combines proposals}\\
            \nl $z_t\leftarrow MadowSample(\hat z_t)$  \hfill$//$ \emph{Obtain the discrete cache vector using Algorithm \ref{alg:ms}}  \\			
			\nl $\Theta_t = \Theta_{t-1} + \theta_t $ \hfill $//$ \emph{Receive the request for slot $t$ and update total grad} \\
			\nl $\hat w_t = \mathcal{P}_{\text{conv}(\Delta_2)} \left\{ w_{t-1} + \frac{1}{\sqrt{t}} l_t \right\}$ \hfill $//$ \emph{Meta-learner observes losses $l_t$ \& updates weights}\\
		}
	\end{footnotesize}
\end{algorithm}

\textbf{Discussion.} The performance advantage of the bound in \eqref{experts-regret} is that it can be strictly negative, depending on the optimistic expert's regret. For example, in case of perfect predictions and non-fixed cost functions, the $\text{min}$ term evaluates to $-\epsilon T$ for some $\epsilon>0$, making the meta-regret negative for large enough $T$. In all cases, the meta-regret is upper bounded by $O(\sqrt{T})$ due to the existence of the robust expert's regret in the $\text{min}$ term, hence we maintain the order-optimal regret for worst-case scenarios with this approach as well. From a computational load perspective, the most challenging step is the projection involved in the calculation of the OGD-based policy (pessimistic expert). However, one can leverage the tailored fast projection proposed in \cite{paschos-jrnl} for that operation. It is also important to stress that this framework allows to combine more than one expert, in order to either to e.g., include more than one  predictor, see discussion also in \cite{naram-jrnl}.

We note that since experts-based optimism is a meta-algorithm whose regret is characterized by that of the experts (i.e., learning algorithms), it can be applied to the other setups of unequal sizes and bipartite caching. The (possibly $\alpha$) meta regret will then be related to that of the (possibly $\alpha$) regret of the optimistic and pessimistic experts. Finally, it is worth noting that the idea of using the experts model for combining multiple caching policies has been previously proposed in \cite{ismail-nips02}, and evaluated in several cases, e.g., see \cite{rodriguez-usenis21} and reference therein, which however do not consider predictors nor provide any theoretical analysis (or, bounds) for the performance of this approach.

%% file: experiments.tex
\section{Experiments}\label{sec:evaluation}

We compare the performance of our algorithms with carefully-selected competitors: the FTRL policy which generalizes the OGD from \cite{paschos-jrnl}, and the FTPL method from \cite{abhishek-sigm20}. We note that these competitors already showed superior performance to the classical methods of LRU and LFU in their experiments. The request traces are created using the MovieLens dataset \cite{mlds} which contains time-stamped movie ratings. We assume a request is initiated to a CDN in the same chronological order as their ratings' timestamps. We consider movies with at least $8$ ratings, leading to a library of $N\!=\!10379$ and we set capacity $C=150$. Each prediction is assumed correct with probability $\rho$. Specifically, we generate a one-hot $\ti\theta_t$ that has $1$ at the file to be requested with probability $\rho$, or at any other random file with probability $1-\rho$. We also experiment with \emph{probabilistic} predictions where the vector components represent the probabilities of files being requested (details in Appendix \ref{appendix:simulations}). For the experiments with unequal-sized files, we generate the sizes uniformly $s_i \sim  U[1, 10]$ and set $C\!=\!500$. For the bipartite network, we use the $100$k variation of the MovieLens dataset and consider files with at least $10$ ratings, leading to $N\!=\!1152$. The network consists of $3$ caches ($C\!=\!150$) and $4$ user locations, the first two connected to caches $1$ \& $2$, and the rest to caches $2$ \& $3$.  

\begin{figure}
     \centering
     \begin{subfigure}[b]{0.24\textwidth}
         \centering
         \includegraphics[width=\textwidth]{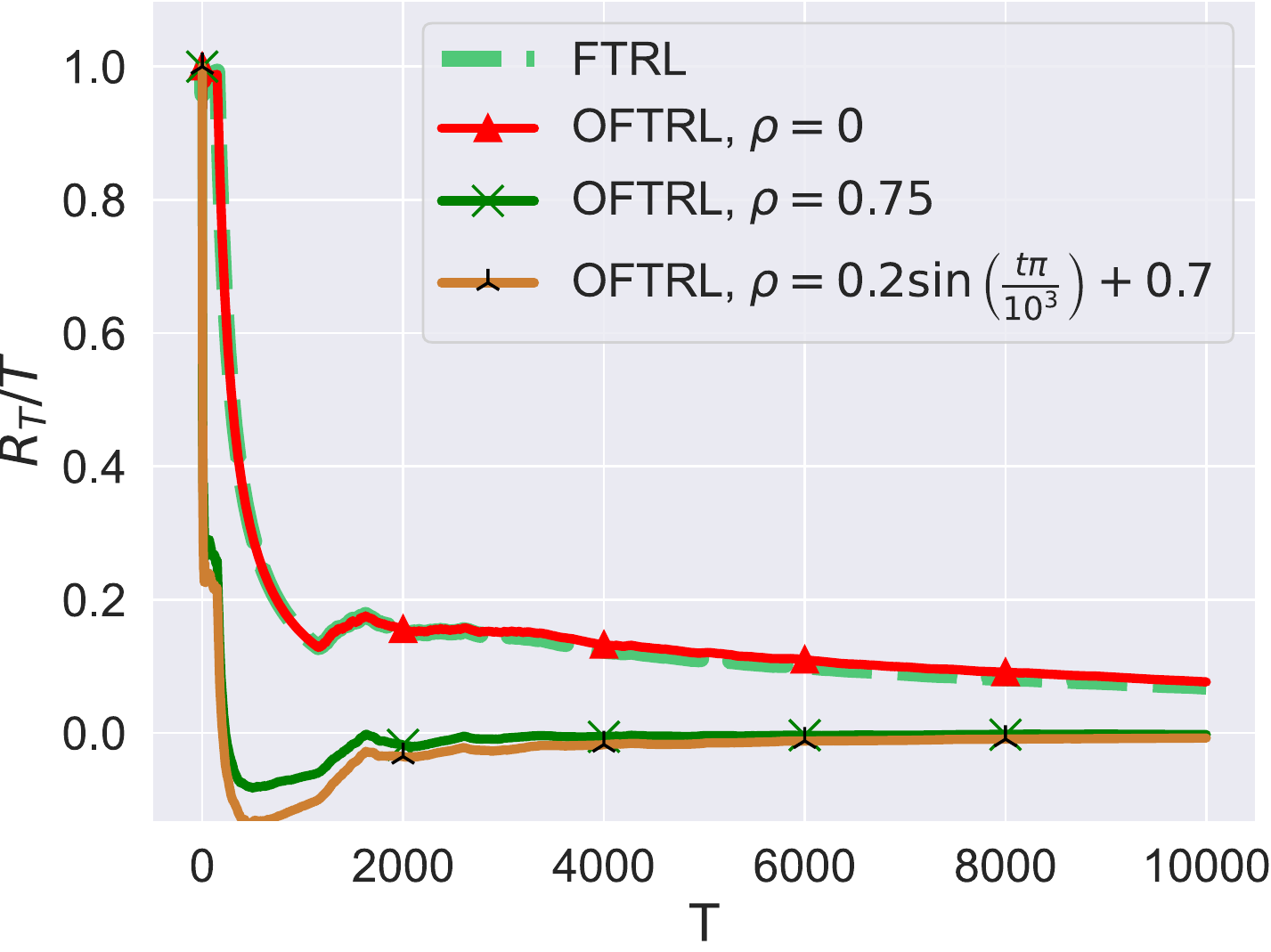}
         \caption{\footnotesize{\texttt{OFTRL-Cache}}}
         \label{fig:es_a}
     \end{subfigure}
     \hfill
     \begin{subfigure}[b]{0.24\textwidth}
         \centering
         \includegraphics[width=\textwidth]{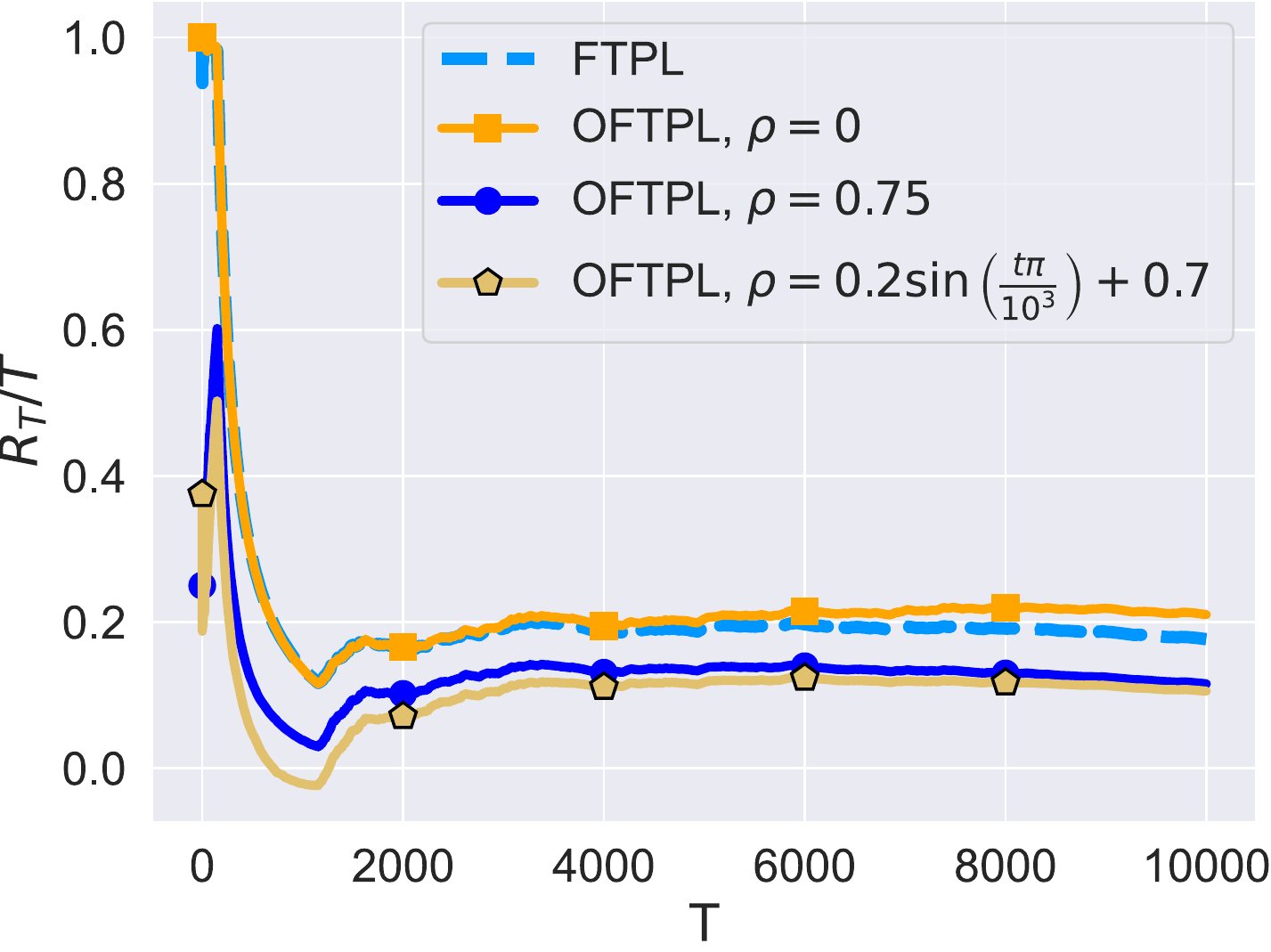}
         \caption{\footnotesize{\texttt{OFTPL-Cache}}}
         \label{fig:es_b}
     \end{subfigure}
     \hfill
     \begin{subfigure}[b]{0.24\textwidth}
         \centering
         \includegraphics[width=\textwidth]{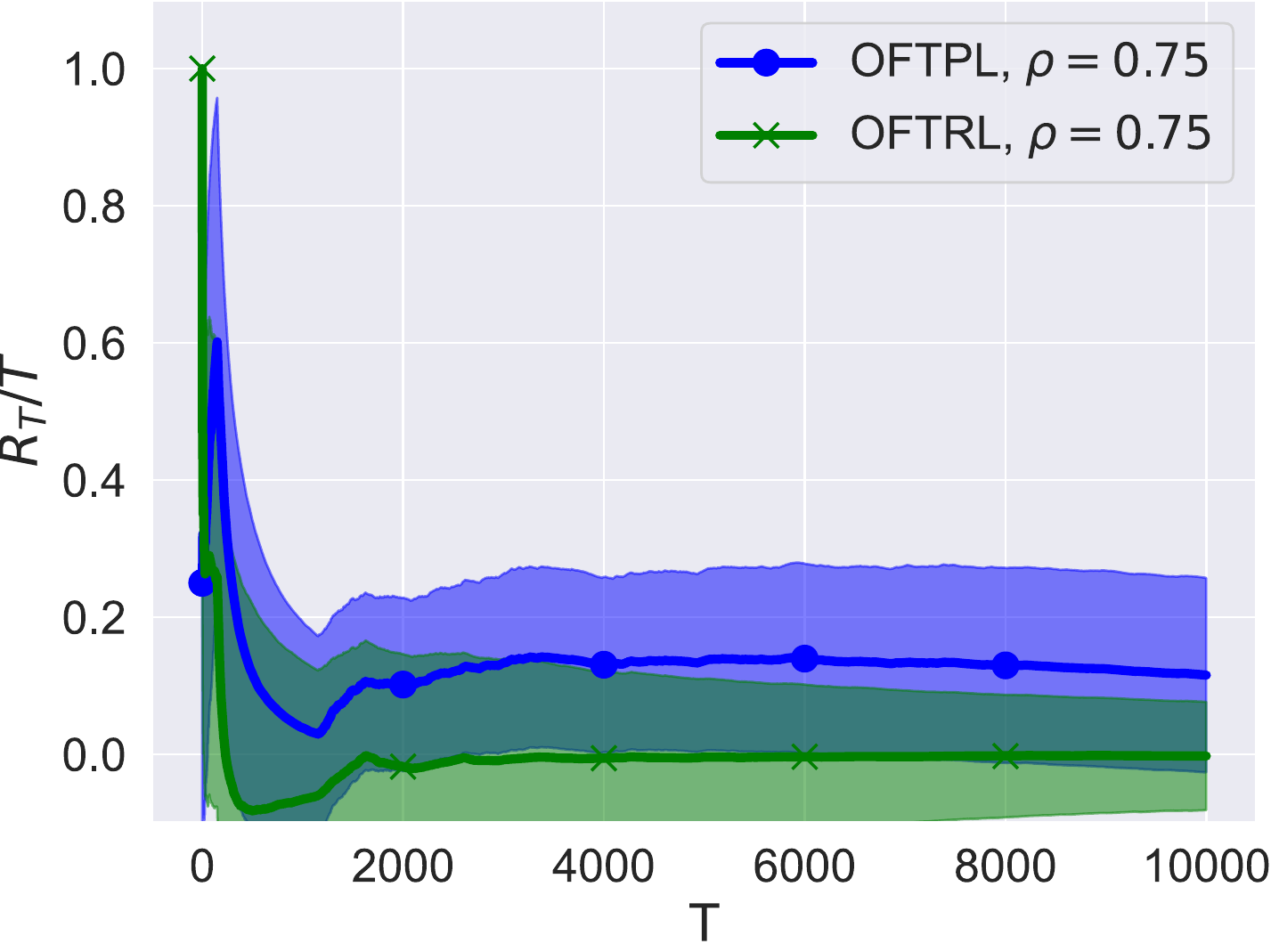}
         \caption{\footnotesize{\texttt{OFTRL/OFTPL,} $\rho\!=\!0.75$}}
         \label{fig:es_c}
     \end{subfigure}
     \begin{subfigure}[b]{0.24\textwidth}
         \centering
         \includegraphics[width=\textwidth]{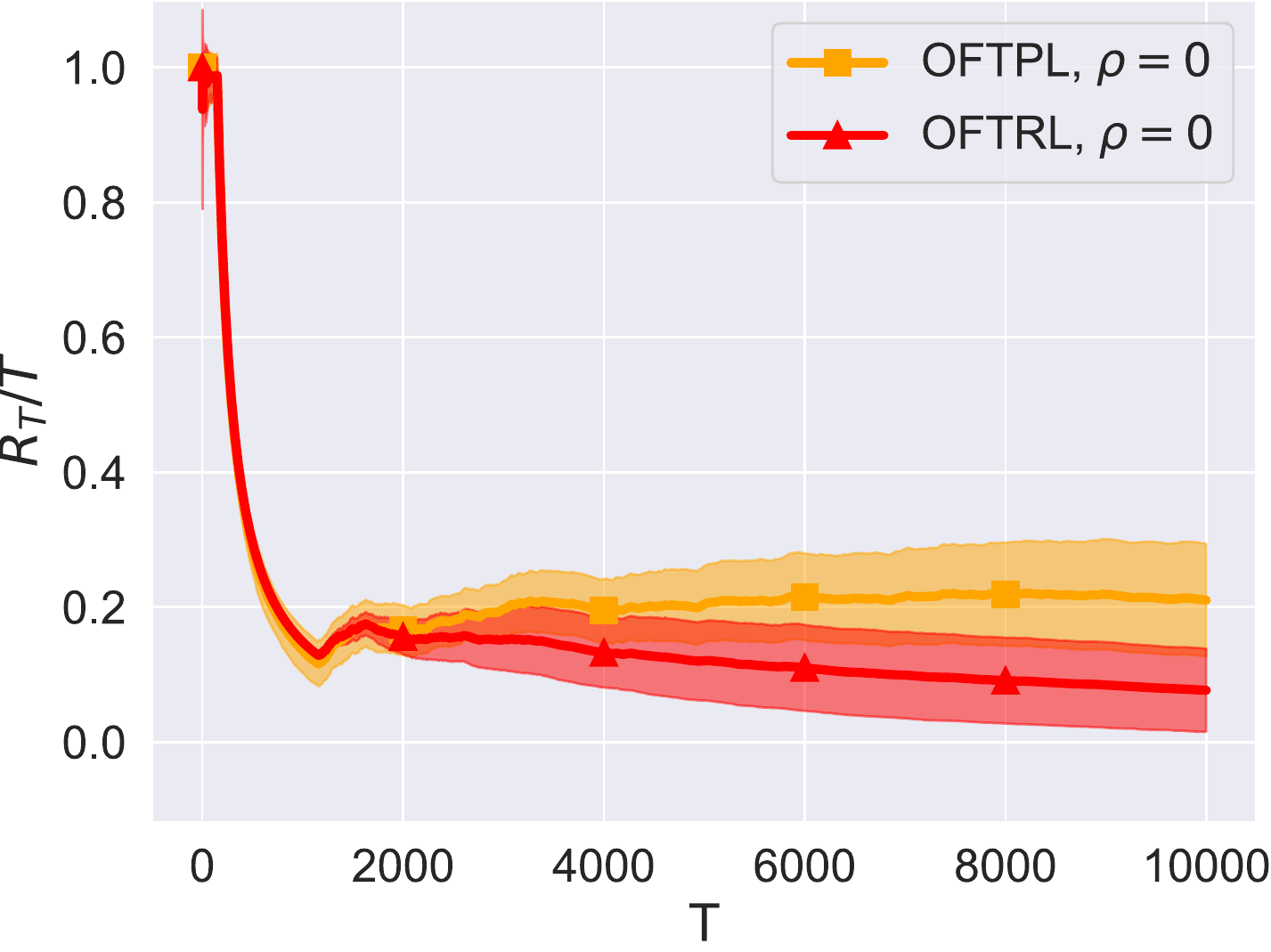}
         \caption{\footnotesize{\texttt{OFTRL/OFTPL,} $\rho\!=\!0$}}
         \label{fig:es_d}
     \end{subfigure}
        \vspace{-2mm}
        \caption{Comparison of $R_T/T$ for \textbf{equal-sized} files in a \textbf{single cache} and different policies: (a) \texttt{OFTRL-Cache} vs. FTRL/OGD; (b) \texttt{OFTPL-Cache} vs. FTPL \cite{abhishek-sigm20}; (c) \texttt{OFTRL-Cache} vs. \text{OFTPL-Cache} for good predictions and in (d) for worst-case predictions. In (c), (d) we plot the $0.95$-confidence interval ($8$ runs).}
        \label{fig:es}
\end{figure}

\begin{figure}
     \centering
     \begin{subfigure}[b]{0.24\textwidth}
         \centering
         \includegraphics[width=\textwidth]{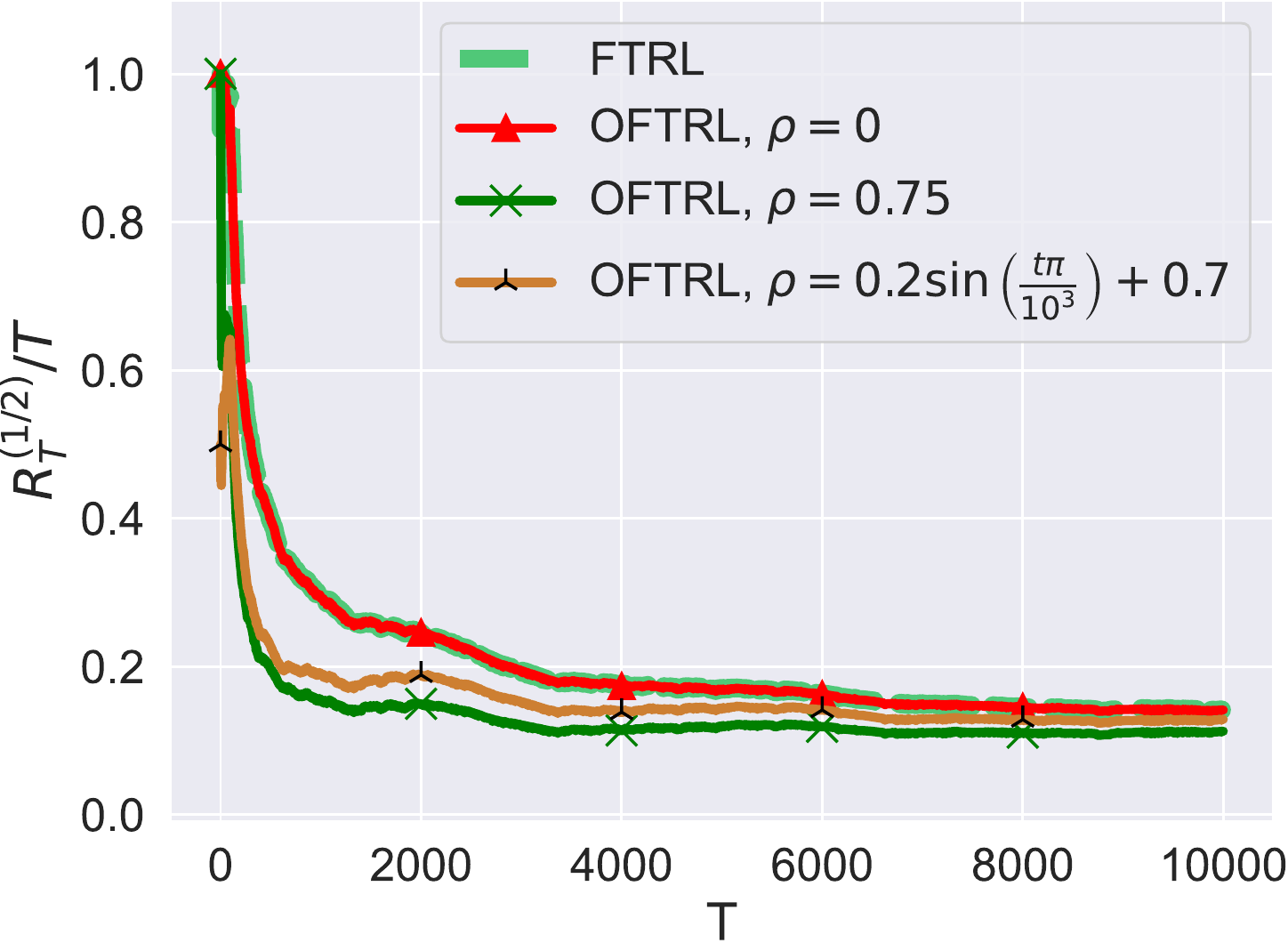}
         \caption{\footnotesize{\texttt{OFTRL-UneqCache}}}
         \label{fig:us_a}
     \end{subfigure}
     \hfill
     \begin{subfigure}[b]{0.24\textwidth}
         \centering
         \includegraphics[width=\textwidth]{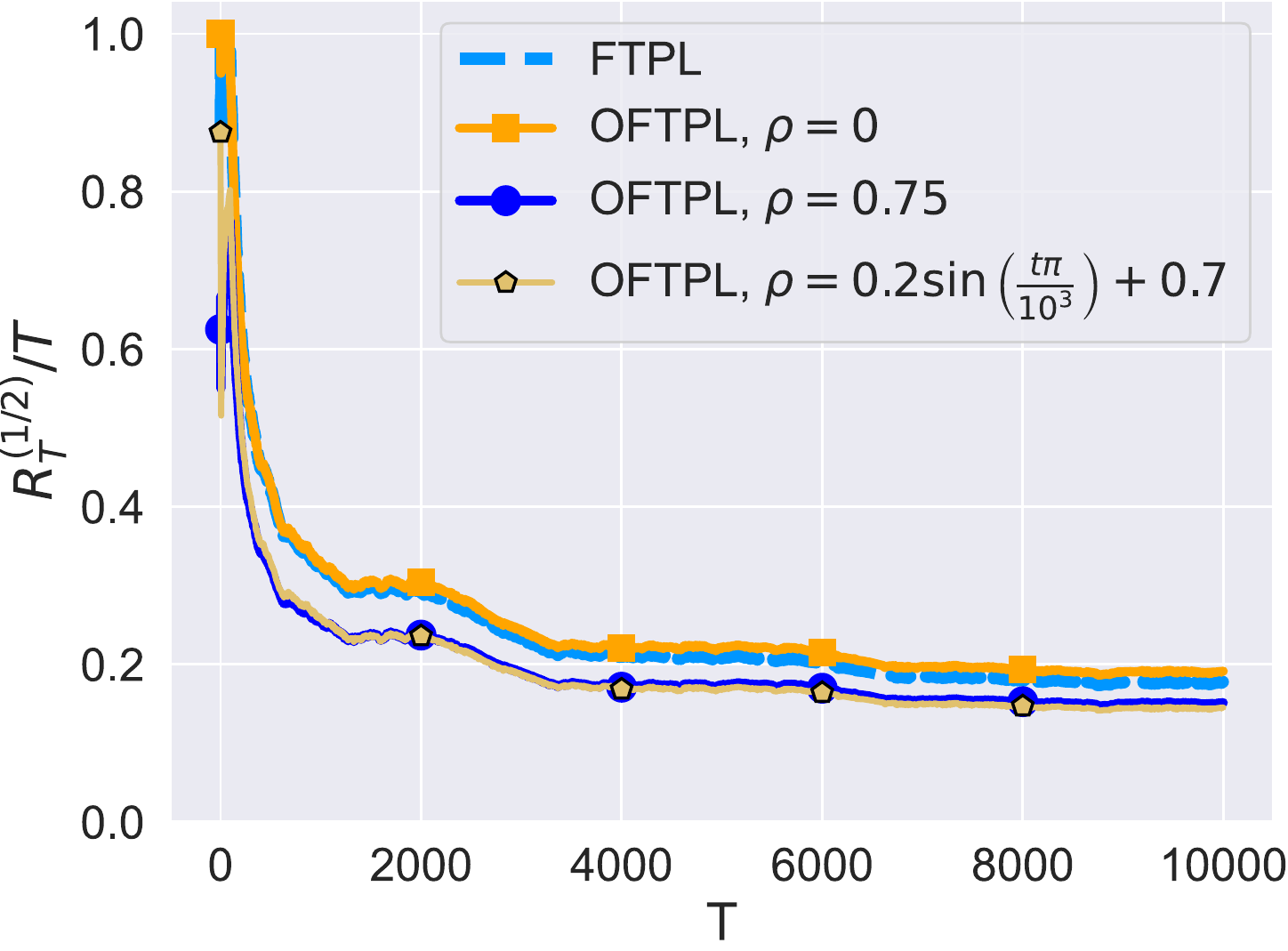}
         \caption{\footnotesize{\texttt{OFTPL-UneqCache}}}
         \label{fig:us_b}
     \end{subfigure}
     \hfill
     \begin{subfigure}[b]{0.24\textwidth}
         \centering
         \includegraphics[width=\textwidth]{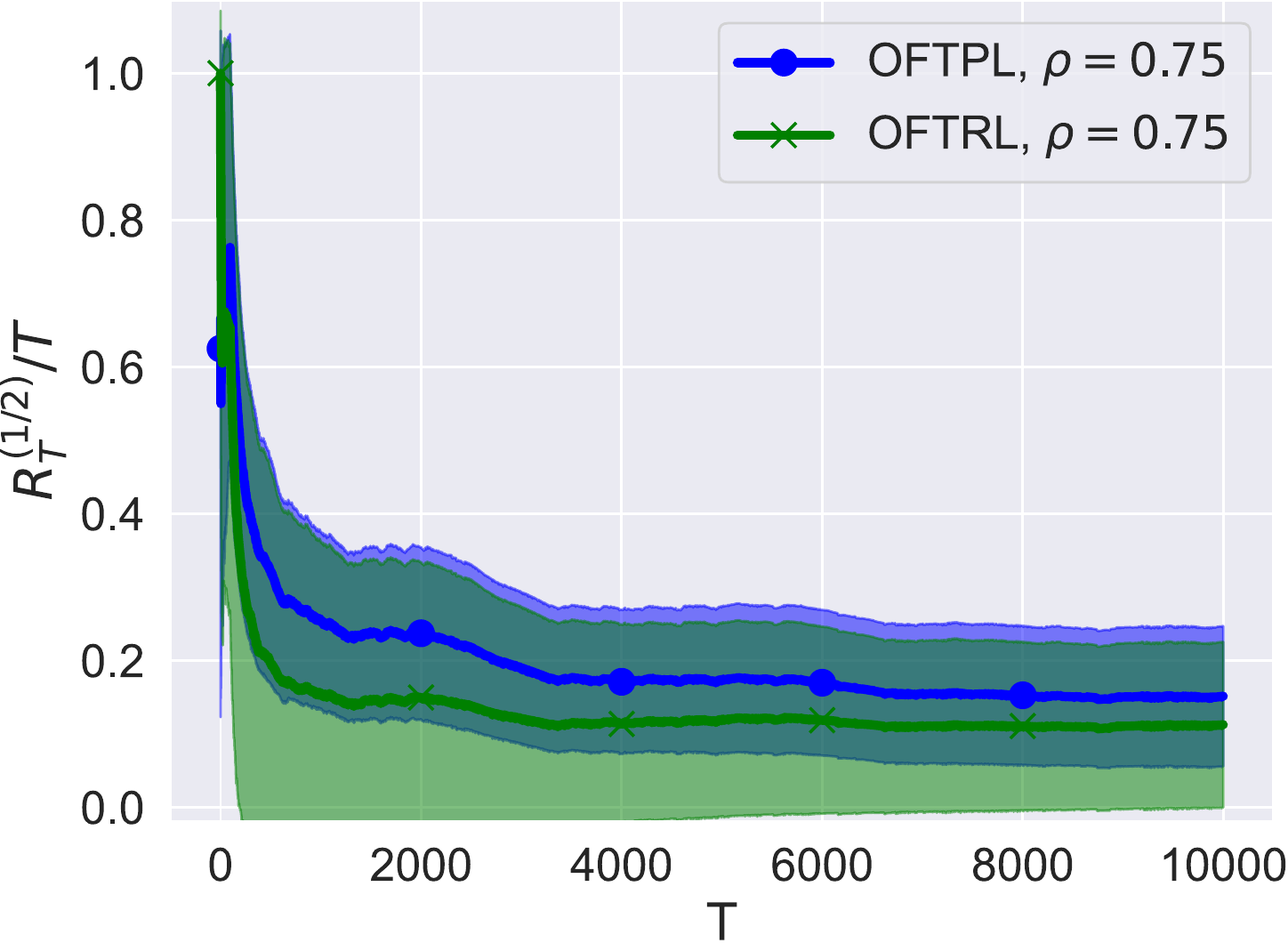}
         \caption{\footnotesize{\texttt{OFTRL/OFTPL,} $\rho\!=\!0.75$}}
         \label{fig:us_c}
     \end{subfigure}
     \begin{subfigure}[b]{0.24\textwidth}
         \centering
         \includegraphics[width=\textwidth]{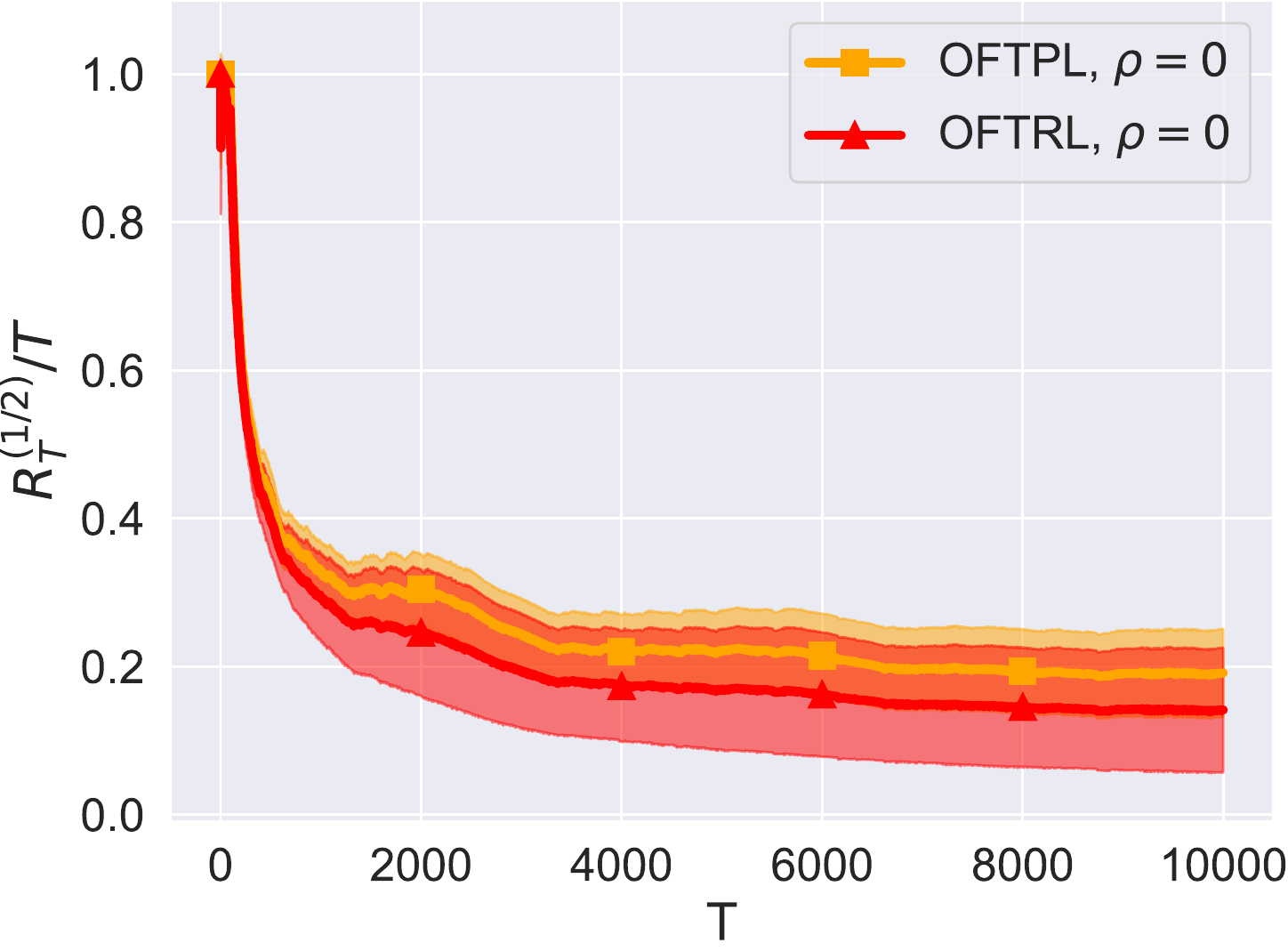}
         \caption{\footnotesize{\texttt{OFTRL/OFTPL,} $\rho\!=\!0$}}
         \label{fig:us_d}
     \end{subfigure}
        \vspace{-3mm}
    \caption{Comparing $R_T^{(\nicefrac{1}{2})}/T$ for \textbf{unequal-sized} files, \textbf{single cache}. (a) \texttt{OFTRL-UneqCache} vs. FTRL/OGD; (b) \texttt{OFTPL-UneqCache} vs. FTPL \cite{abhishek-sigm20}; (c)-(d) \texttt{OFTRL-UneqCache} vs. \text{OFTPL-UneqCache} for good (bad) predictions.}
        \label{fig:us}
        \vspace{-3mm}
\end{figure}

Fig. \ref{fig:es} shows the average regret (hit-rate gap to the optimal) growth with time for FTPL \cite{abhishek-sigm20}, FTRL \cite{paschos-jrnl}, and their proposed optimistic counterparts. We experiment with $\rho=0$, and $\rho=0.75$. If, e.g., the request predictions were based on recommendations, these reflect the cases where the users do not follow the recommendations ($\rho=0$), or actually request the recommended movie/file with probability $75\%$, ($\rho=0.75$). In addition, we experiment with a sinusoidal  $\rho$, which varies between $\rho=0.5$ and $\rho=0.9$, with a period of $10^3$ slots. We observe that optimism accelerates and improves learning the best files to cache, reaching an average improvement of $104\%$ (for OFTRL) and $37.1\%$ (for OFTPL) when $\rho=0.75$, compared to their "vanilla" counterparts (no predictions). Moreover, the performance degradation due to inaccurate predictions is almost negligible: $\leq 8.3 \%$ for OFTRL and $\leq 6.6\%$ for OFTPL. We also plot the $0.95$-confidence interval of $R_T$ in Figures \ref{fig:es_c} and \ref{fig:es_d}, where we note the more condensed distribution for OFTRL: $44.3\%$ and $26.1\%$ tighter at $t=10k$ when $\rho=0.75, \rho=0$, respectively. This is because the distribution ($\{\hat x_t\}_T$ iterates in OFTRL) becomes more concentrated with time; an argument that is not directly applicable to OFTPL, where the randomness is due to solving a \emph{perturbed} linear program. In Fig. \ref{fig:us} we evaluate the algorithms for the \emph{unequal} sizes case and plot the $\nicefrac{1}{2}$-regret. We observe the same pattern of negligible performance degradation when $\rho=0$, while $\rho=0.75$ enables an improvement of $35\%$ (for OFTRL) and $18.8\%$ (for OFTPL). We kindly refer the reader to the appendix for additional experiments for the experts-caching algorithm, the bipartite caching problem, and with probabilistic predictions of varying qualities.

%% file: conclusions.tex
\section{Conclusions and Future Work} \label{sec:conclusions}
In this paper, we presented several provably optimal algorithms that exploit predictions of unknown quality to improve the regret bounds for important variants of the discrete caching problem, while maintaining worst-case guarantees. The tackled problems are general (e.g., the Knapsack problem) and extend beyond caching; and hence the corresponding proposed optimistic algorithms can be applied to other similar problems. Our approach was based on the unified view of FTRL and FTPL algorithms as smoothing operations, where we proposed to make such smoothing adaptive to the predictions' accuracy. This allowed us to obtain a regret that interpolates between $0$ and $O(\sqrt{T})$.

This work also paves the road for several promising extensions. Given that eviction-only policies such as LFU or LRU have provably linear worst-case regret \cite{abhishek-sigm20}, we studied policies that can dynamically prefetch files. Thus, balancing the cache hits with prefetching costs remains to be tackled. Moreover, we note that \emph{static regret} algorithms, like ours, can be used as a subroutine in algorithms with stronger benchmarks, such as the $\Phi$-regret \cite{gordon2008no} and the minimum regret over all Finite-State-Predictors \cite{joshi2022universal}, and extending the study towards such more-refined benchmarks is certainly interesting. Finally, considering unequal routing utility (e.g., link-capacitated model \cite{6883210})  and unequal-sized files for the bipartite network model remains an open question \cite{mukhopadhyay2022k}.


%% file: appendix.tex
\appendix{}
\section{Appendix}

\subsection{Comparative Summary with Related Work}\label{appendix:summary-table}
Table \ref{tab:summary} shows the performance and complexity trade-offs for the presented algorithms, and compares them to recent studies of discrete no-regret caching in the literature. The best case refers to the situation where the request predictions are perfect $\tilde{\theta}_t = \theta_t, \forall t$. The worst case refers to the situation where predictions are furthest from the truth $\tilde{\theta}_t = \arg\max_{\theta}\|\theta - \theta_t\|, \forall t$. The previous studies have the best and worst case columns merged as they do not utilize predictions. Furthermore, the works of \cite{abhishek-sigm20} and \cite{tareq-jrnl} assume and utilize knowledge of the time horizon $T$ ( \cite{10.1145/3491047} uses the standard doubling trick) and use the Lipschitz constant for the gradient (i.e., request) vector. Thus, they are not classified as performing Adaptive Learning (\textbf{Adap. Learn.}) as defined by \cite{mcmahan-survey17}, which argues about the advantages of adaptive algorithms of the sort presented here. While the authors in \cite{abhishek-sigm20} discuss the bipartite model, their simpler linear \emph{elastic} model of utility is different than the one considered here (see \cite[Sec. 3.2]{abhishek-sigm20}). Hence, we compare to their single cache result. Finally, for algorithm $5$, we make explicit the dependence on \emph{weights} regret $R_T^{(w)}$, although it is still $R_T^{(w)} \leq O(\sqrt{T})$ to clarify the cause of inferior performance of the experts-based optimism in the worst case compared to adaptive smoothing, which even appears in the simulations.

\begin{table*}[h]
\begin{footnotesize}
	\centering
	\caption{Online discrete caching policies with \emph{adversarial no regret} guarantees: a summary of the contributions and comparison with literature. 
	For the constant $\alpha$, recall that $\alpha=1$ indicates the regular regret. Otherwise, we have $\alpha$-approximate regret (see equation \eqref{eq:def-alpha-regret}).
	}
	\label{tab:summary}
	\begin{tabularx}{0.99\textwidth}{|m{0.7cm}||m{4.0cm}|m{1.3cm}|m{2.25cm}|m{1.1cm}|m{1.0cm}|Y|}
		\hline
		\multirow{2}{*}{\textbf{Alg.}} &   \multirow{2}{*}{ \textbf{Model and Conditions}} & \multicolumn{2}{c|}{\textbf{Guarantees ($R^{(\alpha)}_T\leq$ )}} & \makecell{\textbf{Comput.}\\ \textbf{Complex.}} &\makecell{\textbf{Approx.}\\ \textbf{Const.}\\ \textbf{$\alpha$}}& 
		\multirow{2}{*}{\makecell{\textbf{Adap.} \\\textbf{Learn.}}} 
		\\ \cline{3-4} 
		&  & \textbf{Best case} & \textbf{Worst case}  & & & \\ \hline
		
		1 &
		$\bullet$ Single cache $\bullet$  Predictions & $0$
		& $O(\sqrt{T})$ & $O(N^2)$ &$1$& \checkmark
		\\ \hline 
		
		2 &
		$\bullet$ Single cache $\bullet$  Predictions
		& 0 & $O(\textit{poly-log}(N)\sqrt{T})$ & $O(N)$ &$1$ & \checkmark \\ \hline
		
		3 &
		$\bullet$ Single cache $\bullet$ Predictions \newline $\bullet$ Unequal sizes
		&  $0$
		& $O(\sqrt{T})$ & $O(N^2)$ &$\nicefrac{1}{2}$ & \checkmark 
		\\ \hline 
		
		4 &
		$\bullet$ Single cache $\bullet$ Predictions \newline $\bullet$ Unequal sizes
		&  0 & $O(\textit{poly-log}(N)\sqrt{T})$ & $O(N)$&$\nicefrac{1}{2}$ & \checkmark 
		\\ \hline 

		5 &
		$\bullet$ Bipartite Network $\bullet$ Predictions
		& 0 & $O(\sqrt{T})$ & $O(N^2)$&$1-\nicefrac{1}{e}$ & \checkmark 
		\\ \hline 
		
		6 &
		$\bullet$ Single Cache $\bullet$ Predictions
	   & $b<0$ & $R_T^{(w)} + O(\sqrt{T})$ & $O(N)$ &$1$& \checkmark
		\\ \hline 

		\cite{abhishek-sigm20} &
		$\bullet$ Single cache
		&  \multicolumn{2}{c|}{$
			O(\textit{poly-log}(N)\sqrt{T})$ } & $O(N)$ &$1$& --
		\\ \hline
		
		\cite{tareq-jrnl} &
		 $\bullet$ Single cache
		&  \multicolumn{2}{c|}{$
			O\left(\sqrt{T}\right)$ } &$O(N)$&$1$& --
		\\ \hline
		
		\cite{leadcache} &
		$\bullet$ Bipartite network
		&  \multicolumn{2}{c|}{$
			O(\textit{poly-log}(N)\sqrt{T})$ } & $O(N)$&$1-\nicefrac{1}{e}$ &\checkmark
		\\ \hline
		
		\cite{10.1145/3491047} &
		$\bullet$ General network
		&  \multicolumn{2}{c|}{$
			O(\textit{poly-log}(N)\sqrt{T})$ } & $O(N)$&$1-\nicefrac{1}{e}$ &--
		\\ \hline
	
	\end{tabularx}
\end{footnotesize}	
\end{table*}

\subsection{Madow's Sampling Algorithm } \label{appendix:madow}
Algorithm \ref{alg:ms} describes how we obtain an integral caching vector from the continuous one. We start by sampling a uniform scalar and then loop for $C$ iterations, including in our gradually-built set exactly one item per iteration. Hence we ensure the resulting set satisfies the capacity constraint. During an iteration, we include an item if its probability (continuous variable) falls in a carefully designed range: $[\pi_{j-1},\pi_{j}]$. Hence, each item is included with probability $\pi_{j-1} - \pi_{j} = \hat x_i$. We refer the reader to \cite{madow} for futher details.

\begin{algorithm}
\caption{Madow's Sampling (\texttt{MadowSample})}
\label{alg:ms}
\begin{footnotesize}
\nl \textbf{Input}: $\hat x \in [0, 1]^N, \sum x_{i\in\N} \leq C $.\\%
\nl \textbf{Output}: Random set $S$,\quad \text{s.t} $|S|=C$ \text{and} $\Pr(i \in S ) = x_i$\\%
\nl Sample a uniformly random scalar $U\in[0, 1]$ \\
\nl Define the cumulative probabilities $\pi_0=0, \quad \pi_i = \pi_{i-1} + \hat x_i, \quad \forall\ 1 \leq i \leq N$\\

\nl \For{ $i=0,1,\ldots, C$  }{
    \nl $S \gets S \cup \left\{j: \pi_{j-1} \leq U+i < \pi_j \right\} $
}
\nl return $S$
\end{footnotesize}
\end{algorithm}

\subsection{Dependent Rounding Algorithm ($DepRound$)}
The dependent rounding algorithm operates sequentially. At each iteration, it picks two continuous variables and transfers at least one of them into an integer (through the \emph{if} statements in lines $4$ to $8$), while adjusting the other one (lines $9$ to $12$). Hence, we ensure that when the algorithm terminates, only one item is still fractional. The properties of the resulting vector listed in Lemma \ref{lemm:dep-round} are proved in \cite[Lem. 2.1]{depround}.

\begin{algorithm}[ht]
\caption{Dependent Rounding (\texttt{DepRound}) }
\label{alg:dr}
\begin{footnotesize}
\nl \textbf{Input}: $a \in [0, 1]^N, s \in R_+^N$.\\%
\nl \textbf{Output}: $b$ satisfying points in lemma-\ref{lemm:dep-round}.\\
\nl \While{$a$ contains two or more fractional elements}{
    \nl Denote the two left most fracitonal elements $a_i$ and $a_j$.\\
    \nl \If {$0 \leq s_i\ a_i + s_j\ a_j \leq \min\{a_i, a_j\}$}{
        Set $b_i = 0$ with probability $\nicefrac{s_j a_j}{s_i a_i + s_j a_j}$. With the remaining probability set $b_j = 0$
        }
    \nl \If {$a_i \leq s_i\ a_i + s_j\ a_j \leq a_j$}{
    Set $a_i = 1$ with probability ${a_i}$. With the remaining probability set $a_i = 0$
    }
    \nl \If {$a_j \leq s_i\ a_i + s_j\ a_j \leq a_i$}{
    Set $a_j = 1$ with probability ${a_j}$. With the remaining probability set $a_j = 0$
    }
    \nl \If {$\max\{a_i, a_j\} \leq  s_i\ a_i + s_j\ a_j \leq a_i+ a_j$} {
        Set $b_i = 1$ with probability $\nicefrac{s_j(1-a_j)}{(s_i(1-a_i)) + (s_j(1-a_j))}$. With the remaining set $b_j = 1$
        }
    \nl if $b_i = 0$ set $b_j =  \nicefrac{s_i}{s_j}\ a_i + a_j $\\
    \nl if $b_i = 1$ set $b_j = a_j -  \nicefrac{s_i}{s_j}\ (1-a_i) $\\
    \nl if $b_j = 0$ set $b_i = a_i +  \nicefrac{s_j}{s_i}\ a_j $\\
    \nl if $b_j = 1$ set $b_i = a_i -  \nicefrac{s_j}{s_i}\ (1-a_j) $\\
}
return $b$
\end{footnotesize}
\end{algorithm}



\subsection{Proof of Theorem \ref{thm:oftpl-unequal}}
\label{appendix:theorem-oftpl-unequal}

\begin{proof}
	Since ${y_{ti}} \in \{0, 1\}\ \forall t, i$, and the sampling in line $8$ is uniform, we get $\expec{}{y_{ti}} = \frac{1}{2}$. Hence  
	\begin{eqnarray}
		\expec{}{y_{ti}}  \geq \frac{1}{2} \hat y_{ti}. \label{eq:pw_app}
	\end{eqnarray}
	where we have used that $\hat y_t \in [0, 1]^N$. Now, from the definition of $\nicefrac{1}{2}$-Regret we have:
\begin{align}
R_T^{(\nicefrac{1}{2})}&= \frac{1}{2}\sum_{t=1}^T \langle \theta_t, y^\star \rangle - \mathbb{E}\left[\sum_{t=1}^T \langle \theta_t, {y}_t \rangle\right] \stackrel{(a)}{\leq} \frac{1}{2}\left(\sum_{t=1}^T \langle \theta_t, y^\star \rangle - \sum_{t=1}^T \langle \theta_t, \hat{y}_t \rangle\right) \\
&\stackrel{(b)}{=} 1.84\sqrt{C}\left(\ln\frac{Ne}{C}\right)^{1/4}
\sqrt{\sumT \err_1^2}. 
\end{align}
where inequality $(a)$ follows from the $\nicefrac{1}{2}$-approximation property of the randomized rounding algorithm \eqref{eq:pw_app}; and $(b)$ follows from the result of Theorem\footnote{Theorem \ref{thm:oftpl} operates on integral decisions $y_t$. Nonetheless, even if we allow $y^\star, y_t \in \text{conv}(\mathcal{X})$, they are still integral due to the linear program in line $6$ of Algorithm \ref{alg:oftpl} ($\{0, 1\}$ decision variables with non-negative coefficients).} \ref{thm:oftpl}.
	
\end{proof}

\subsection{Proof of Theorem \ref{thm:oftrl_US}} \label{appendix:theorem-oftrl_US}

First, we show that the $\nicefrac{1}{2}$-point-wise approximation holds for $\{x_t\}_t$. Then, we re-use the result of Theorem \ref{thm:oftrl-m}. In detail, by Lemma \ref{lemm:dep-round}, the \texttt{DepRound} subroutine returns $\bar x_t$ such that $\expec{}{\bar x_t} = \hat x_t$. Then, by the same argument about uniform sampling in the proof of Theorem \ref{thm:oftpl-unequal}, we have that $\expec{}{x_t} \geq \frac{1}{2} \bar x_t$, where  $x_t \in \mathcal{X}_s$. We recover our $\nicefrac{1}{2}$-approximation guarantee for OFTRL iterates $\{x_t\}_t$:
\begin{align}
\expec{}{x_t} \geq \frac{1}{2} \hat{x}_t. \label{eq:pw_app_oftrl}
\end{align}
By the definition of $\nicefrac{1}{2}$-regret guarantee, we have 
\begin{align}
\!\!\!R_T^{(\nicefrac{1}{2})} &= \frac{1}{2}\sum_{t=1}^T \langle \theta_t, \hat x^\star \rangle - \mathbb{E}\left[\sum_{t=1}^T \langle \theta_t, {x}_t \rangle\right] \stackrel{(a)}{\leq} \frac{1}{2}\big(\sum_{t=1}^T \langle \theta_t, \hat x^\star \rangle - \sum_{t=1}^T \langle \theta_t, \hat{x}_t \rangle\big) \stackrel{(b)}=  \sqrt{C}\sqrt{\sumT \err_2^2}. \notag
\end{align}
where inequality $(a)$ follows from the $\nicefrac{1}{2}$-approximation property of the randomized rounding algorithm \eqref{eq:pw_app_oftrl}; and $(b)$ follows from the result of Theorem \ref{thm:oftrl-m}. Although that theorem states the bound for the regret of the integral decisions $\{x_t\}_t$, we have seen in its proof that the bound is essentially the same for the continuous actions.


\subsection{Proof of Theorem \ref{thm:oftrl-m-bp}}
\label{appendix:oftrl-m-bp}
We start from the result of \cite[Thm. 1]{naram-jrnl} (or its earlier version from \cite{ocol}), which provides guarantees for the regret of the continuous variables $\hat x$. For the moment, assume that we have the following point-wise approximation for the decision variables $\mathbb{E}[\hat k] = k$, $\expec{}{\hat u} \geq (1-\nicefrac{1}{e})u$, and that the capacity constraints are respected. Then, due to the linearity of the objective function $f_t(\cdot)$, we get that our $\alpha$-regret (with $\alpha = 1-\nicefrac{1}{e}$) is:
\begin{align}
R_T^{(1-\nicefrac{1}{e})} &= (1-\frac{1}{e})\sum_{t=1}^T \langle \theta_t, \hat x^\star \rangle - \mathbb{E}\sum_{t=1}^T \langle \theta_t, {x}_t \rangle \stackrel{(a)}{\leq} 1 - \frac{1}{e}\big(\sum_{t=1}^T \langle \theta_t, \hat x^\star \rangle - \sum_{t=1}^T \langle \theta_t, \hat{x}_t \rangle\big) \notag \\
&\stackrel{(b)}{=} (2-\frac{2}{e})\sqrt{1+JC}\sqrt{\sumT \err_2^2}.  \notag
\end{align}
where inequality $(a)$ follows from the $(1-\nicefrac{1}{e})$-approximation property of the randomized rounding algorithm; and $(b)$ follows from the result of Theorem \cite[Thm. 1]{naram-jrnl}.

Now, to show that $\mathbb{E}[\hat k] = k$, we follow the same argument in the proof of Theorem \ref{thm:oftrl-m}. Namely, due to Madow's sampling, each file is included with probability $\hat k$, and at most $C_j$ files are included at each cache. Hence, $\expec{}{ k_{nj}} = \hat k_{nj}, \forall{n, j}$. Regarding the point-wise approximation for $u$, we define the set $\mathcal{J}^{i}$ of  caches connected to user $i$: 
\[
\mathcal{J}^{i} = \left\{j \in \mathcal{J} \big|\  d_{ij} = 1\right\}.
\]
Then, we have
\begin{align}
 \expec{}{u_{nij}} &\stackrel{(a)} = \Pr[u_{nij} = 1] \stackrel{(b)}= \Pr[\vee_{j \in \mathcal{J}^i} k_{nj} = 1 ] \stackrel{(c)}= 1 - \prod_{j \in \mathcal{J}^i} \big(1 - \hat k_{nj}\big) 
    \\
    &\stackrel{(d)}\geq 1 - e^{-\sum_{j\in\mathcal{J}^i} \hat k_{nj}}  \stackrel{(e)}\geq 1 - e^{-\hat u_{nij}} \stackrel{(f)}\geq \left(1-\frac{1}{e}\right)\hat u_{nij} \notag.
\end{align}
Where in the above chain of inequalities, $(a)$ follows from  $u_{nij}$ being binary variable; $(b)$ by the construction of the algorithm (step $8$); and $(c)$ from the independent rounding for each cache. Also, inequality $(d)$ follows from from $e^x \geq 1 + x, \forall{x} \in \mathbb{R}$; $(e)$ from the relaxed version of the caching/routing constraint of $\mathbb P_2$, and finally $(f)$ from the concavity of $1-e^{-x}$ and the domain of $\hat u_{nij}$ being restricted to $[0, 1]$. 


\subsection{Additional Simulations} 

\subsubsection{Probabilistic Predictions} \label{appendix:simulations}

Let us first demonstrate, with a simple example, that using probabilistic predictions is beneficial for the performance of optimistic algorithms. The regret bound of the proposed algorithms depend on the terms $\err$, $\forall t$. Now, consider a prediction $\ti\theta_t $ that places $\epsilon$ probability mass on the correct file, and the remaining uniformly over the rest of the files in the library. Then, we get: \[
\err_2 \approx 1-\epsilon, \quad \text{since} \quad \frac{(1-\epsilon)^2}{(N-1)} \approx 0,
\]
compared to a mis-prediction (or, mistaken) one-hot $\ti\theta_t$ which will have $\err_2=\sqrt{2}$. Using the $\ell_1$ norm, a one-hot mistake costs $2$ compared to $2-2\epsilon$ for the probabilistic one. We stress again that all the results presented in this work hold both for probabilistic and for deterministic predictions. The former can be taken directly from the output of a forecasting model, while one can create the latter by simply using the highest-probability request.

We continue by presenting experimental results for a probabilistic prediction model with varying accuracy. In detail, in Fig. \ref{fig:soft} we measure the regret of the proposed OFTRL and OFTPL policies after $5k$ time steps (i.e., file request) using the well-known YouTube request trace \cite{zink2008watch} with $N=10^4$ and $C=150$. $R_{5k}$ is measured using prediction vectors with varying density that is placed on the file to be requested. Namely, if at step $t$, the requested file is $n$, we feed the optimistic algorithms with a prediction vector: 
\[
\ti\theta \quad  \text{with:} \quad  \ti\theta^t_{n} = \zeta \quad \text{and} \quad \ti\theta^t_{n'} = (1-\zeta)/(N-1),\ \forall n'\neq n.
\]
That is, the prediction vector has $\zeta$ probability placed on the file to be requested, and the remaining ($1-\zeta$) uniformly distributed across the remaining files. We note that when the prediction vector is almost uniform (i.e., $\ti\theta$ contains no useful information), the optimistic versions nearly match the non optimistic ones. 


We can see that at $\zeta=0.1$, both OFTRL and OFTPL already start outperforming their non-optimistic counterparts, by $9.8\%$ and $1.4\%$, respectively. Also, they outperform the best-in-hindsight benchmark $x^\star$ when the accuracy becomes reasonably high ($\zeta\geq0.8$). Lastly, we see that OFTRL has a performance advantage of up to $59.6\%$ compared to OFTPL, when fed the same predictions, at the expense of its additional computation complexity. 

\begin{figure}
     \centering
     \includegraphics[width=0.75\textwidth]{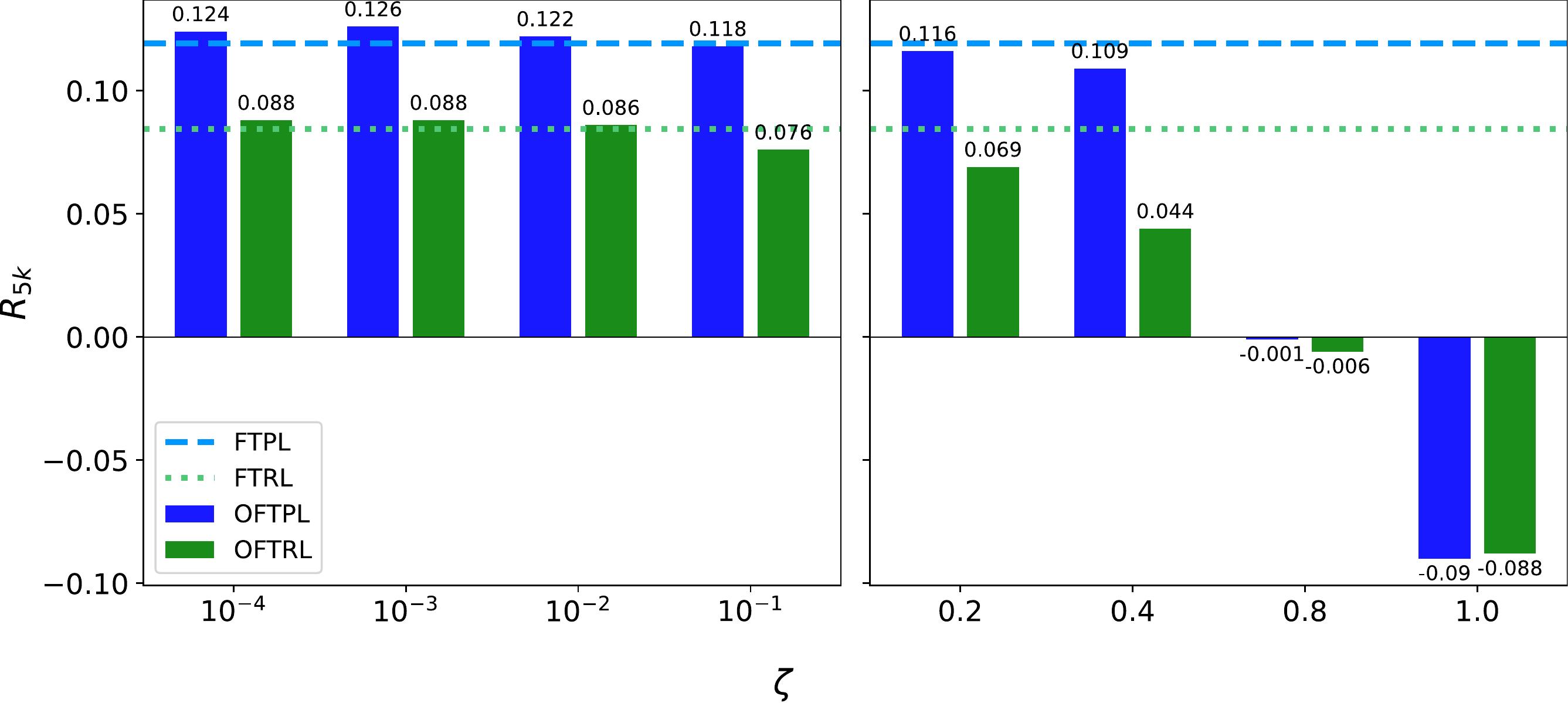}
     \vspace{-5mm}
     \captionof{figure}{Regret with varying probability mass placed on the correct file in the prediction vector.}
     \label{fig:soft}
\end{figure}

\subsubsection{Algorithms \ref{alg:oftrl-bp} and \ref{alg:exps}}

\begin{figure*}
\begin{minipage}{.35\textwidth}
     \includegraphics[width=0.95\textwidth]{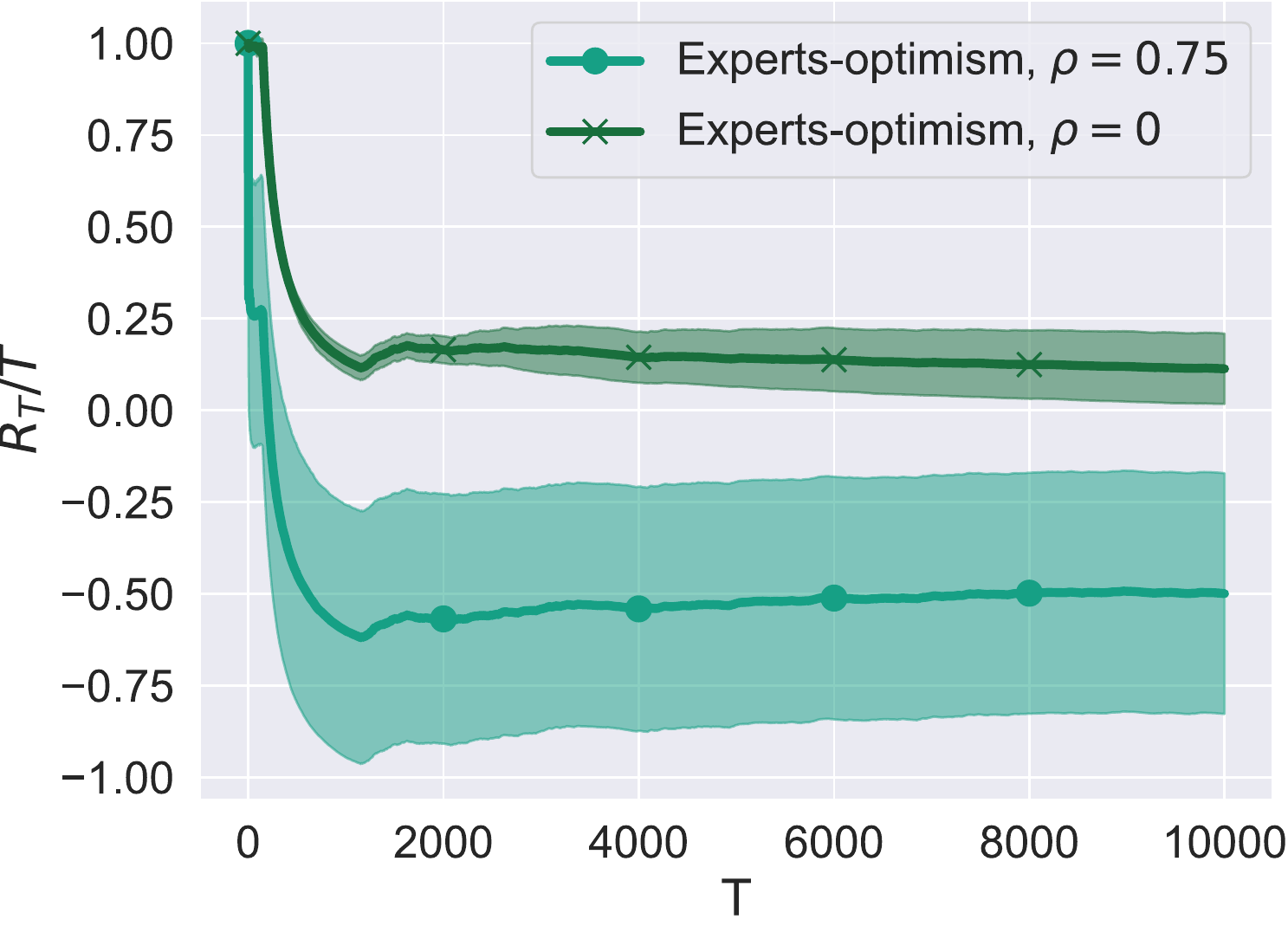}
    \vspace{-1mm}     
     \captionof{figure}{Average regret of experts-based optimism.}
     \label{fig:experts_eval}
\end{minipage}\ \ 
\begin{minipage}{.35\textwidth}
     \includegraphics[width=0.95\textwidth]{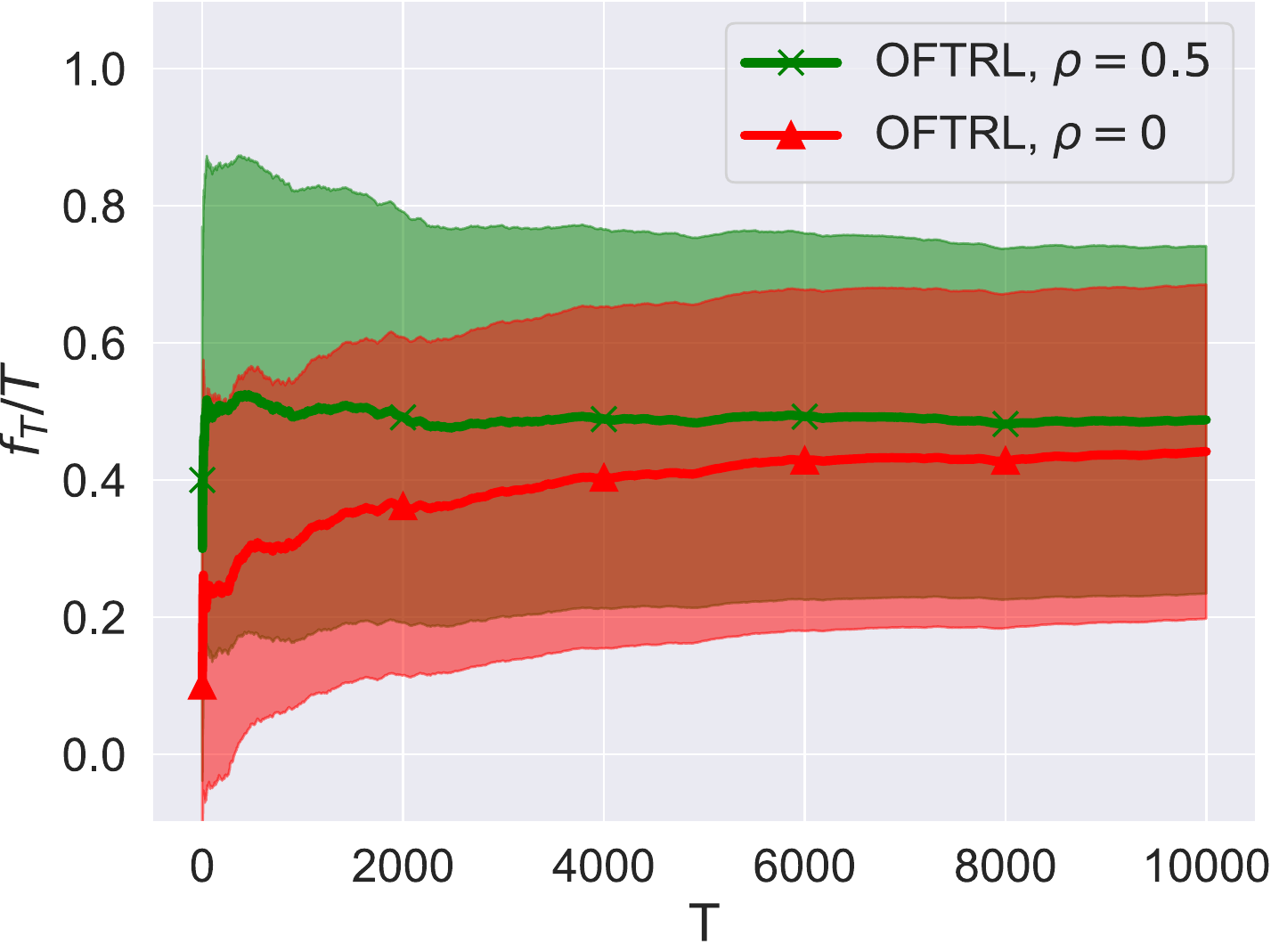}
     \vspace{-1mm}
     \captionof{figure}{\small{Average OFTRL hit-rate  for cache network.}}
     \label{fig:bp}
\end{minipage}
\end{figure*}

Fig. \ref{fig:experts_eval} plots the regret for the expert-based Algorithm \ref{alg:exps}. Note that the $R_T$ can reach negative values, i.e., outperform better the benchmark, when $\rho=0.5$. This is aligned with the bound in \eqref{experts-regret} and hints to the fact that stronger benchmarks can be used for this algorithm. However, it performs worse than the regularization-based optimism in the case where $\rho=0$, achieving regret $R_T=0.113$ at time $T=10k$ compared to $R_T=0.075$ (OFTRL). Lastly, the bipartite \emph{utility} is shown in Fig. \ref{fig:bp}, the \emph{hit-ratio} of OFTRL is approximately $0.49$ when $\rho=0.5$. Expectantly, the performance drops when $\rho=0$, but steadily increases from $0.30$ at $T=500$, to $0.44$ at $T=10k$.

%% file: main.bbl

\begin{thebibliography}{74}


\ifx \showCODEN    \undefined \def \showCODEN     #1{\unskip}     \fi
\ifx \showDOI      \undefined \def \showDOI       #1{#1}\fi
\ifx \showISBNx    \undefined \def \showISBNx     #1{\unskip}     \fi
\ifx \showISBNxiii \undefined \def \showISBNxiii  #1{\unskip}     \fi
\ifx \showISSN     \undefined \def \showISSN      #1{\unskip}     \fi
\ifx \showLCCN     \undefined \def \showLCCN      #1{\unskip}     \fi
\ifx \shownote     \undefined \def \shownote      #1{#1}          \fi
\ifx \showarticletitle \undefined \def \showarticletitle #1{#1}   \fi
\ifx \showURL      \undefined \def \showURL       {\relax}        \fi
\providecommand\bibfield[2]{#2}
\providecommand\bibinfo[2]{#2}
\providecommand\natexlab[1]{#1}
\providecommand\showeprint[2][]{arXiv:#2}

\bibitem[{A. Giovanidis, and A. Avranas}(2016)]%
        {giovanidis-georgraphic}
\bibfield{author}{\bibinfo{person}{{A. Giovanidis, and A. Avranas}}.}
  \bibinfo{year}{2016}\natexlab{}.
\newblock \showarticletitle{{Spatial Multi-LRU: Distributed Caching for
  Wireless Networks with Coverage Overlaps}}.
\newblock \bibinfo{journal}{\emph{{arXiv:1612.04363}}} (\bibinfo{year}{2016}).
\newblock


\bibitem[Abernethy et~al\mbox{.}(2014)]%
        {abernethy14}
\bibfield{author}{\bibinfo{person}{Jacob Abernethy}, \bibinfo{person}{Chansoo
  Lee}, \bibinfo{person}{Abhinav Sinha}, {and} \bibinfo{person}{Ambuj Tewari}.}
  \bibinfo{year}{2014}\natexlab{}.
\newblock \showarticletitle{Online Linear Optimization via Smoothing}. In
  \bibinfo{booktitle}{\emph{Proc. of COLT}}.
\newblock


\bibitem[Alon and Spencer(2016)]%
        {alon2016probabilistic}
\bibfield{author}{\bibinfo{person}{Noga Alon} {and} \bibinfo{person}{Joel~H
  Spencer}.} \bibinfo{year}{2016}\natexlab{}.
\newblock \bibinfo{booktitle}{\emph{The probabilistic method}}.
\newblock \bibinfo{publisher}{John Wiley \& Sons}.
\newblock


\bibitem[Amatriain(2012)]%
        {recommend2}
\bibfield{author}{\bibinfo{person}{Xavier Amatriain}.}
  \bibinfo{year}{2012}\natexlab{}.
\newblock \showarticletitle{{Building Industrial-Scale Real-World Recommender
  Systems}}. In \bibinfo{booktitle}{\emph{Proc. of RecSys}}.
\newblock


\bibitem[Anderson et~al\mbox{.}(2022)]%
        {daron-ton}
\bibfield{author}{\bibinfo{person}{Daron Anderson}, \bibinfo{person}{George
  Iosifidis}, {and} \bibinfo{person}{Douglas Leith}.}
  \bibinfo{year}{2022}\natexlab{}.
\newblock \showarticletitle{{Lazy Lagrangians with Predictions for Online
  Learning}}.
\newblock \bibinfo{journal}{\emph{{arXiv preprint arXiv:2201.02890}}}
  (\bibinfo{year}{2022}).
\newblock


\bibitem[Andrew et~al\mbox{.}(2013)]%
        {pmlr-v30-Andrew13}
\bibfield{author}{\bibinfo{person}{Lachlan Andrew}, \bibinfo{person}{Siddharth
  Barman}, \bibinfo{person}{Katrina Ligett}, \bibinfo{person}{Minghong Lin},
  \bibinfo{person}{Adam Meyerson}, \bibinfo{person}{Alan Roytman}, {and}
  \bibinfo{person}{Adam Wierman}.} \bibinfo{year}{2013}\natexlab{}.
\newblock \showarticletitle{A Tale of Two Metrics: Simultaneous Bounds on
  Competitiveness and Regret}. In \bibinfo{booktitle}{\emph{Proc. of COLT}}.
\newblock


\bibitem[Antoniadis et~al\mbox{.}(2020)]%
        {antoniadis_20}
\bibfield{author}{\bibinfo{person}{Antonios Antoniadis},
  \bibinfo{person}{Christian Coester}, \bibinfo{person}{Marek Elias},
  \bibinfo{person}{Adam Polak}, {and} \bibinfo{person}{Bertrand Simon}.}
  \bibinfo{year}{2020}\natexlab{}.
\newblock \showarticletitle{{Online Metric Algorithms with Untrusted
  Predictions}}. In \bibinfo{booktitle}{\emph{Proc. of ICML}}.
\newblock


\bibitem[Bektas et~al\mbox{.}(2007)]%
        {bektas}
\bibfield{author}{\bibinfo{person}{T Bektas}, \bibinfo{person}{O Oguz}, {and}
  \bibinfo{person}{Ouveysi I.}} \bibinfo{year}{2007}\natexlab{}.
\newblock \showarticletitle{{Designing Cost-effective Content Distribution
  Networks}}.
\newblock \bibinfo{journal}{\emph{{Computers \& Operations Research}}}
  \bibinfo{volume}{34} (\bibinfo{year}{2007}), \bibinfo{pages}{2436--2449}.
\newblock


\bibitem[Bertsekas(1973)]%
        {bertsekas_stochastic}
\bibfield{author}{\bibinfo{person}{Dimitri~P Bertsekas}.}
  \bibinfo{year}{1973}\natexlab{}.
\newblock \showarticletitle{{Stochastic optimization problems with
  nondifferentiable cost functionals}}.
\newblock \bibinfo{journal}{\emph{Journal of Optimization Theory and
  Applications}} \bibinfo{volume}{12}, \bibinfo{number}{2}
  (\bibinfo{year}{1973}), \bibinfo{pages}{218--231}.
\newblock


\bibitem[Bhaskara et~al\mbox{.}(2020a)]%
        {google-2020}
\bibfield{author}{\bibinfo{person}{Aditya Bhaskara}, \bibinfo{person}{Ashok
  Cutkosky}, \bibinfo{person}{Ravi Kumar}, {and} \bibinfo{person}{Manish
  Purohit}.} \bibinfo{year}{2020}\natexlab{a}.
\newblock \showarticletitle{{Online Learning with Imperfect Hints}}. In
  \bibinfo{booktitle}{\emph{Proc. of ICML}}.
\newblock


\bibitem[Bhaskara et~al\mbox{.}(2020b)]%
        {google-nips-2020}
\bibfield{author}{\bibinfo{person}{A. Bhaskara}, \bibinfo{person}{A. Cutkosky},
  \bibinfo{person}{R. Kumar}, {and} \bibinfo{person}{M. Purohit}.}
  \bibinfo{year}{2020}\natexlab{b}.
\newblock \showarticletitle{Online Learning with Many Hints}. In
  \bibinfo{booktitle}{\emph{Proc. of NeurIPS}}.
\newblock


\bibitem[Bhattacharjee et~al\mbox{.}(2020)]%
        {abhishek-sigm20}
\bibfield{author}{\bibinfo{person}{Rajarshi Bhattacharjee},
  \bibinfo{person}{Subhankar Banerjee}, {and} \bibinfo{person}{Abhishek
  Sinha}.} \bibinfo{year}{2020}\natexlab{}.
\newblock \showarticletitle{{Fundamental Limits on the Regret of Online
  Network-Caching}}.
\newblock \bibinfo{journal}{\emph{Proc. ACM Meas. Anal. Comput. Syst.}}
  \bibinfo{volume}{4}, \bibinfo{number}{2} (\bibinfo{year}{2020}),
  \bibinfo{numpages}{31}~pages.
\newblock


\bibitem[Byrka et~al\mbox{.}(2017)]%
        {depround}
\bibfield{author}{\bibinfo{person}{Jaros\l{}aw Byrka}, \bibinfo{person}{Thomas
  Pensyl}, \bibinfo{person}{Bartosz Rybicki}, \bibinfo{person}{Aravind
  Srinivasan}, {and} \bibinfo{person}{Khoa Trinh}.}
  \bibinfo{year}{2017}\natexlab{}.
\newblock \showarticletitle{An Improved Approximation for K-Median and Positive
  Correlation in Budgeted Optimization}.
\newblock \bibinfo{journal}{\emph{ACM Trans. Algorithms}} \bibinfo{volume}{13},
  \bibinfo{number}{2} (\bibinfo{year}{2017}).
\newblock


\bibitem[Chatzieleftheriou et~al\mbox{.}(2019)]%
        {jordan-tmc}
\bibfield{author}{\bibinfo{person}{Livia~Elena Chatzieleftheriou},
  \bibinfo{person}{Merkouris Karaliopoulos}, {and} \bibinfo{person}{Iordanis
  Koutsopoulos}.} \bibinfo{year}{2019}\natexlab{}.
\newblock \showarticletitle{{Jointly Optimizing Content Caching and
  Recommendations in Small Cell Networks}}.
\newblock \bibinfo{journal}{\emph{IEEE Trans. Mobile Comput.}}
  \bibinfo{volume}{18}, \bibinfo{number}{1} (\bibinfo{year}{2019}),
  \bibinfo{pages}{125--138}.
\newblock


\bibitem[Cohen and Hazan(2015)]%
        {cohen15}
\bibfield{author}{\bibinfo{person}{Alon Cohen} {and} \bibinfo{person}{Tamir
  Hazan}.} \bibinfo{year}{2015}\natexlab{}.
\newblock \showarticletitle{{Following the Perturbed Leader for Online
  Structured Learning}}. In \bibinfo{booktitle}{\emph{Proc. of the ICML}}.
\newblock


\bibitem[Comden et~al\mbox{.}(2019)]%
        {comden-sigmetrics19}
\bibfield{author}{\bibinfo{person}{J. Comden}, \bibinfo{person}{S. Yao},
  \bibinfo{person}{N. Chen}, \bibinfo{person}{H. Xing}, {and}
  \bibinfo{person}{Z. Liu}.} \bibinfo{year}{2019}\natexlab{}.
\newblock \showarticletitle{{Online Optimization in Cloud Resource
  Provisioning: Predictions, Regrets and Algorithms}}.
\newblock \bibinfo{journal}{\emph{Proc. ACM Meas. Anal. Comput. Syst.}}
  \bibinfo{volume}{1}, \bibinfo{number}{3} (\bibinfo{year}{2019}),
  \bibinfo{numpages}{30}~pages.
\newblock


\bibitem[Cormen et~al\mbox{.}(2022)]%
        {cormen2022introduction}
\bibfield{author}{\bibinfo{person}{Thomas~H Cormen}, \bibinfo{person}{Charles~E
  Leiserson}, \bibinfo{person}{Ronald~L Rivest}, {and}
  \bibinfo{person}{Clifford Stein}.} \bibinfo{year}{2022}\natexlab{}.
\newblock \bibinfo{booktitle}{\emph{Introduction to algorithms}}.
\newblock \bibinfo{publisher}{MIT press}.
\newblock


\bibitem[{D. Chatzopoulos, C. Bermejo, Z. Huang, and P. Hui}(2017)]%
        {chatzopoulos-ARVR}
\bibfield{author}{\bibinfo{person}{{D. Chatzopoulos, C. Bermejo, Z. Huang, and
  P. Hui}}.} \bibinfo{year}{2017}\natexlab{}.
\newblock \showarticletitle{{Mobile Augmented Reality Survey: From Where We Are
  to Where We Go}}.
\newblock \bibinfo{journal}{\emph{{IEEE Access}}}  \bibinfo{volume}{5}
  (\bibinfo{year}{2017}), \bibinfo{pages}{6917--950}.
\newblock


\bibitem[Dantzig(1957)]%
        {dantzing}
\bibfield{author}{\bibinfo{person}{George~B. Dantzig}.}
  \bibinfo{year}{1957}\natexlab{}.
\newblock \showarticletitle{Discrete-Variable Extremum Problems}.
\newblock \bibinfo{journal}{\emph{Operations Research}} \bibinfo{volume}{5},
  \bibinfo{number}{2} (\bibinfo{year}{1957}), \bibinfo{pages}{266--277}.
\newblock


\bibitem[De~Rooij et~al\mbox{.}(2014)]%
        {follow-hedge}
\bibfield{author}{\bibinfo{person}{Steven De~Rooij}, \bibinfo{person}{Tim
  Van~Erven}, \bibinfo{person}{Peter~D. Gr\"{u}nwald}, {and}
  \bibinfo{person}{Wouter~M. Koolen}.} \bibinfo{year}{2014}\natexlab{}.
\newblock \showarticletitle{Follow the Leader If You Can, Hedge If You Must}.
\newblock \bibinfo{journal}{\emph{J. Mach. Learn. Res.}} \bibinfo{volume}{15},
  \bibinfo{number}{1} (\bibinfo{year}{2014}), \bibinfo{pages}{1281–1316}.
\newblock


\bibitem[{E. Leonardi, and G. Neglia}(2018)]%
        {leonardi18}
\bibfield{author}{\bibinfo{person}{{E. Leonardi, and G. Neglia}}.}
  \bibinfo{year}{2018}\natexlab{}.
\newblock \showarticletitle{{Implicit Coordination of Caches in Small-cell
  Networks Under Unknown Popularity Profiles}}.
\newblock \bibinfo{journal}{\emph{{IEEE J. Sel. Areas Commun.}}}
  \bibinfo{volume}{36}, \bibinfo{number}{6} (\bibinfo{year}{2018}),
  \bibinfo{pages}{1276--1285}.
\newblock


\bibitem[Fu et~al\mbox{.}(2022)]%
        {9681354}
\bibfield{author}{\bibinfo{person}{Yaru Fu}, \bibinfo{person}{Yue Zhang},
  \bibinfo{person}{Angus Wong}, {and} \bibinfo{person}{Tony~Q.S. Quek}.}
  \bibinfo{year}{2022}\natexlab{}.
\newblock \showarticletitle{Revenue Maximization: The Interplay Between
  Personalized Bundle Recommendation and Wireless Content Caching}.
\newblock \bibinfo{journal}{\emph{{IEEE} Trans. on Mobile Comput.}}
  (\bibinfo{year}{2022}).
\newblock
\urldef\tempurl%
\url{https://doi.org/10.1109/TMC.2022.3142809}
\showDOI{\tempurl}


\bibitem[{G. Gracioli, A. Alhammad, R. Mancuso, A. A. Frohlich, R.
  Pellizzoni}(2015)]%
        {embedded-caching}
\bibfield{author}{\bibinfo{person}{{G. Gracioli, A. Alhammad, R. Mancuso, A. A.
  Frohlich, R. Pellizzoni}}.} \bibinfo{year}{2015}\natexlab{}.
\newblock \showarticletitle{{A Survey on Cache Management Mechanisms for
  Real-Time Embedded Systems}}.
\newblock \bibinfo{journal}{\emph{{ACM Comput. Surv.}}} \bibinfo{volume}{48},
  \bibinfo{number}{2} (\bibinfo{year}{2015}), \bibinfo{numpages}{37}~pages.
\newblock


\bibitem[Giannakas et~al\mbox{.}(2021)]%
        {9528062}
\bibfield{author}{\bibinfo{person}{Theodoros Giannakas},
  \bibinfo{person}{Pavlos Sermpezis}, {and} \bibinfo{person}{Thrasyvoulos
  Spyropoulos}.} \bibinfo{year}{2021}\natexlab{}.
\newblock \showarticletitle{Network Friendly Recommendations: Optimizing for
  Long Viewing Sessions}.
\newblock \bibinfo{journal}{\emph{{IEEE} Trans. on Mobile Comput.}}
  (\bibinfo{year}{2021}).
\newblock
\urldef\tempurl%
\url{https://doi.org/10.1109/TMC.2021.3109727}
\showDOI{\tempurl}


\bibitem[Gomez-Uribe and Hunt(2016)]%
        {netflix}
\bibfield{author}{\bibinfo{person}{Carlos~A. Gomez-Uribe} {and}
  \bibinfo{person}{Neil Hunt}.} \bibinfo{year}{2016}\natexlab{}.
\newblock \showarticletitle{{The Netflix Recommender System: Algorithms,
  Business Value, and Innovation}}.
\newblock \bibinfo{journal}{\emph{ACM Trans. Manage. Inf. Syst.}}
  \bibinfo{volume}{6}, \bibinfo{number}{4} (\bibinfo{year}{2016}),
  \bibinfo{numpages}{19}~pages.
\newblock


\bibitem[Gordon et~al\mbox{.}(2008)]%
        {gordon2008no}
\bibfield{author}{\bibinfo{person}{Geoffrey~J Gordon}, \bibinfo{person}{Amy
  Greenwald}, {and} \bibinfo{person}{Casey Marks}.}
  \bibinfo{year}{2008}\natexlab{}.
\newblock \showarticletitle{No-regret learning in convex games}. In
  \bibinfo{booktitle}{\emph{Proceedings of the 25th international conference on
  Machine learning}}. \bibinfo{pages}{360--367}.
\newblock


\bibitem[Gramacy et~al\mbox{.}(2002)]%
        {ismail-nips02}
\bibfield{author}{\bibinfo{person}{Robert Gramacy}, \bibinfo{person}{Manfred
  Warmuth}, \bibinfo{person}{Scott Brandt}, {and} \bibinfo{person}{Ismail
  Ari}.} \bibinfo{year}{2002}\natexlab{}.
\newblock \showarticletitle{{Adaptive Caching by Refetching}}. In
  \bibinfo{booktitle}{\emph{Proc. of NIPS}}.
\newblock


\bibitem[Guillemin et~al\mbox{.}(2013)]%
        {volatility}
\bibfield{author}{\bibinfo{person}{Fabrice Guillemin}, \bibinfo{person}{Thierry
  Houdoin}, {and} \bibinfo{person}{Stéphanie Moteau}.}
  \bibinfo{year}{2013}\natexlab{}.
\newblock \showarticletitle{Volatility of YouTube content in Orange networks
  and consequences}. In \bibinfo{booktitle}{\emph{Proc. of ICC}}.
\newblock


\bibitem[Harper and Konstan(2015)]%
        {mlds}
\bibfield{author}{\bibinfo{person}{F.~Maxwell Harper} {and}
  \bibinfo{person}{Joseph~A. Konstan}.} \bibinfo{year}{2015}\natexlab{}.
\newblock \showarticletitle{{The MovieLens Datasets: History and Context}}.
\newblock \bibinfo{journal}{\emph{ACM Trans. Interact. Intell. Syst.}}
  \bibinfo{volume}{5}, \bibinfo{number}{4} (\bibinfo{year}{2015}),
  \bibinfo{numpages}{19}~pages.
\newblock


\bibitem[Hazan(2019)]%
        {hazan-book}
\bibfield{author}{\bibinfo{person}{Elad Hazan}.}
  \bibinfo{year}{2019}\natexlab{}.
\newblock \bibinfo{title}{{Introduction to Online Convex Optimization}}.
\newblock
\newblock
\urldef\tempurl%
\url{https://arxiv.org/abs/1909.05207}
\showURL{%
\tempurl}


\bibitem[Hutter and Poland(2005)]%
        {JMLR:v6:hutter05a}
\bibfield{author}{\bibinfo{person}{Marcus Hutter} {and} \bibinfo{person}{Jan
  Poland}.} \bibinfo{year}{2005}\natexlab{}.
\newblock \showarticletitle{Adaptive Online Prediction by Following the
  Perturbed Leader}.
\newblock \bibinfo{journal}{\emph{Journal of Machine Learning Research}}
  \bibinfo{volume}{6}, \bibinfo{number}{22} (\bibinfo{year}{2005}),
  \bibinfo{pages}{639--660}.
\newblock


\bibitem[Jelenkovi\'{c} and Kang(2008)]%
        {lru-sigmetrics08}
\bibfield{author}{\bibinfo{person}{Predrag~R. Jelenkovi\'{c}} {and}
  \bibinfo{person}{Xiaozhu Kang}.} \bibinfo{year}{2008}\natexlab{}.
\newblock \showarticletitle{{Characterizing the Miss Sequence of the LRU
  Cache}}.
\newblock \bibinfo{journal}{\emph{SIGMETRICS Perform. Eval. Rev.}}
  \bibinfo{volume}{36}, \bibinfo{number}{2} (\bibinfo{year}{2008}),
  \bibinfo{pages}{119--121}.
\newblock


\bibitem[Joshi and Sinha(2022)]%
        {joshi2022universal}
\bibfield{author}{\bibinfo{person}{Ativ Joshi} {and} \bibinfo{person}{Abhishek
  Sinha}.} \bibinfo{year}{2022}\natexlab{}.
\newblock \showarticletitle{Universal Caching}.
\newblock \bibinfo{journal}{\emph{arXiv preprint arXiv:2205.04860}}
  (\bibinfo{year}{2022}).
\newblock


\bibitem[{K. Chen, and L. Huang}(2018)]%
        {longbo-ton18}
\bibfield{author}{\bibinfo{person}{{K. Chen, and L. Huang}}.}
  \bibinfo{year}{2018}\natexlab{}.
\newblock \showarticletitle{{Timely-Throughput Optimal Scheduling With
  Prediction}}.
\newblock \bibinfo{journal}{\emph{{IEEE/ACM Tran. on Networking}}}
  \bibinfo{volume}{26}, \bibinfo{number}{6} (\bibinfo{year}{2018}),
  \bibinfo{pages}{2457--2470}.
\newblock


\bibitem[Kalai and Vempala(2005)]%
        {kalai2005efficient}
\bibfield{author}{\bibinfo{person}{Adam Kalai} {and} \bibinfo{person}{Santosh
  Vempala}.} \bibinfo{year}{2005}\natexlab{}.
\newblock \showarticletitle{Efficient algorithms for online decision problems}.
\newblock \bibinfo{journal}{\emph{J. Comput. System Sci.}}
  \bibinfo{volume}{71}, \bibinfo{number}{3} (\bibinfo{year}{2005}),
  \bibinfo{pages}{291--307}.
\newblock


\bibitem[Leconte et~al\mbox{.}(2016)]%
        {mathieu}
\bibfield{author}{\bibinfo{person}{Mathieu Leconte}, \bibinfo{person}{Georgios
  Paschos}, \bibinfo{person}{Lazaros Gkatzikis}, \bibinfo{person}{Moez Draief},
  \bibinfo{person}{Spyridon Vassilaras}, {and} \bibinfo{person}{Symeon
  Chouvardas}.} \bibinfo{year}{2016}\natexlab{}.
\newblock \showarticletitle{{Placing Dynamic Content in Caches with Small
  Population}}. In \bibinfo{booktitle}{\emph{Proc. of IEEE INFOCOM}}.
\newblock


\bibitem[Lee et~al\mbox{.}(1999)]%
        {lfu-sigmetrics99}
\bibfield{author}{\bibinfo{person}{Donghee Lee}, \bibinfo{person}{Jongmoo
  Choi}, \bibinfo{person}{Jong-Hun Kim}, \bibinfo{person}{Sam~H. Noh},
  \bibinfo{person}{Sang~Lyul Min}, \bibinfo{person}{Yookun Cho}, {and}
  \bibinfo{person}{Chong~Sang Kim}.} \bibinfo{year}{1999}\natexlab{}.
\newblock \showarticletitle{{On the Existence of a Spectrum of Policies That
  Subsumes the Least Recently Used (LRU) and Least Frequently Used (LFU)
  Policies}}.
\newblock \bibinfo{journal}{\emph{SIGMETRICS Perform. Eval. Rev.}}
  \bibinfo{volume}{27}, \bibinfo{number}{1} (\bibinfo{year}{1999}),
  \bibinfo{pages}{134--143}.
\newblock


\bibitem[Li et~al\mbox{.}(2021)]%
        {10.1145/3491047}
\bibfield{author}{\bibinfo{person}{Yuanyuan Li}, \bibinfo{person}{Tareq
  Si~Salem}, \bibinfo{person}{Giovanni Neglia}, {and} \bibinfo{person}{Stratis
  Ioannidis}.} \bibinfo{year}{2021}\natexlab{}.
\newblock \showarticletitle{Online Caching Networks with Adversarial
  Guarantees}.
\newblock \bibinfo{journal}{\emph{Proc. ACM Meas. Anal. Comput. Syst.}}
  \bibinfo{volume}{5}, \bibinfo{number}{3} (\bibinfo{year}{2021}),
  \bibinfo{numpages}{39}~pages.
\newblock


\bibitem[Lykouris and Vassilvtiskii(2018)]%
        {Lykouris-ML}
\bibfield{author}{\bibinfo{person}{Thodoris Lykouris} {and}
  \bibinfo{person}{Sergei Vassilvtiskii}.} \bibinfo{year}{2018}\natexlab{}.
\newblock \showarticletitle{{Competitive Caching with Machine Learned Advice}}.
  In \bibinfo{booktitle}{\emph{Proc. of ICML}}.
\newblock


\bibitem[{M. A. Maddah-Ali, and U. Niesen}(2014)]%
        {maddah-ali}
\bibfield{author}{\bibinfo{person}{{M. A. Maddah-Ali, and U. Niesen}}.}
  \bibinfo{year}{2014}\natexlab{}.
\newblock \showarticletitle{{Fundamental Limits of Caching}}.
\newblock \bibinfo{journal}{\emph{{IEEE Trans. Inf. Theory}}}
  \bibinfo{volume}{60}, \bibinfo{number}{5} (\bibinfo{year}{2014}),
  \bibinfo{pages}{2856--2867}.
\newblock


\bibitem[Madow(1949)]%
        {madow}
\bibfield{author}{\bibinfo{person}{William~G. Madow}.}
  \bibinfo{year}{1949}\natexlab{}.
\newblock \showarticletitle{{On the Theory of Systematic Sampling}}.
\newblock \bibinfo{journal}{\emph{{The Annals of Mathematical Statistics}}}
  \bibinfo{volume}{20}, \bibinfo{number}{3} (\bibinfo{year}{1949}),
  \bibinfo{pages}{333--354}.
\newblock


\bibitem[Martello and Toth(1990)]%
        {knapsack-book}
\bibfield{author}{\bibinfo{person}{Silvano Martello} {and}
  \bibinfo{person}{Paolo Toth}.} \bibinfo{year}{1990}\natexlab{}.
\newblock \bibinfo{booktitle}{\emph{{Knapsack Problems: Algorithms and Computer
  Implementations}}}.
\newblock \bibinfo{publisher}{J. Willey \& Sons}.
\newblock


\bibitem[McMahan(2017)]%
        {mcmahan-survey17}
\bibfield{author}{\bibinfo{person}{H.~Brendan McMahan}.}
  \bibinfo{year}{2017}\natexlab{}.
\newblock \showarticletitle{{A Survey of Algorithms and Analysis for Adaptive
  Online Learning}}.
\newblock \bibinfo{journal}{\emph{J. Mach. Learn. Res.}} \bibinfo{volume}{18},
  \bibinfo{number}{1} (\bibinfo{year}{2017}), \bibinfo{pages}{3117--3166}.
\newblock


\bibitem[Mhaisen et~al\mbox{.}(2022a)]%
        {naram-jrnl}
\bibfield{author}{\bibinfo{person}{Naram Mhaisen}, \bibinfo{person}{George
  Iosifidis}, {and} \bibinfo{person}{Douglas Leith}.}
  \bibinfo{year}{2022}\natexlab{a}.
\newblock \bibinfo{title}{Online Caching with no Regret: Optimistic Learning
  via Recommendations}.
\newblock
\newblock
\urldef\tempurl%
\url{https://arxiv.org/abs/2204.09345}
\showURL{%
\tempurl}


\bibitem[Mhaisen et~al\mbox{.}(2022b)]%
        {ocol}
\bibfield{author}{\bibinfo{person}{Naram Mhaisen}, \bibinfo{person}{George
  Iosifidis}, {and} \bibinfo{person}{Douglas Leith}.}
  \bibinfo{year}{2022}\natexlab{b}.
\newblock \showarticletitle{{Online Caching with Optimistic Learning}}. In
  \bibinfo{booktitle}{\emph{Proc. of IFIP Networking}}.
\newblock


\bibitem[Mohri and Yang(2016)]%
        {mohri-aistats2016}
\bibfield{author}{\bibinfo{person}{Mehryar Mohri} {and} \bibinfo{person}{Scott
  Yang}.} \bibinfo{year}{2016}\natexlab{}.
\newblock \showarticletitle{{Accelerating Online Convex Optimization via
  Adaptive Prediction}}. In \bibinfo{booktitle}{\emph{Proc. of AISTATS}}.
\newblock


\bibitem[Mukhopadhyay et~al\mbox{.}(2022)]%
        {mukhopadhyay2022k}
\bibfield{author}{\bibinfo{person}{Samrat Mukhopadhyay},
  \bibinfo{person}{Sourav Sahoo}, {and} \bibinfo{person}{Abhishek Sinha}.}
  \bibinfo{year}{2022}\natexlab{}.
\newblock \showarticletitle{k-experts-Online Policies and Fundamental Limits}.
  In \bibinfo{booktitle}{\emph{International Conference on Artificial
  Intelligence and Statistics}}. PMLR, \bibinfo{pages}{342--365}.
\newblock


\bibitem[O.~Dekel(2017)]%
        {dekel-nips17}
\bibfield{author}{\bibinfo{person}{et~al. O.~Dekel}.}
  \bibinfo{year}{2017}\natexlab{}.
\newblock \showarticletitle{{Online Learning with a Hint}}. In
  \bibinfo{booktitle}{\emph{Proc. of NeurIPS}}.
\newblock


\bibitem[Olmos et~al\mbox{.}(2014)]%
        {kauffman}
\bibfield{author}{\bibinfo{person}{Felipe Olmos}, \bibinfo{person}{Bruno
  Kauffmann}, \bibinfo{person}{Alain Simonian}, {and} \bibinfo{person}{Yannick
  Carlinet}.} \bibinfo{year}{2014}\natexlab{}.
\newblock \showarticletitle{{Catalog dynamics: Impact of content publishing and
  perishing on the performance of a LRU cache}}. In
  \bibinfo{booktitle}{\emph{Proc. of ITC}}.
\newblock


\bibitem[Orabona(2019)]%
        {orabona2021modern}
\bibfield{author}{\bibinfo{person}{Francesco Orabona}.}
  \bibinfo{year}{2019}\natexlab{}.
\newblock \bibinfo{title}{{A Modern Introduction to Online Learning}}.
\newblock
\newblock
\urldef\tempurl%
\url{https://arxiv.org/abs/1912.13213}
\showURL{%
\tempurl}


\bibitem[Paria and Sinha(2021)]%
        {leadcache}
\bibfield{author}{\bibinfo{person}{Debjit Paria} {and}
  \bibinfo{person}{Abhishek Sinha}.} \bibinfo{year}{2021}\natexlab{}.
\newblock \showarticletitle{LeadCache: Regret-Optimal Caching in Networks}. In
  \bibinfo{booktitle}{\emph{Proc. of NeurIPS}}.
\newblock


\bibitem[Paschos et~al\mbox{.}(2020b)]%
        {paschos-book}
\bibfield{author}{\bibinfo{person}{Georgios Paschos}, \bibinfo{person}{George
  Iosifidis}, {and} \bibinfo{person}{Giuseppe Caire}.}
  \bibinfo{year}{2020}\natexlab{b}.
\newblock \showarticletitle{{Cache Optimization Models and Algorithms}}.
\newblock \bibinfo{journal}{\emph{FnT in Communications and Information
  Theory}} \bibinfo{volume}{16}, \bibinfo{number}{3–4}
  (\bibinfo{year}{2020}), \bibinfo{pages}{156--345}.
\newblock


\bibitem[Paschos et~al\mbox{.}(2020a)]%
        {paschos-jrnl}
\bibfield{author}{\bibinfo{person}{Georgios~S. Paschos},
  \bibinfo{person}{Apostolos Destounis}, {and} \bibinfo{person}{George
  Iosifidis}.} \bibinfo{year}{2020}\natexlab{a}.
\newblock \showarticletitle{{Online Convex Optimization for Caching Networks}}.
\newblock \bibinfo{journal}{\emph{IEEE/ACM Trans. Networking}}
  \bibinfo{volume}{28}, \bibinfo{number}{2} (\bibinfo{year}{2020}),
  \bibinfo{pages}{625--638}.
\newblock


\bibitem[Paschos et~al\mbox{.}(2018)]%
        {paschos-jsac}
\bibfield{author}{\bibinfo{person}{Georgios~S. Paschos},
  \bibinfo{person}{George Iosifidis}, \bibinfo{person}{Meixia Tao},
  \bibinfo{person}{Don Towsley}, {and} \bibinfo{person}{Giuseppe Caire}.}
  \bibinfo{year}{2018}\natexlab{}.
\newblock \showarticletitle{{The Role of Caching in Future Communication
  Systems and Networks}}.
\newblock \bibinfo{journal}{\emph{IEEE J. Select. Areas Commun.}}
  \bibinfo{volume}{36}, \bibinfo{number}{6} (\bibinfo{year}{2018}),
  \bibinfo{pages}{1111--1125}.
\newblock


\bibitem[Poularakis et~al\mbox{.}(2014)]%
        {6883210}
\bibfield{author}{\bibinfo{person}{Konstantinos Poularakis},
  \bibinfo{person}{George Iosifidis}, {and} \bibinfo{person}{Leandros
  Tassiulas}.} \bibinfo{year}{2014}\natexlab{}.
\newblock \showarticletitle{{Approximation Algorithms for Mobile Data Caching
  in Small Cell Networks}}.
\newblock \bibinfo{journal}{\emph{IEEE Trans. Commun.}} \bibinfo{volume}{62},
  \bibinfo{number}{10} (\bibinfo{year}{2014}), \bibinfo{pages}{3665--3677}.
\newblock


\bibitem[Rakhlin and Sridharan(2013)]%
        {rakhlin-nips2013}
\bibfield{author}{\bibinfo{person}{Alexander Rakhlin} {and}
  \bibinfo{person}{Karthik Sridharan}.} \bibinfo{year}{2013}\natexlab{}.
\newblock \showarticletitle{{Optimization, Learning, and Games with Predictable
  Sequences}}. In \bibinfo{booktitle}{\emph{Proc. of NeurIPS}}.
\newblock


\bibitem[Rodriguez et~al\mbox{.}(2021)]%
        {rodriguez-usenis21}
\bibfield{author}{\bibinfo{person}{Liana Rodriguez}, \bibinfo{person}{Farzana
  Yusuf}, \bibinfo{person}{Steven Lyons}, \bibinfo{person}{Eysler Paz}, {and}
  \bibinfo{person}{Raju Rangaswami}.} \bibinfo{year}{2021}\natexlab{}.
\newblock \showarticletitle{{Learning Cache Replacement with Cacheus}}. In
  \bibinfo{booktitle}{\emph{Proc. of USENIX Conferecne on File and Storage
  Technologies}}.
\newblock


\bibitem[Rohatgi(2020)]%
        {Rohatgi}
\bibfield{author}{\bibinfo{person}{Dhruv Rohatgi}.}
  \bibinfo{year}{2020}\natexlab{}.
\newblock \showarticletitle{Near-Optimal Bounds for Online Caching with Machine
  Learned Advice}. In \bibinfo{booktitle}{\emph{Proc. of ACM-SIAM SODA}}.
\newblock


\bibitem[Rutten et~al\mbox{.}(2022)]%
        {rutten_22}
\bibfield{author}{\bibinfo{person}{Daan Rutten}, \bibinfo{person}{Nico
  Christianson}, \bibinfo{person}{Debankur Mukherjee}, {and}
  \bibinfo{person}{Adam Wierman}.} \bibinfo{year}{2022}\natexlab{}.
\newblock \bibinfo{title}{{Online Optimization with Untrusted Predictions}}.
\newblock
\newblock
\urldef\tempurl%
\url{https://arxiv.org/abs/2202.03519}
\showURL{%
\tempurl}


\bibitem[Sachs et~al\mbox{.}(2022)]%
        {sarah-between}
\bibfield{author}{\bibinfo{person}{Sarah Sachs}, \bibinfo{person}{Hédi
  Hadiji}, \bibinfo{person}{Tim van Erven}, {and} \bibinfo{person}{Cristóbal
  Guzmán}.} \bibinfo{year}{2022}\natexlab{}.
\newblock \bibinfo{title}{Between Stochastic and Adversarial Online Convex
  Optimization: Improved Regret Bounds via Smoothness}.
\newblock
\newblock
\urldef\tempurl%
\url{https://arxiv.org/abs/2202.07554}
\showURL{%
\tempurl}


\bibitem[Sadeghi et~al\mbox{.}(2018)]%
        {giannakis-q-learning}
\bibfield{author}{\bibinfo{person}{Alireza Sadeghi}, \bibinfo{person}{Fatemeh
  Sheikholeslami}, {and} \bibinfo{person}{Georgios~B. Giannakis}.}
  \bibinfo{year}{2018}\natexlab{}.
\newblock \showarticletitle{{Optimal and Scalable Caching for 5G Using
  Reinforcement Learning of Space-Time Popularities}}.
\newblock \bibinfo{journal}{\emph{IEEE J. Select. Areas Commun.}}
  \bibinfo{volume}{12}, \bibinfo{number}{1} (\bibinfo{year}{2018}),
  \bibinfo{pages}{180--190}.
\newblock


\bibitem[Sadeghi et~al\mbox{.}(2019)]%
        {8790766}
\bibfield{author}{\bibinfo{person}{Alireza Sadeghi}, \bibinfo{person}{Fatemeh
  Sheikholeslami}, \bibinfo{person}{Antonio~G. Marques}, {and}
  \bibinfo{person}{Georgios~B. Giannakis}.} \bibinfo{year}{2019}\natexlab{}.
\newblock \showarticletitle{{Reinforcement Learning for Adaptive Caching With
  Dynamic Storage Pricing}}.
\newblock \bibinfo{journal}{\emph{IEEE J. Select. Areas Commun}}
  \bibinfo{volume}{37}, \bibinfo{number}{10} (\bibinfo{year}{2019}),
  \bibinfo{pages}{2267--2281}.
\newblock


\bibitem[Shanmugam et~al\mbox{.}(2013)]%
        {femtocaching}
\bibfield{author}{\bibinfo{person}{Karthikeyan Shanmugam},
  \bibinfo{person}{Negin Golrezaei}, \bibinfo{person}{Alexandros~G Dimakis},
  \bibinfo{person}{Andreas~F Molisch}, {and} \bibinfo{person}{Giuseppe Caire}.}
  \bibinfo{year}{2013}\natexlab{}.
\newblock \showarticletitle{{Femtocaching: Wireless Content Delivery Through
  Distributed Caching Helpers}}.
\newblock \bibinfo{journal}{\emph{IEEE Trans. Inform. Theory}}
  \bibinfo{volume}{59}, \bibinfo{number}{12} (\bibinfo{year}{2013}),
  \bibinfo{pages}{8402--8413}.
\newblock


\bibitem[Si~Salem et~al\mbox{.}(2021a)]%
        {tareq-jrnl}
\bibfield{author}{\bibinfo{person}{T. Si~Salem}, \bibinfo{person}{G. Neglia},
  {and} \bibinfo{person}{S. Ioannidis}.} \bibinfo{year}{2021}\natexlab{a}.
\newblock \bibinfo{title}{No-Regret Caching via Online Mirror Descent}.
\newblock
\newblock
\urldef\tempurl%
\url{https://arxiv.org/abs/2101.12588}
\showURL{%
\tempurl}


\bibitem[Si~Salem et~al\mbox{.}(2021b)]%
        {tareq-conf}
\bibfield{author}{\bibinfo{person}{Tareq Si~Salem}, \bibinfo{person}{Giovanni
  Neglia}, {and} \bibinfo{person}{Stratis Ioannidis}.}
  \bibinfo{year}{2021}\natexlab{b}.
\newblock \showarticletitle{{No-Regret Caching via Online Mirror Descent}}. In
  \bibinfo{booktitle}{\emph{Proc. of ICC}}.
\newblock


\bibitem[Somuyiwa et~al\mbox{.}(2018)]%
        {gunduz-reinforcement}
\bibfield{author}{\bibinfo{person}{Samuel~O. Somuyiwa},
  \bibinfo{person}{András György}, {and} \bibinfo{person}{Deniz Gündüz}.}
  \bibinfo{year}{2018}\natexlab{}.
\newblock \showarticletitle{{A Reinforcement-Learning Approach to Proactive
  Caching in Wireless Networks}}.
\newblock \bibinfo{journal}{\emph{IEEE J. Select. Areas Commun.}}
  \bibinfo{volume}{36}, \bibinfo{number}{6} (\bibinfo{year}{2018}),
  \bibinfo{pages}{1331--1344}.
\newblock


\bibitem[Suggala and Netrapalli(2020a)]%
        {suggala-2}
\bibfield{author}{\bibinfo{person}{Arun Suggala} {and}
  \bibinfo{person}{Praneeth Netrapalli}.} \bibinfo{year}{2020}\natexlab{a}.
\newblock \showarticletitle{Follow the Perturbed Leader: Optimism and Fast
  Parallel Algorithms for Smooth Minimax Games}. In
  \bibinfo{booktitle}{\emph{Proc. of NeurIPS}}.
\newblock


\bibitem[Suggala and Netrapalli(2020b)]%
        {suggala-1}
\bibfield{author}{\bibinfo{person}{Arun~Sai Suggala} {and}
  \bibinfo{person}{Praneeth Netrapalli}.} \bibinfo{year}{2020}\natexlab{b}.
\newblock \showarticletitle{Online Non-Convex Learning: Following the Perturbed
  Leader is Optimal}. In \bibinfo{booktitle}{\emph{Proc. of ALT}}.
\newblock


\bibitem[Traverso et~al\mbox{.}(2013)]%
        {snm}
\bibfield{author}{\bibinfo{person}{Stefano Traverso}, \bibinfo{person}{Mohamed
  Ahmed}, \bibinfo{person}{Michele Garetto}, \bibinfo{person}{Paolo Giaccone},
  \bibinfo{person}{Emilio Leonardi}, {and} \bibinfo{person}{Saverio
  Niccolini}.} \bibinfo{year}{2013}\natexlab{}.
\newblock \showarticletitle{{Temporal Locality in Today's Content Caching: Why
  It Matters and How to Model It}}.
\newblock \bibinfo{journal}{\emph{SIGCOMM Comput. Commun. Rev.}}
  \bibinfo{volume}{43}, \bibinfo{number}{5} (\bibinfo{year}{2013}),
  \bibinfo{pages}{5--12}.
\newblock


\bibitem[Vazirani(2001)]%
        {vazirani2001approximation}
\bibfield{author}{\bibinfo{person}{Vijay~V Vazirani}.}
  \bibinfo{year}{2001}\natexlab{}.
\newblock \bibinfo{booktitle}{\emph{Approximation algorithms}}.
  Vol.~\bibinfo{volume}{1}.
\newblock \bibinfo{publisher}{Springer}.
\newblock


\bibitem[{W. Wang, and C. Lu}(2015)]%
        {projection-simplex}
\bibfield{author}{\bibinfo{person}{{W. Wang, and C. Lu}}.}
  \bibinfo{year}{2015}\natexlab{}.
\newblock \showarticletitle{{Projection onto the Capped Simplex}}.
\newblock \bibinfo{journal}{\emph{{arXiv preprint arXiv:1503.01002}}}
  (\bibinfo{year}{2015}).
\newblock


\bibitem[{X. Huang, S. Bian, X. Gao, W. Wu, Z. Shao, Y. Yang, J. C.S.
  Lui}(2021)]%
        {lui-ton21}
\bibfield{author}{\bibinfo{person}{{X. Huang, S. Bian, X. Gao, W. Wu, Z. Shao,
  Y. Yang, J. C.S. Lui}}.} \bibinfo{year}{2021}\natexlab{}.
\newblock \showarticletitle{{Online VNF Chaining and Predictive Scheduling:
  Optimality and Trade-Offs}}.
\newblock \bibinfo{journal}{\emph{{IEEE/ACM Tran. on Networking}}}
  \bibinfo{volume}{29}, \bibinfo{number}{4} (\bibinfo{year}{2021}),
  \bibinfo{pages}{1867--1880}.
\newblock


\bibitem[{Z. Zhou, X. Chen, W. Wu, D. Wu, and J. Zhang}(2019)]%
        {xu-mobihoc19}
\bibfield{author}{\bibinfo{person}{{Z. Zhou, X. Chen, W. Wu, D. Wu, and J.
  Zhang}}.} \bibinfo{year}{2019}\natexlab{}.
\newblock \showarticletitle{{Predictive Online Server Provisioning for
  Cost-Efficient IoT Data Streaming Across Collaborative Edges}}. In
  \bibinfo{booktitle}{\emph{Proc. of ACM Mobihoc}}.
\newblock


\bibitem[Zink et~al\mbox{.}(2009)]%
        {zink2008watch}
\bibfield{author}{\bibinfo{person}{Michael Zink}, \bibinfo{person}{Kyoungwon
  Suh}, \bibinfo{person}{Yu Gu}, {and} \bibinfo{person}{Jim Kurose}.}
  \bibinfo{year}{2009}\natexlab{}.
\newblock \showarticletitle{{Characteristics of YouTube Network Traffic at a
  Campus Network - Measurements, Models, and Implications}}.
\newblock \bibinfo{journal}{\emph{Comput. Netw.}} \bibinfo{volume}{53},
  \bibinfo{number}{4} (\bibinfo{year}{2009}), \bibinfo{pages}{501--514}.
\newblock


\end{thebibliography}
